\pdfoutput=1
\documentclass[11pt]{article}
\pdfoutput=1
\usepackage[margin=0.91in]{geometry} 
\usepackage{kpfonts}
\usepackage[colorlinks=true,citecolor=blue]{hyperref}
\usepackage{url}

\usepackage{lipsum}



\usepackage[utf8]{inputenc} 
\usepackage[T1]{fontenc}    
\usepackage{booktabs}       
\usepackage{amsfonts}       
\usepackage{nicefrac}       
\usepackage{microtype,soul}      
\usepackage{natbib}

\usepackage{amsfonts}
\usepackage{xcolor}
\let\proof\relax
\let\endproof\relax
\usepackage{amsthm}
\usepackage{amsmath}
\usepackage{amssymb}
\usepackage{soul}
\usepackage{comment}
\usepackage{enumitem}
\usepackage{bbm}
\usepackage{graphicx}
\usepackage{float}
\usepackage{subfigure}
\usepackage{mathtools}
\usepackage[ruled,linesnumbered,noend,noline]{algorithm2e}
\usepackage{bm}
\usepackage{accents}

\newtheorem{definition}{Definition}
\newtheorem{assumption}{Assumption}
\newtheorem{theorem}{Theorem}
\newtheorem{corollary}{Corollary}
\newtheorem{lemma}{Lemma}
\newtheorem{proposition}{Proposition}
\newtheorem{remark}{Remark}

\theoremstyle{remark}

\newcommand{\ra}{\rightarrow}
\newcommand{\mb}[1]{\mathbb{#1}}

\newcommand{\mc}[1]{\mathcal{#1}}

\newcommand{\rr}{\mathbb{R}}
\newcommand{\ee}{\mathbb{E}}
\newcommand{\pp}{\mathbb{P}}

\newcommand{\nn}{\mathbb{N}}
\renewcommand{\epsilon}{\varepsilon}

\newcommand{\abs}[1]{\left| #1 \right|}
\newcommand{\norm}[1]{\left\| #1 \right\|}
	
\newcommand{\e}[1]{\ee\left[#1\right]}
\newcommand{\inner}[1]{\left\langle #1 \right\rangle}

\newcommand{\floor}[1]{\left\lfloor #1 \right\rfloor}
\renewcommand{\l}{\left(}
\renewcommand{\r}{\right)}
\renewcommand{\L}{\left[}
\newcommand{\R}{\right]}
\newcommand{\eqa}{\stackrel{(a)}{=}}
\newcommand{\eqb}{\stackrel{(b)}{=}}
\newcommand{\eqc}{\stackrel{(c)}{=}}

\newcommand{\indicator}{\mathbbm{1}}
\newcommand{\grad}{\nabla}

\newcommand{\ol}{\overline}
\renewcommand{\ul}{\underline}
\newcommand{\ut}{\undertilde}
\newcommand{\ot}{\tilde}

\newcommand{\T}{^\intercal}

\newcommand\defeq{\stackrel{\mathclap{\normalfont\tiny\mbox{def}}}{=}}

\newcommand{\wmst}[1]{{{}}}
\DeclareMathOperator*{\argmin}{argmin} 

\usepackage{tikz}

\newcommand{\loss}{L}
\newcommand{\amax}{A_{\max}}
\usepackage{enumitem}
\renewcommand{\texttt}[1]{{\fontfamily{lmtt}\selectfont#1}}

\title{On Improving Model-Free Algorithms for Decentralized \\Multi-Agent Reinforcement Learning\footnotetext{Research of W.M.\@ and T.B.\@ was supported in part by the ONR MURI Grant N00014-16-1-2710 and in part by the IBM-Illinois Discovery Accelerator Institute. We thank Zihan Zhang and Chen-Yu Wei for the helpful discussions and feedback. }
}

\makeatletter
\def\@fnsymbol#1{\ensuremath{\ifcase#1\or  \natural \or \dagger\or * \or \ddagger\or
		\mathsection\or \mathparagraph\or \|\or **\or \dagger\dagger
		\or \ddagger\ddagger \else\@ctrerr\fi}}
\makeatother

\author{\normalsize Weichao Mao\thanks{Department of Electrical and Computer Engineering \& Coordinated Science Laboratory, University of Illinois Urbana-Champaign. Email addresses: \{weichao2, basar1\}@illinois.edu}  \qquad \quad Lin F. Yang\thanks{Department of Electrical and Computer Engineering, University of California, Los Angeles. Part of this work done while the author was visiting DeepMind. Email address:  linyang@ee.ucla.edu} \qquad \quad Kaiqing Zhang\thanks{Laboratory for Information \& Decision Systems, Massachusetts Institute of Technology. Email address: kaiqing@mit.edu} \qquad \quad Tamer Ba\c{s}ar$^\natural$ }
\date{\normalsize }

\begin{document}

\maketitle

\begin{abstract}
	Multi-agent reinforcement learning (MARL) algorithms often suffer from an exponential sample complexity dependence on the number of agents, a phenomenon known as \emph{the curse of multiagents}. In this paper, we address this challenge by investigating sample-efficient model-free algorithms in \emph{decentralized} MARL, and aim to improve existing algorithms along this line. For learning (coarse) correlated equilibria  in general-sum Markov games, we propose \emph{stage-based} V-learning algorithms that significantly simplify the algorithmic design and analysis of recent works, and circumvent a rather complicated no-\emph{weighted}-regret bandit subroutine. For learning Nash equilibria in Markov potential games, we propose an independent policy gradient algorithm with a decentralized momentum-based variance reduction technique. All our algorithms are decentralized in that each agent can make decisions based on only its local information. Neither communication nor centralized coordination is required during learning, leading to a natural generalization to a large number of agents. We also provide numerical simulations to corroborate our theoretical findings. 
\end{abstract}

\section{Introduction}\label{sec:intro}
Many real-world sequential decision-making problems involve the strategic interactions of multiple agents in a shared environment, which are commonly addressed with multi-agent reinforcement learning (MARL). Successful applications of MARL include playing the game of Go~\citep{silver2016mastering}, Poker~\citep{brown2018superhuman}, real-time strategy  games~\citep{vinyals2019grandmaster}, autonomous driving~\citep{shalev2016safe}, and robotics~\citep{kober2013reinforcement}. 

Despite the empirical successes, sample-efficient solutions are still relatively lacking for MARL with a large number of agents, mostly due to the well-known challenge named \emph{the curse of multiagents}~\citep{jin2021v}: The joint action space in a MARL problem is equal to the Cartesian product of the individual action spaces of all agents, which scales exponentially in the number of agents. A typical kind of algorithms that easily fail at this challenge are those using centralized/joint learning~\citep{boutilier1996planning,claus1998dynamics}. Specifically, centralized learning assumes the existence of a single coordinator who can access the local information of all the agents, and learns policies jointly for all of them. This centralized training (though possibly decentralized execution) approach has become a common practice in empirical MARL \citep{oliehoek2008optimal,foerster2016learning,lowe2017multi,rashid2018qmix,son2019qtran,mao2020information}. Centralized learning essentially reduces the multi-agent problem to a single-agent one, but unfortunately suffers from the exponential dependence as it usually needs to exhaustively search the joint action space. 

Such a computation bottleneck can be partially resolved by allowing communications among the agents and hence distributing the workload to each of them~\citep{kar13,zhang2018fully,dubey2021provably}. However, communication-based methods instead suffer from the additional communication overheads, which can be unrealistic in some real-world scenarios where communication may be expensive and/or unreliable, such as in unmanned aerial vehicle (UAV) field coverage~\citep{pham2018cooperative}. 

Given the aforementioned limitations, in this paper, we are interested in a more practical setting: \emph{decentralized} learning\footnote{This setting has been studied under various names in the literature, including individual learning \citep{leslie2005individual}, decentralized learning~\citep{arslan2016decentralized}, agnostic learning~\citep{tian2020provably,wei2021last}, and independent learning~\citep{claus1998dynamics,daskalakis2020independent}. It also belongs to a more general  category of teams/games with decentralized information structure~\citep{ho1980team,nayyar2013common,nayyar2013decentralized}.}. We focus on solutions where each agent can make decisions based on only its local information (e.g., local actions and rewards), and need not communicate with its opponents or be coordinated by any central controller during learning. In fact, in our algorithms, the agents can be completely oblivious to the presence of other agents. Under such weak assumptions, decentralized algorithms are suitable for many practical MARL scenarios~\citep{fudenberg1998theory}, and do not suffer from the exponential sample \& computation complexity. Such algorithms are naturally model-free, as they do not maintain explicit estimates of the transition functions. Compared with model-based algorithms, model-free ones typically enjoy higher time- and space-efficiency, and are more compatible with the modern deep RL architectures~\citep{jin2018q,zhang2020almost}. 

In this paper, we investigate the theoretical aspects of decentralized MARL in the non-asymptotic regime. We address the curse of multiagents by presenting sample-efficient model-free algorithms that scale to a large number of agents, and aim to improve the existing algorithms along this line. Our main contributions are summarized as follows.

\vspace{.8em}
\noindent\textbf{Contributions.} 1) For general-sum Markov games (Section~\ref{sec:cce_ce}), we present algorithms that learn an $\epsilon$-approximate coarse correlated equilibrium (CCE) in $\widetilde{O}(H^5SA_{\max}/\epsilon^2)$ episodes, and an  $\epsilon$-approximate correlated equilibrium (CE) in $\widetilde{O}(H^5SA_{\max}^2/\epsilon^2)$ episodes, where $S$ is the number of states, $A_{\max}$ is the size of the largest individual action space, and $H$ is the length of an episode. Our algorithms rely on a novel \emph{stage-based} V-learning method that significantly simplifies the algorithmic design and analysis of recent works. 2) In the important special case of Markov potential games (MPGs, Section~\ref{sec:mpg}), we propose an independent policy gradient algorithm  that learns an $\epsilon$-approximate Nash equilibrium (NE) in $\propto \widetilde{O}(1/\epsilon^{4.5})$ episodes. Our algorithm utilizes a momentum-based variance reduction technique that can be executed in a decentralized way. 3) We further provide numerical results that corroborate our theoretical findings (Section~\ref{sec:simulations}). All our algorithms are decentralized and model-free, and readily generalize to a large number of agents.

\vspace{.8em}
\noindent\textbf{Related Work.} A common mathematical framework of MARL is stochastic games \citep{shapley1953stochastic}, which are also referred to as Markov games. Early attempts to learn NE in Markov games include \citet{littman1994markov,littman2001friend,hu2003nash,hansen2013strategy}, but they either assume the transition kernel and rewards are known, or only yield asymptotic guarantees. Recently, various sample-efficient methods have been proposed \citep{wei2017online,bai2020provable,sidford2020solving,xie2020learning,bai2020near,liu2020sharp,zhao2021provably,guo2021decentralized}, mostly for learning in two-player zero-sum Markov games. Several works have investigated zero-sum games in a \emph{decentralized} setting as we consider here \citep{daskalakis2020independent,tian2020provably,wei2021last,sayin2021decentralized},  but these results do not carry over in any way to general-sum games or MPGs. We refer the reader to Appendix~\ref{app:related} for a more detailed discussion on these related works. 

For general-sum games, \citet{rubinstein2016settling} has shown a sample complexity lower bound for learning NE that is exponential in the number of agents. Recently, \citet{liu2020sharp} has presented a line of results on learning NE, CE, or CCE, but their algorithm is model-based, and suffers from such exponential dependence. \citet{song2021can,jin2021v,mao2022provably} have proposed V-learning based methods for learning CCE and/or CE, and our stage-based V-learning significantly simplifies the algorithmic design and analysis along this line. Learning CE and CCE has also been extensively studied in normal-form games with no state transitions~\citep{hart2000simple,cesa2006prediction,blum2007external}.

Another line of research~\citep{macua2018learning,mguni2021learning} has considered learning in Markov potential games. \citet{arslan2016decentralized} has shown that decentralized Q-learning can converge to NE in weakly acyclic games, which cover potential games as a special case. Their algorithm requires a coordinated exploration phase, and only yields asymptotic guarantees. Two recent works \citep{zhang2021gradient,leonardos2021global} have proposed independent policy gradient methods in MPGs, which are most relevant to ours. We improve their sample complexity dependence on $\epsilon$ by utilizing decentralized variance reduction, and we do not require the two-timescale framework to coordinate policy evaluation as in \citet{zhang2021gradient}. \citet{fox2021independent} has shown that independent natural policy gradient also converges to NE, though only  asymptotic convergence has been established. Finally, MPGs have also been studied in \citet{song2021can}, but their model-based method is not decentralized, and requires the agents to take turns to learn the policies.

\section{Preliminaries}\label{sec:preliminaries}
An $N$-player episodic Markov game is defined by a tuple $(\mc{N}, H, \mc{S},\{\mc{A}_i\}_{i=1}^N, \{r_i\}_{i=1}^N, P)$, where (1) $\mc{N} = \{1,2,\dots,N\}$ is the set of agents; (2) $H\in\mb{N}_+$ is the number of time steps in each episode; (3) $\mc{S}$ is the finite state space; (4) $\mc{A}_i$ is the finite action space for agent $i\in\mc{N}$; (5) $r_i:[H]\times \mc{S}\times \mc{A} \ra [0,1]$ is the reward function for agent $i$, where $\mc{A} = \times_{i=1}^N \mc{A}_i$ is the joint action (or action profile) space; and (6) $P: [H]\times \mc{S}\times \mc{A} \ra \Delta(\mc{S})$ is the transition kernel. We remark that both the reward function and the state transition function depend on the joint actions of all the agents. We assume for simplicity that the reward function is deterministic. Our results can be easily generalized to stochastic reward functions. Let $S = |\mc{S}|$, $A_i = |\mc{A}_i|,\forall i\in\mc{N}$, and $A_{\max} = \max_{i\in\mc{N}}A_i$.

The agents interact in an unknown environment for $K$ episodes. We assume that the initial state $s_1$ of the environment follows a fixed distribution $\rho\in\Delta(\mc{S})$. At each time step $h\in[H]$, the agents observe the state $s_h \in \mc{S}$, and take actions $a_{h,i} \in\mc{A}_i, i\in\mc{N}$ simultaneously.  Agent $i$ then receives its private reward $r_{h,i}(s_h,\bm{a}_h)$, where $\bm{a}_h =  (a_{h,1},\dots, a_{h,N}) $, and the environment transitions to the next state $s_{h+1}\sim P_h(\cdot | s_h,\bm{a}_h)$. Note that the state transition here is general and not restricted to be deterministic. This makes decentralized learning considerably more challenging, as the agents cannot implicitly coordinate by enumerating/rehearsing all possible states. We focus on the \emph{decentralized} setting, where each agent only observes the states and its own rewards and actions, but not the rewards or actions of the other agents. In fact, in our algorithms, each agent is completely oblivious of the existence of the others, and does not communicate with each other. This decentralized information structure requires each agent to learn to make decisions based on only its local information. 

\vspace{.8em}
\noindent\textbf{Policy and value function.} A (Markov) policy $\pi_i:[H]\times \mc{S}\ra \Delta(\mc{A}_i)$ for agent $i\in\mc{N}$ is a mapping from the time index and state space to a distribution over its own action space. We use $\Pi_i$ to denote the space of Markov policies for agent $i$, and let $\Pi = \times_{i=1}^N \Pi_i$. Each agent seeks to find a policy that maximizes its own cumulative reward.  A joint policy (or policy profile) $\pi = (\pi_1,\dots,\pi_N)$ induces a probability measure over the sequence of states and joint actions. For notational convenience, we use the subscript $-i$ to denote the set of agents excluding agent $i$, i.e., $\mc{N}\backslash \{i\}$. For example, we can rewrite $\pi = (\pi_i, \pi_{-i})$ using this convention. For a policy profile $\pi$, and for any $h\in[H]$, $s \in \mc{S}$, and $a\in \mc{A}$, we define the value function and the state-action value function (or $Q$-function) for agent $i$ as follows: 
\begin{align}
V_{h,i}^\pi (s) &\defeq \ee_\pi \bigg[ \sum_{h'=h}^{H}r_{h',i}(s_{h'},\bm{a}_{h'})\mid s_h = s \bigg],\label{eqn:value} \\
Q_{h,i}^\pi (s,\bm{a}) &\defeq \ee_\pi \bigg[ \sum_{h'=h}^{H}r_{h',i}(s_{h'},\bm{a}_{h'})\mid s_h = s, \bm{a}_h = \bm{a} \bigg]. \nonumber
\end{align}
For ease of notation, we also write $V_{h,i}^{(\pi_i,\pi_{-i})}(s)$ as $V_{h,i}^{\pi_i,\pi_{-i}}(s)$, and similarly for $Q_{h,i}^{(\pi_i,\pi_{-i})}(s,a)$.

\vspace{.8em}
\noindent\textbf{Best response and Nash equilibrium.}
 For agent $i$, a policy $\pi_i^{\star}$ is a \emph{best response} to $\pi_{-i}$ for a given initial state $s_1$ if $V_{1,i}^{\pi_i^{\star}, \pi_{-i}}(s_1) = \sup_{\pi_i} V_{1,i}^{\pi_{i }, \pi_{-i}}(s_1)$. A policy profile $\pi = (\pi_{i}, \pi_{-i })\in\Pi $ is a \emph{Nash equilibrium} (NE) if $\pi_{i }$ is a best response to $\pi_{-i }$ for all $i\in\mc{N}$.  We also have an approximate notion of Nash equilibrium as follows:

\begin{definition}\label{def:approxNE} ($\epsilon$-approximate Nash equilibrium).
	For any $\epsilon>0$, a policy profile $\pi = (\pi_{i}, \pi_{-i })\in\Pi$ is an $\epsilon$-approximate Nash equilibrium for an initial state $s_1$ if $V_{1,i}^{\pi_{i }, \pi_{-i}}(s_1) \geq \sup_{\pi_{i'}} V_{1,i}^{\pi_{i'}, \pi_{-i}}(s_1) - \epsilon$, $\forall i\in\mc{N}$. 
\end{definition}

\vspace{.8em}
\noindent\textbf{Markov potential game.} One particular subclass of games that we are interested in is the Markov potential game. Specifically, an episodic Markov game is an MPG if there exists a global potential function $\Phi_s: \Pi \ra [0, \Phi_{\max}]$ for every initial state $s\in\mc{S}$, such that for any $i\in\mc{N}$, any $\pi_i,\pi_{i'}\in\Pi_i$, and any $\pi_{-i}\in\Pi_{-i}$,
\begin{equation}\label{eqn:potential}
\Phi_s(\pi_i,\pi_{-i}) - \Phi_s(\pi_{i'},\pi_{-i}) = V_{1,i}^{\pi_i,\pi_{-i}}(s) - V_{1,i}^{\pi_{i'},\pi_{-i}}(s).
\end{equation}
Our definition of MPG follows \citet{song2021can}, which in turn is a variant of the definitions introduced in \citet{macua2018learning,leonardos2021global,zhang2021gradient}. It follows immediately that MPGs cover Markov teams~\citep{lauer2000algorithm} as a special case, a cooperative setting where all agents share the same reward function. 

\vspace{.8em}
\noindent\textbf{Correlated policy.}
More generally, we define $\pi =\{\pi_{h}: \rr \times (\mc{S}\times \mc{A})^{h-1}\times\mc{S}\ra \Delta(\mc{A})\}_{h\in[H]}$ as a (non-Markov) \emph{correlated policy}, where for each $h\in[H]$, $\pi_{h}$ maps from a random variable $z\in\rr$ and a history of length $h-1$ to a distribution over the joint action space. We assume that the agents following a correlated policy can access a common source of randomness (e.g., a common random seed) for the random variable $z$. We let $\pi_i$ and $\pi_{-i}$ be the proper marginal policies of $\pi$ whose outputs are restricted to $\Delta(\mc{A}_i)$ and $\Delta(\mc{A}_{-i})$, respectively. 

For non-Markov correlated policies, we can still define their value functions at step $h=1$ in a sense similar to~\eqref{eqn:value}. A best response $\pi_i^{\star}$ with respect to the non-Markov policies $\pi_{-i}$ is a policy (independent of the randomness of  $\pi_{-i}$) that maximizes agent $i$'s value at step 1, i.e., $V_{1,i}^{\pi_i^{ \star}, \pi_{-i}}(s_1) = \sup_{\pi_i} V_{1,i}^{\pi_{i }, \pi_{-i}}(s_1)$. The best response to the non-Markov policies of the opponents is not necessarily Markov. 

\vspace{.8em}
\noindent\textbf{(Coarse) correlated equilibrium.}
Given the PPAD-hardness of calculating Nash equilibria in general-sum games~\citep{daskalakis2009complexity}, we introduce two relaxed solution concepts, namely coarse correlated equilibrium (CCE) and correlated equilibrium (CE). A CCE states that no agent has the incentive to deviate from a correlated policy $\pi$ by playing a different independent policy.

\begin{algorithm*}[!tbp]
	\textbf{Initialize:} $\overline{V}_{h,i}(s) \gets H-h+1, \ot{V}_{h,i}(s) \gets H-h+1, N_h(s)\gets 0, \check{N}_h(s)\gets 0, \check{r}_{h,i}(s)\gets 0,\allowbreak \check{v}_{h,i}(s)\gets 0,\allowbreak \check{T}_h(s)\gets H, \mu_{h,i}(a\mid s)\gets 1/A_i$, and $L_{h,i}(s,a)\gets 0$, $\forall h\in[H],s\in\mc{S},a\in\mc{A}_i$. 
	
	\For{episode $k\gets 1$ to $K$}
	{
		Receive $s_1$\;
		\For{step $h\gets 1$ to $H$}
		{
			$ N_h(s_h) \gets N_h(s_h) + 1,  \check{n}\defeq \check{N}_h(s_h) \gets \check{N}_h(s_h) + 1$\;
			Take action $a_{h,i} \sim \mu_{h,i}(\cdot \mid s_h)$, and observe reward $r_{h,i}$ and next state $s_{h+1}$\;
			$\check{r}_{h,i}(s_h)\gets \check{r}_{h,i}(s_h) + r_{h,i}, \check{v}_{h,i}(s_h)\gets \check{v}_{h,i}(s_h) + \overline{V}_{h+1,i}(s_{h+1}) $\;
			$\eta_i \gets \sqrt{\iota/A_i \check{T}_h(s_h)}, \gamma_i \gets \eta_i / 2$\;
			$L_{h,i}(s_h,a_{h,i}) \gets L_{h,i}(s_h,a_{h,i})+ \frac{[H-h+1-(r_{h,i} + \overline{V}_{h+1,i}(s_{h+1}))]/H}{\mu_{h,i}(a_{h,i} \mid s_h)+\gamma_i}$\label{line:9}\;
			$\mu_{h,i}(a\mid s_h)\gets \frac{\exp(-\eta_i L_{h,i}(s_h,a))}{\sum_{a'\in\mc{A}_i} \exp(-\eta_i L_{h,i}(s_h,a')) },\forall a\in\mc{A}_i$\label{line:10}\;
			\If{$N_h(s_h)\in\mc{L}$\label{line:11}}
			{
				\texttt{//Entering a new stage}
				
				$\ot{V}_{h,i}(s_h)\gets  \frac{\check{r}_{h,i}(s_h)}{\check{n}} + \frac{\check{v}_{h,i}(s_h)}{\check{n}}+b_{\check{n}}$, where 
				$b_{\check{n}}\gets 6\sqrt{H^2 A_i\iota / \check{n}}$\label{line:1}\;
				$\ol{V}_{h,i}(s_h)\gets \min\{\ot{V}_{h,i}(s_h),H-h+1\}$\;
				$\check{N}_h(s_h)\gets 0, \check{r}_{h,i}(s_h)\gets 0, \check{v}_{h,i}(s_h)\gets 0, \check{T}_h(s_h)\gets \floor{(1+\frac{1}{H})\check{T}_h(s_h)}$\;
				$\mu_{h,i}(a\mid s_h)\gets 1/A_i, L_{h,i}(s_h,a)\gets 0, \forall a\in\mc{A}_i$\label{line:16}\;
			}
		}
	}
	\caption{Stage-Based V-Learning for CCE (agent $i$)}\label{alg:sbv}
\end{algorithm*}

\begin{definition}\label{def:CCE}
	(CCE). A correlated policy $\pi$ is an $\epsilon$-approximate coarse correlated equilibrium for an initial state $s_1$ if
	$
	V_{1,i}^{\pi_i^{ \star}, \pi_{-i}}(s_1) - V_{1,i}^{\pi}(s_1) \leq \epsilon,\forall i\in\mc{N}.
	$
\end{definition}

CCE relaxes NE by allowing possible correlations in the policies. Before introducing the definition of CE, we need to first specify the concept of a strategy modification. 

\begin{definition}
	(Strategy modification). For agent $i$, a strategy modification $\psi_i = \{\psi_{h,i}^s:h\in[H],s\in\mc{S}\}$ is a set of mappings from agent $i$'s action space to itself, i.e., $\psi_{h,i}^s:\mc{A}_i\ra \mc{A}_i$. 
\end{definition}

Given a strategy modification $\psi_i$, for any policy $\pi$, step $h$ and state $s$, if $\pi$ selects the joint action $\bm{a}_h = (a_{h,1},\dots,a_{h,N})$, then the modified policy $\psi_i \diamond \pi$ will select $(a_{h,1},\dots,a_{h,i-1}, \psi_{h,i}^s(a_{h,i}),a_{h,i+1},\dots,a_{h,N})$. 
Let $\Psi_i$ denote the set of all possible strategy modifications for agent $i$. A CE is a distribution where no agent has the incentive to deviate from a correlated policy $\pi$ by using any strategy modification. It is known that \{NE\}$\subset$\{CE\}$\subset$\{CCE\} in general-sum games~\citep{nisan2007algorithmic}. 

\begin{definition}
	(CE). A correlated policy $\pi$ is an $\epsilon$-approximate correlated equilibrium for initial state $s_1$ if 
	\[
	\sup_{\psi_i\in \Psi_i} V_{1,i}^{\psi_i\diamond \pi}(s_1) - V_{1,i}^{\pi}(s_1) \leq \epsilon,\forall i\in\mc{N}.
	\]
\end{definition}

\section{Stage-Based V-Learning for General-Sum Markov Games}\label{sec:cce_ce}

In this section, we introduce our stage-based V-learning algorithms for learning CCE and CE in general-sum Markov games, and establish their sample complexity guarantees.

\subsection{Learning CCE}\label{subsec:cce}
The Stage-Based V-Learning for CCE algorithm run by agent $i\in\mc{N}$ is presented in Algorithm~\ref{alg:sbv}. The agent maintains upper confidence bounds on the value functions to actively explore the unknown environment, and uses a stage-based rule to independently update the value estimates. 

For each step-state pair $(h,s)\in[H]\times\mc{S}$, we divide the visitations to this pair into multiple \emph{stages}, where the lengths of the stages increase exponentially at a rate of $(1+1/H)$ \citep{zhang2020almost}. Specifically, we let $e_1 = H$, and $e_{i+1} = \lfloor (1+1/H) e_i\rfloor, i\geq 1$ denote the lengths of the stages, and let the partial sums $\mc{L}\defeq \{ \sum_{i=1}^{j} e_i \mid j = 1,2,3,\dots \}$ denote the set of ending times of the stages.  For each $(h,s)$ pair, we update our optimistic estimates $\ol{V}_h(s_h)$ of the value function at the end of each stage (i.e., when the total number of visitations to $(s,h)$ lies in the set $\mc{L}$), using samples only from this single stage (Lines \ref{line:11}-\ref{line:16}). This way, our stage-based V-learning ensures that only the most recent $O(1/H)$ fraction of the collected samples are used to calculate $\ol{V}_h(s_h)$, while the first $1-O(1/H)$ fraction is forgotten. Such a stage-based update framework in some sense mimics the celebrated optimistic Q-learning algorithm with a learning rate of $\alpha_t = \frac{H+1}{H+t}$ \citep{jin2018q}, which also roughly uses the last $O(1/H)$ fraction of samples for value updates. Stage-based value updates also create a stage-wise
stationary environment for the agents, thereby partly alleviating the well-known challenge of \emph{non-stationarity} in MARL. As a side remark, stage-based Q-learning has also achieved near-optimal regret bounds in single-agent RL~\citep{zhang2020almost}.

At each time step $h$ and state $s_h$, agent $i$ selects its action $a_{h,i}$ by following a distribution $\mu_{h,i}(\cdot \mid s_h)$, where $\mu_{h,i}(\cdot \mid s_h)$ is updated using an adversarial bandit subroutine (Lines~\ref{line:9}-\ref{line:10}). This is consistent with the recent works under the V-learning framework~\citep{jin2021v,song2021can,mao2022provably}, but with a vital improvement: Existing works using the celebrated $\alpha_t = \frac{H+1}{H+t}$ learning rate for V-learning inevitably entail a no-\emph{weighted}-regret bandit problem, because such a time-varying learning rate assigns different weights to each step in the history. A few methods such as weighted follow-the-regularized-leader~\citep{jin2021v,song2021can} and stabilized online mirror descent~\citep{mao2022provably} have been recently proposed to address such a challenge, by simultaneously dealing with a changing step size, a weighted regret, and a high-probability guarantee, at the cost of less natural algorithms and more sophisticated analyses. In contrast, our stage-based V-learning assigns uniform weights to each step in the previous stage, and hence leads to a standard no(-average)-regret bandit problem. This allows us to directly plug in any off-the-shelf adversarial bandit algorithm and its analysis to our problem. For example, Algorithm~\ref{alg:sbv} utilizes a simple Exp3~\citep{auer2002nonstochastic} subroutine for policy updates, and a standard implicit exploration technique~\citep{neu2015explore} to achieve high-probability guarantees. We provide a more detailed discussion on such an improvement in Remark~\ref{rmk:learning_rate} of Appendix~\ref{app:cce}.

Based on the policy trajectories from Algorithm~\ref{alg:sbv}, we construct an output policy profile $\bar{\pi}$ that we will show is a CCE. For any step $h\in[H]$ of an episode $k\in[K]$ and any state $s\in\mc{S}$, we let $\mu_{h,i}^k(\cdot \mid s)\in\Delta(\mc{A}_i)$ be the distribution prescribed by Algorithm~\ref{alg:sbv} at this step.  Let $\check{N}_{h}^{k}(s)$ denote the value of $\check{N}_{h}(s)$ at the \emph{beginning} of the $k$-th episode. Our construction of the output policy is presented in Algorithm~\ref{alg:certify}, which follows the ``certified policies'' introduced in~\citet{bai2020near}. We further let the agents sample the episode indices using a common random seed\footnote{Such common randomness is also termed a correlation device, and is standard in decentralized learning \citep{bernstein2009policy,arabneydi2015reinforcement,zhang2019online}. Note that the correlation device is never used during the learning process to coordinate the exploration, but is simply used to synchronize the selection of the policies after they have been generated. A common random seed is generally considered as a mild assumption and does not break the decentralized paradigm. }, and hence the output policy is correlated by nature. Note that our stage-based update rule also simplifies the generating procedure of the output policy: In the original construction of \citet{bai2020near}, the certified policy plays a weighted mixture of $\{\mu_{h,i}^k(\cdot \mid s):k\in[K]\}$, while in Algorithm~\ref{alg:certify}, we only need to uniformly sample an episode index from the previous stage.

\begin{algorithm}[!tbp]
	\textbf{Input:} The distribution trajectory  specified by Algorithm~\ref{alg:sbv}: $\{\mu_{h,i}^k:i\in\mc{N},h\in[H],k\in[K]\}$;
	
	Uniformly sample $k$ from $[K]$\;
	\For{step $h\gets 1$ to $H$}
	{
		Receive $s_{h}$\;
		Take joint action $\bm{a}_{h}\sim \times_{i=1}^N\mu_{h,i}^{k}(\cdot \mid s_{h})$\;
		Uniformly sample $j$ from $\{1,2,\dots, \check{N}_{h}^{k}(s_{h})\}$\;
		Set $k\gets \check{l}_{h,j}^{k'}$, where  $\check{l}^{k}_{h,j}$ is the index of the episode such that state $s_{h}$ was visited the $j$-th time (among the total $\check{N}_{h}^{k}(s_{h})$ times) in the last stage\; 
	}
	\caption{Construction of the Output Policy $\bar{\pi}$}\label{alg:certify}
\end{algorithm}

The following theorem presents the sample complexity guarantee of Algorithm~\ref{alg:sbv} for learning CCE in general-sum Markov games. Our sample complexity bound improves over \citet{mao2022provably} and matches those established in \citet{song2021can,jin2021v}, while significantly simplifying their algorithmic design and analysis. The proof is deferred to Appendix~\ref{app:cce} due to space limitations.

\begin{theorem}\label{thm:cce}
	(Sample complexity of learning CCE). For any $p\in (0,1]$, set $\iota = \log(2NSA_{\max} KH/p)$, and let the agents run Algorithm~\ref{alg:sbv} for $K$ episodes with $K= O(S A_{\max}H^5 \iota/\epsilon^2)$. Then, with probability at least $1-p$, the output policy $\bar{\pi}$ of Algorithm~\ref{alg:certify} is an $\epsilon$-approximate CCE. 
\end{theorem}

\subsection{Learning CE}\label{subsec:ce}
In this subsection, we aim at learning a more strict solution concept named correlated equilibrium. Our algorithm for learning CE (a complete description presented in Algorithm~\ref{alg:ce} of Appendix~\ref{app:ce}) also relies on stage-based V-learning, but replaces the no-regret learning subroutine in Algorithm~\ref{alg:sbv} with a no-swap-regret learning algorithm. Our no-swap-regret algorithm follows the generic reduction introduced in~\citet{blum2007external}, and converts a follow-the-regularized-leader (FTRL) algorithm with sublinear external regret to a no-swap-regret algorithm \citep{jin2021v}. A detailed description of such a no-swap-regret FTRL subroutine as well as its regret analysis is presented in Appendix~\ref{app:ce}. Again, due to the stage-based update rule, we can avoid the additional complication of dealing with a weighted swap regret as faced by recent works~\citep{jin2021v,song2021can}. The construction of the output policy $\bar{\pi}$ is the same as Algorithm~\ref{alg:certify} and thus omitted. The following theorem shows that our sample complexity guarantee for learning CE improves over \citet{song2021can} and matches the best known result in the literature~\citep{jin2021v}. The proof of the theorem can also be found in Appendix~\ref{app:ce}.

\begin{theorem}\label{thm:ce}
	(Sample complexity of learning CE). For any $p\in (0,1]$, set $\iota = \log(2NSA_{\max} KH/p)$, and let the agents run Algorithm~\ref{alg:ce} for $K$ episodes with $K= O(S A_{\max}^2H^5 \iota/\epsilon^2)$. Then, with probability at least $1-p$, the output policy $\bar{\pi}$ is an $\epsilon$-approximate CE. 
\end{theorem}

As a final remark, notice that both the V-learning and the no-regret learning components of our algorithms are decentralized, which can be implemented using only the states observed and the local action and reward information, without any communication or central coordination among the agents. In addition, the sample complexity of our algorithms only depend on $A_{\max}$ instead of $\prod_{i=1}^N A_i$. This allows our methods to easily generalize to a large number of agents.

\section{Learning NE in Markov Potential Games}\label{sec:mpg}
In this section, we present an algorithm for learning Nash equilibria in decentralized Markov potential games, an important subclass of Markov games. Motivated by~\citet{leonardos2021global,zhang2021gradient}, we utilize a policy gradient method, where each agent independently runs a projected gradient ascent (PGA) algorithm to update their policies. We start from the case where the policy gradients can be calculated exactly (using an infinite number of samples), and then move to the more practical case where the gradients are estimated using finite samples. 

We first introduce a few notations for ease of presentation. Let $d_{h,\rho}^\pi(s)$ be the probability of visiting state $s$ at step $h$ by following policy $\pi$ starting from the initial state distribution $\rho$, i.e., 
$d_{h,\rho}^\pi(s) \defeq \pp^\pi(s_h = s\mid s_1\sim \rho).$ 
We also overload the notations of the value function and the potential function, and write $V_{1,i}^\pi(\rho) \defeq \ee_{s_1\sim\rho}\L V_{1,i}^\pi(s_1) \R$ and $\Phi_\rho(\pi) = \ee_{s\sim\rho}\L \Phi_{s}(\pi) \R$. We further introduce the following variant of the distribution mismatch coefficient~\citep{agarwal2021theory} to characterize the difficulty of exploration.  
\begin{definition}
	(Finite-horizon distribution mismatch coefficient). Given two policies $\pi,\pi'\in\Pi$ and an initial state distribution $\rho\in\Delta(\mc{S})$, we define
	$$\norm{\frac{d_{\rho}^{\pi'}}{d_{\rho}^{\pi}}}_{\infty} \defeq \max_{h\in[H],s_h\in\mc{S}} \frac{d_{h,\rho}^{\pi'}(s_h)}{d_{h,\rho}^{\pi}(s_h)},\text{ and } D \defeq \max_{\pi,\pi'} \norm{\frac{d_{\rho}^{\pi'}}{d_{\rho}^{\pi}}}_{\infty}. $$
\end{definition}

\subsection{Exact Gradient Estimates}\label{subsec:exact}
The PGA algorithm updates the policy as follows:
\begin{equation}\label{eqn:pga}
\pi_i^{(t+1)} \gets \text{Proj}_{\Pi_i} \l \pi_i^{(t)} + \eta \grad_{\pi_i}V_{1,i}^{\pi^{(t)}}(\rho) \r, 
\end{equation}
where $\pi_i^{(t)}$ is the policy of agent $i$ at the $t$-th iteration, $\text{Proj}_{\Pi_i}$ denotes the Euclidean projection onto $\Pi_i$, and $\eta>0$ is the step size. Here, we use the direct parameterization of the policy~\citep{agarwal2021theory}, where $\pi_{h,i}(a\mid s) = \theta_{h,i}^{s,a}$ for some $\theta_{h,i}^{s,a}\geq 0$ and $\sum_{a\in\mc{A}_i}\theta_{h,i}^{s,a} = 1,\forall i\in\mc{N},h\in[H],s\in\mc{S},a\in\mc{A}_i$.  We assume for now that the policy gradients $\grad_{\pi_i}V_{1,i}^{\pi^{(t)}}(\rho)$ can be calculated exactly, and such an assumption will be relaxed in the next subsection. 

Before presenting the analysis of PGA, we first introduce the following definition of an approximate stationary point.

\begin{definition}\label{def:stationary_point}
	For any $\epsilon>0$, a policy profile $\pi = (\pi_1,\dots,\pi_N)$ is a (first-order) $\epsilon$-approximate stationary point of a function $\Phi_\rho:\Pi\ra[0,\Phi_{\max}]$ if for any $\delta_1\in \rr^{A_1},\dots,\delta_N\in\rr^{A_N}$, such that $\sum_{i\in\mc{N}}\norm{\delta_i}_2^2 \leq 1$ and $\pi_i + \delta_i \in \Delta(\mc{A}_i),\forall i\in\mc{N}$, it holds that
	\[
	\sum_{i\in\mc{N}} \delta_i\T \grad_{\pi_i} \Phi_\rho(\pi) \leq \epsilon.
	\]
\end{definition}

Intuitively, $\pi$ is an approximate stationary point if the function $\Phi_\rho$ cannot increase by more than $\epsilon$ along any direction that lies in the intersection of the policy space and the neighborhood of $\pi$. The following lemma establishes the equivalence between stationary points and NE. 

\begin{lemma}\label{lemma:stationarity_implies_nash}
	Let $\pi=(\pi_1,\dots,\pi_N)$ be an $\epsilon$-approximate stationary point of the potential function $\Phi_\rho$ of an MPG for some $\epsilon>0$. Then, $\pi$ is a $D\sqrt{SH}\epsilon$-approximate NE.
\end{lemma}

The proof of Lemma~\ref{lemma:stationarity_implies_nash} relies on a gradient domination property that has been shown in single-agent RL~\citep{agarwal2021theory}. Its multi-agent counterpart has been studied in \citet{zhang2021gradient,leonardos2021global}, though to the best of our knowledge, a gradient domination property in finite-horizon episodic MDPs/MPGs is still missing in the literature. For completeness, we derive such a result, together with the finite-horizon variants of the policy gradient theorem~\citep{sutton2000policy} and performance difference lemma~\citep{kakade2002approximately} in Appendix~\ref{app:mpg}. With the above results, we arrive at the convergence guarantee of PGA in the exact gradient case. The proof of Theorem~\ref{thm:mpg_exact_gradient} is also deferred to Appendix~\ref{app:mpg}. 

\begin{theorem}\label{thm:mpg_exact_gradient}
	For any initial state distribution $\rho\in\Delta(\mc{S})$, let the agents independently run the projected gradient ascent updates~\eqref{eqn:pga} with a step size $\eta = \frac{1}{4NA_{\max}H^3}$ for $T = \frac{32NSA_{\max}D^2H^4\Phi_{\max}}{\epsilon^2}$ iterations. Then, there exists $t\in[T]$, such that $\pi^{(t)}$ is an $\epsilon$-approximate Nash equilibrium policy profile for the MPG. 
\end{theorem}

\subsection{Finite-Sample Gradient Estimates}\label{subsec:finite}

When the exact policy gradients are not given, we need to replace $\grad_{\pi_i}V_{1,i}^{\pi^{(t)}}(\rho)$ in \eqref{eqn:pga} with an estimate $\hat{\grad}_{\pi_i}^{(t)}(\pi^{(t)})$ that is calculated from a finite number of samples. For any policy $\pi$ used in the $t$-iteration of PGA, we use a standard REINFORCE~\citep{williams1992simple} gradient estimator
\begin{equation}\label{eqn:gradient_estimator}
\hat{\grad}^{(t)}_{\pi_i}(\pi) = R_i^{(t)}\sum_{h=1}^H\grad \log\pi_{h,i}(a_{h,i}^{(t)}\mid s_h^{(t)}),
\end{equation}
where $R_i^{(t)} = \sum_{h=1}^H r_{h,i}^{(t)}(s_h^{(t)}, \bm{a}_h^{(t)})$ is the sum of rewards obtained at iteration $t$, and $s_1^{(t)}\sim\rho$. 

To ensure that the variance of the gradient estimator is bounded, we let each agent use an epsilon-greedy variant of direct policy parameterization. Specifically, each agent $i$ selects its actions according to a policy $\pi_i$, such that $\pi_{h,i}(a\mid s) = (1-\tilde{\epsilon})\theta_{h,i}^{s,a} + \tilde{\epsilon}/A_i$, where $\theta_{h,i}^{s,a}\geq 0$, $\sum_{a\in\mc{A}_i}\theta_{h,i}^{s,a}=1$, and $\tilde{\epsilon}>0$ is the exploration parameter. In the following lemma, we show that the gradient estimator \eqref{eqn:gradient_estimator} under $\tilde{\epsilon}$-greedy exploration is unbiased, has a bounded variance, and is mean-squared smooth. The first two properties have appeared in~\citet{daskalakis2020independent,leonardos2021global}, while the third property is new and is used to derive an improved sample complexity result in our analysis. 

\begin{lemma}\label{lemma:unbiased_bounded}
	For any agent $i\in\mc{N}$ and any iteration $t\in[T]$, the REINFORCE gradient estimator \eqref{eqn:gradient_estimator} with $\tilde{\epsilon}$-greedy exploration is an unbiased estimator with a bounded variance:
	\[
	\begin{aligned}
	&\ee_{\pi^{(t)}}\L \hat{\grad}_{\pi_i}^{(t)}(\pi^{(t)}) \R = \grad_{\pi_i}V_{1,i}^{\pi^{(t)}}(\rho), \\
	&\ee_{\pi^{(t)}} \norm{\hat{\grad}_{\pi_i}^{(t)}(\pi^{(t)}) - \grad_{\pi_i}V_{1,i}^{\pi^{(t)}}(\rho)}_2^2  \leq \frac{A_{\max}^2 H^4}{\tilde{\epsilon}}.
	\end{aligned}
	\]
	Further, it is mean-squared smooth, i.e., for any $\pi^{\prime(t)}\in\Pi_i$,
	\[
	\ee_{\pi^{(t)}} \norm{\hat{\grad}_{\pi_i}^{(t)}(\pi^{(t)}) - \hat{\grad}_{\pi_i'}^{(t)}(\pi^{\prime(t)})}_2^2  \leq \frac{A^3_{\max}H^3}{\tilde{\epsilon}^3} \|\pi^{(t)} - \pi^{\prime(t)}\|_2^2.
	\]
\end{lemma}

Each agent now runs (projected) stochastic gradient ascent (SGA) to update its policy, where the gradient estimator is given by \eqref{eqn:gradient_estimator}. In the following, we present the analysis of a generic stochastic gradient descent method that might be of independent interest, and the SGA policy update rule is simply an instantiation of such a generic method.

Consider a generic stochastic non-convex optimization problem as follows: We are given an objective function $F:\rr^n \ra \rr$, and our goal is to find a point $x\in\mc{X}\subseteq \rr^n$ such that $\grad F(x)$ is close to $0$, where $\mc{X}$ is the feasible region. We do not have accurate information about the function $F$, and can only access it through a stochastic sampling oracle $f(\cdot, \xi)$, where the random variable $\xi$ represents the ``randomness'' of the oracle. We introduce the following assumptions that are standard in smooth non-convex optimization~\citep{arjevani2019lower}. 

\begin{assumption}\label{assumption:1}
		1. We have access to a stream of random variables $\xi_1,\dots,\xi_T$, such that the gradient estimators are unbiased and have bounded variances: $\grad \ee_{\xi_t}[f(x,\xi_t)] =\grad F(x)$, and $\ee[\norm{\grad f(x,\xi_t) - \grad F(x)}^2_2]\leq \sigma^2$ for some $\sigma>0$ for all $t\in[T]$ and $x\in\mc{X}$. 
		
		2. The objective $F$ has bounded initial sub-optimality and is $L$-smooth: $F(x_0) - \inf_{x\in\mc{X}}F(x) < \infty$, and $\norm{\grad F(x)- \grad F(y)}_2\leq L \cdot \norm{x-y}_2, \forall x,y\in \rr^n$ for some $L>0$. The stochastic oracle is mean-squared smooth for the same constant $L$: $\ee[\norm{\grad f(x,\xi) - \grad f(y,\xi)}^2_2] \leq L^2 \cdot \norm{x-y}^2_2,\forall x,y\in \rr^n$.
\end{assumption}

\begin{algorithm}[!t]
	$d_1 \gets \grad f(x_1,\xi_1)$\;	
	\For{$t\gets 1$ to $T$}
	{
		$\eta_t \gets \frac{k}{(w + \sigma^2 t)^{1/3}}$\;
		$x_{t+1}\gets \text{Proj}_{\mc{X}}(x_t - \eta_t d_t)$\;
		$a_{t+1}\gets c\eta_t^2$\;
		$d_{t+1}\gets \grad f(x_{t+1},\xi_{t+1})+ (1-a_{t+1})(d_t -\grad f(x_t,\xi_{t+1}))$;
	}
	Output $x_\tau$ where $\tau$ is uniformly sampled from $[T]$\;
	\caption{Stochastic Recursive Momentum with Projections}\label{alg:storm}
\end{algorithm}

For an improved sample complexity bound, we utilize a momentum-based stochastic gradient descent (SGD) method with variance reduction~\citep{johnson2013accelerating,allen2016variance,reddi2016stochastic}. Our method is a variant of the non-adaptive STOchastic Recursive Momentum (STORM) algorithm proposed in~\citet{cutkosky2019momentum}, and is formally described in Algorithm~\ref{alg:storm}. It achieves an optimal convergence rate of $O(1/T^{1/3})$, which improves over the standard convergence rate $O(1/T^{1/4})$ of SGD with no variance reduction (e.g., \citet{ghadimi2013stochastic}). The key advantage of this method is to apply variance reduction in a \emph{decentralized} way: Compared with other SGD methods with variance reduction (e.g., \citet{allen2016variance,reddi2016stochastic,fang2018spider}), our momentum-based algorithm does not require a batch of samples to compute checkpoint gradients. The agents hence do not need to coordinate on when to stop updating policies and to collect a batch of samples for a fixed policy profile, a common behavior when using batch-based methods. 

The following result characterizes the convergence rate of Algorithm~\ref{alg:storm}, and is a variant of the analysis given in~\citet{cutkosky2019momentum}. The proofs of Proposition~\ref{thm:storm_app} and its supporting lemmas are given in Appendix~\ref{app:storm}. 

\begin{proposition}\label{thm:storm_app}
	Suppose  Assumption~\ref{assumption:1} holds, and let $x_{t+1}^+ = \text{Proj}_{\mc{X}}(x_t - \eta_t \grad F(x_t))$. For any $b>0$, let $k = \frac{b\sigma^{\frac{2}{3}}}{L}, c=L^{2}\left(32+1 /\left(7 b^{3}\right)\right), w=\sigma^{2} \max ((4 b)^{3}, 2,(32 b+\frac{1}{7 b^{2}})^{3} / 64)$, and $M = 16(F(x_1)-\inf_{x\in\mc{X}}F(x))\allowbreak + \frac{w^{1/3}\sigma^2}{2L^2 k} + \frac{k^3 c^2}{L^2}\ln(T+2)$. Then, the following result holds for Algorithm~\ref{alg:storm}:
	\[
	\ee\L\frac{1}{T}\sum_{t=1}^T  \norm{\frac{1}{\eta_t}(x_{t+1}^+-x_{t})}^2_2 \R \leq \frac{Mw^{1/3}}{Tk} + \frac{M\sigma^{2/3}}{T^{2/3}k}.
	\]
\end{proposition}

Since we have shown in Lemma~\ref{lemma:unbiased_bounded} that the conditions in Assumption~\ref{assumption:1} are satisfied by the potential function $\Phi_{\rho}$ and the REINFORCE policy gradient estimator $\hat{\grad}_{\pi_i}^{(t)}(\pi^{(t)})$, we can let each agent run an instance of Algorithm~\ref{alg:storm} and the convergence result in Proposition~\ref{thm:storm_app} directly applies. This leads us to the following sample complexity guarantee of learning Nash equilibria in MPGs. The proof of Theorem~\ref{thm:finite} can be found in Appendix~\ref{app:mpg_finite}.

\begin{theorem}\label{thm:finite}
	For any initial policies and any $\epsilon>0$, let the agents independently run SGA policy updates (Algorithm~\ref{alg:storm}) for $T$ iterations with $T = O(1/\epsilon^{4.5})\cdot \text{poly}(N, D, S, A_{\max}, H)$. Then, there exists $t\in[T]$, such that $\pi^{(t)}$ is an $\epsilon$-approximate NE in expectation. 
\end{theorem}

\begin{figure*}[!h]
\centering
\subfigure[Matrix team]{\includegraphics[width=0.33\textwidth]{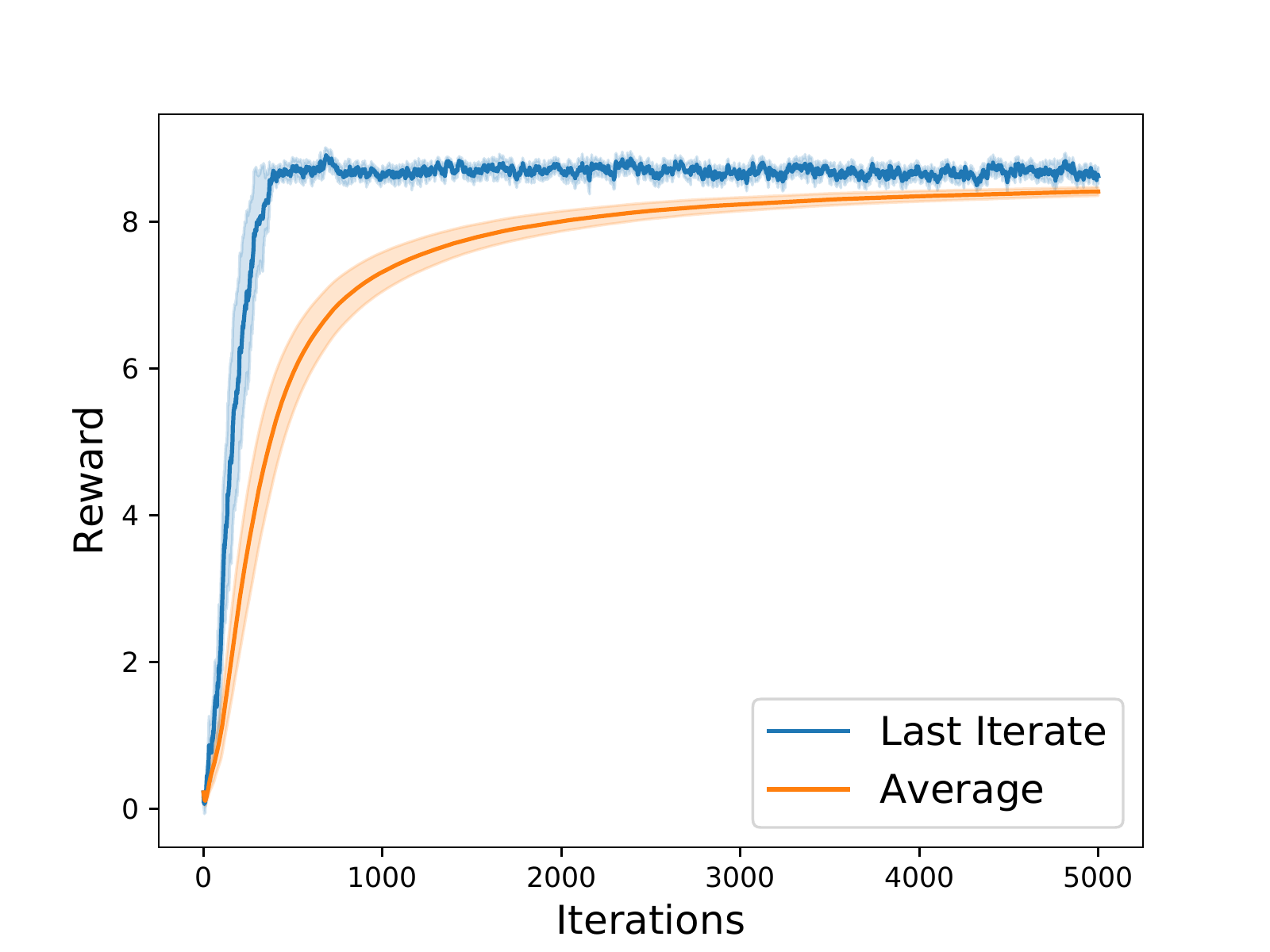}}
\hspace{-.46cm}\subfigure[GoodState]{\includegraphics[width=0.33\textwidth]{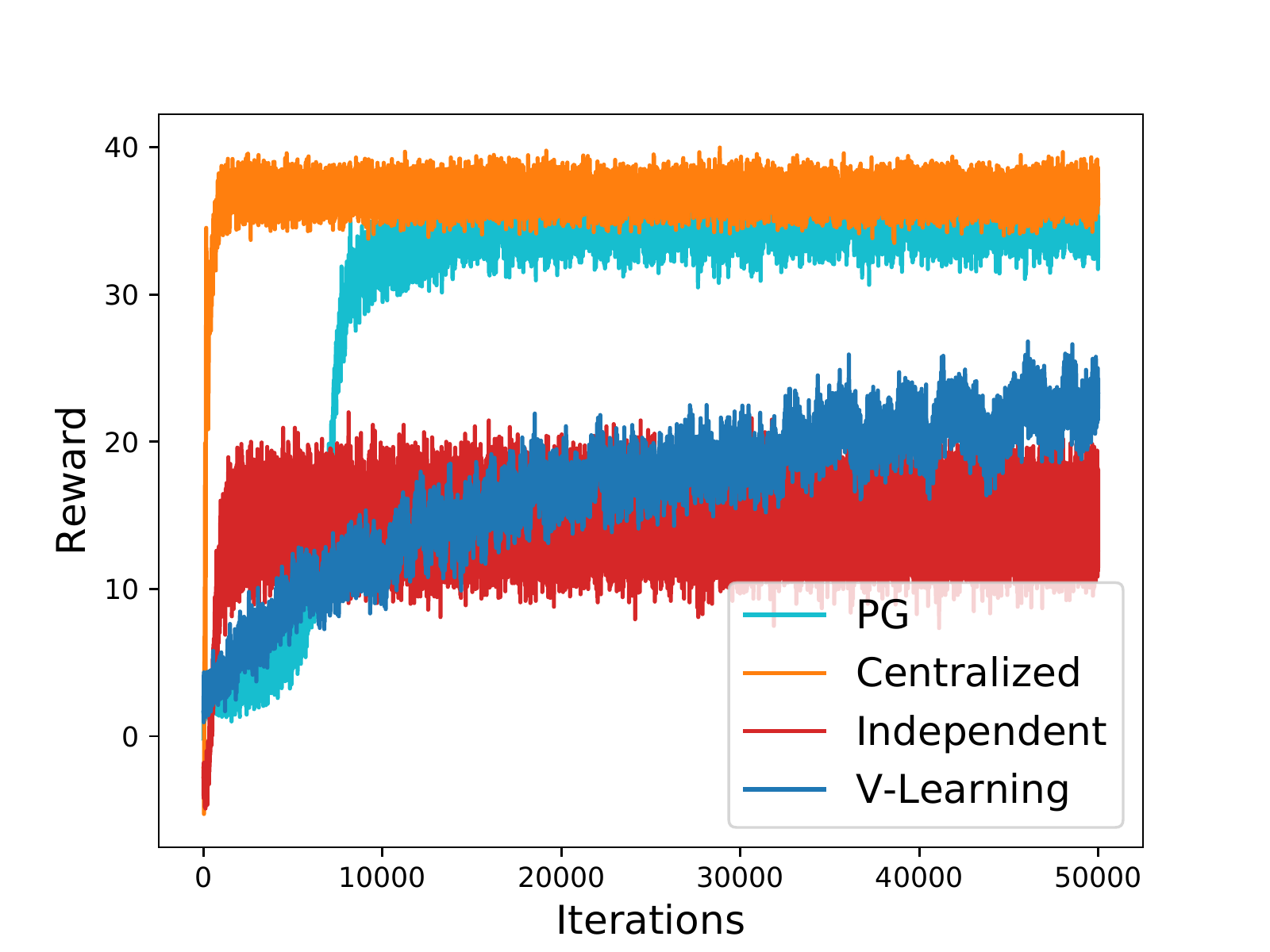}}
\hspace{-.46cm}\subfigure[BoxPushing]{\includegraphics[width=0.33\textwidth]{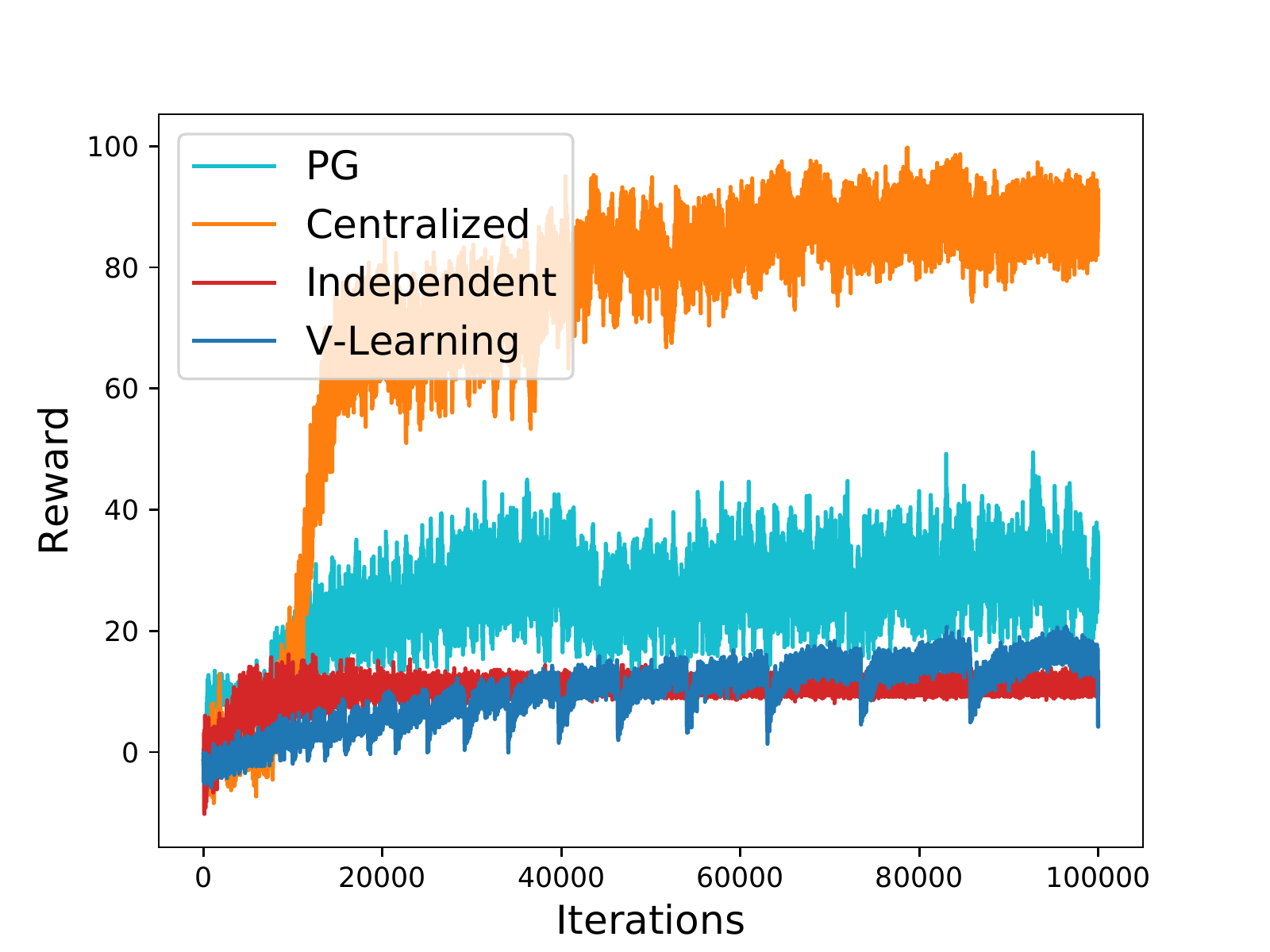}}
\caption{(a) Rewards of Algorithm~\ref{alg:storm} on the matrix team task, and rewards of Algorithms~\ref{alg:sbv} and~\ref{alg:storm} on the (b) GoodState and (c) BoxPushing tasks. ``Last Iterate'' denotes the policy of the current iterate $t$, while ``Average'' represents a uniformly sampled policy from the first $t$ iterates. ``V-Learning''  and ``PG'' denote the policies at the current iterate $t$ of Algorithms~\ref{alg:sbv} and~\ref{alg:storm}, respectively.  ``Centralized'' is an oracle that can control the actions of the agents in a centralized way. In ``Independent'', each agent runs a na\"{i}ve single-agent Q-learning algorithm independently, by taking greedy actions with respect to its local Q-function estimates. All results are averaged over $20$ runs. }
\label{fig:main}
\end{figure*}

The polynomial sample complexity dependence on $\amax$ is a natural benefit of decentralized learning, while centralized methods would typically have an exponential dependence $\prod_{i=1}^N A_i$. Such an improvement becomes more significant as the number of agents $N$ increases. Also note that our sample complexity bound in Theorem~\ref{thm:finite} holds in expectation. To obtain a standard high-probability result that holds with probability $1-p$, one could either apply Markov's inequality and tolerate an additional $O(1/p)$ factor of sample complexity, or replace our SGA method with one that has high-probability guarantees~\citep{li2020high}. For completeness, we present a sample complexity lower bound in the order of $\Omega(1/\epsilon^2)\cdot \text{poly}(H,S,\amax)$ in Appendix~\ref{app:lowerbound}, which is achieved by a reduction to  single-agent RL.

Finally, we show that our algorithm can nearly find the globally optimal NE (i.e., the NE that maximizes the potential function, which is guaranteed to exist~\citep{leonardos2021global}) in an important subclass of MPGs named \emph{smooth MPGs}. Our definition of a $(\lambda,\omega)$-smooth MPG, adapted from the definition of smooth games~\citep{roughgarden2009intrinsic,radanovic2019learning}, is formally introduced in Definition~\ref{def:smooth} of Appendix~\ref{app:smooth}. Let $\pi^\star$ be a policy that maximizes the potential function, i.e., $\Phi_{\rho}(\pi^\star) = \max_{\pi\in\Pi}\Phi_{\rho}(\pi)$, and let $V^\star_{1,i}$ denote the value function for agent $i$ under policy $\pi^\star$. The following theorem states that Algorithm~\ref{alg:storm} can nearly find the globally optimal NE in smooth MPGs. Its proof can be found in Appendix~\ref{app:smooth}.

\begin{theorem}\label{thm:smooth} 
	In a $(\lambda,\omega)$-smooth MPG, for any initial policies and any $\epsilon>0$, let the agents independently run SGA policy updates (Algorithm~\ref{alg:storm}) for $T$ iterations with $T = O(1/\epsilon^{4.5})\cdot \text{poly}(N, D, S, A_{\max}, H)$. Then, there exists $t\in[T]$, such that
	\[
	\begin{aligned}
	\ee \L V_{1,i}^{\pi^{(t)}}(\rho)\R \geq  \frac{\lambda}{1+\omega} V_{1,i}^{\star}(\rho) -\frac{\epsilon}{1+\omega},\forall i \in\mc{N}.
	\end{aligned}
	\] 
\end{theorem} 

We remark that our definition of smooth MPGs generalizes that of smooth teams in \citet{radanovic2019learning,mao2020near}, who assume an identical reward function of all the agents. Our approach also significantly improves the two works in that we design natural update rules for all the agents, who play symmetric roles in the self-play setting; the other two works only assign the algorithm to \emph{one} agent, and have to \emph{assume} that the policies of the other agent(s) change slowly. 

\section{Simulations}\label{sec:simulations}

We empirically evaluate Algorithm~\ref{alg:storm} (SGA) on a classic matrix team task~\citep{claus1998dynamics}, and both Algorithms~\ref{alg:sbv} and~\ref{alg:storm} on two Markov games, namely GoodState and BoxPushing~\citep{seuken2007improved}. Figure~\ref{fig:main} illustrates the performances of the algorithms in terms of the collected rewards. Detailed descriptions of the simulations are deferred to Appendix~\ref{app:simulations} due to space limitations. Overall, our simulations show more encouraging results than what our theory suggests: For Algorithm~\ref{alg:sbv}, the \emph{actual} policy trajectories converge and achieve high rewards, even though our theoretical guarantees only hold for a ``certified'' output policy. Further, Algorithm~\ref{alg:storm} achieves the globally-optimal NE frequently in our simulations, even though our theory does not guarantee so in general. Both algorithms outperform ``Independent'' learning and in certain cases approach the performance of a ``Centralized'' oracle. 

\section{Concluding Remarks}\label{sec:conclusions}

In this paper, we have studied sample-efficient MARL in decentralized scenarios. We have proposed stage-based V-learning algorithms that learn CCE and CE in general-sum Markov games, and policy gradient algorithms that learn NE in Markov potential games. Our algorithms have improved existing results either through a simplified algorithmic design or a sharper sample complexity bound. An interesting future direction would be to tighten the sample complexity upper and lower bounds established in this paper. The problem of efficiently finding the globally optimal NE in generic MPGs through decentralized learning is also left open.


\bibliography{ref}
\bibliographystyle{iclr2022_conference}

\appendix

\onecolumn

~\\
\centerline{{\fontsize{13.5}{13.5}\selectfont \textbf{Supplementary Materials for ``On Improving Model-Free Algorithms for }}}

\vspace{6pt}
\centerline{\fontsize{13.5}{13.5}\selectfont \textbf{
		Decentralized Multi-Agent Reinforcement Learning''}}
\vspace{10pt}

\section{Detailed Discussions on Related Work}\label{app:related}
A common mathematical framework of multi-agent RL is stochastic games \citep{shapley1953stochastic}, which are also referred to as Markov games. Early attempts to learn Nash equilibria in Markov games include \citet{littman1994markov,littman2001friend,hu2003nash,hansen2013strategy}, but they either assume the transition kernel and rewards are known, or only yield asymptotic guarantees. More recently, various sample efficient methods have been proposed \citep{wei2017online,bai2020provable,sidford2020solving,xie2020learning,bai2020near,liu2020sharp,zhao2021provably}, mostly for learning in two-player zero-sum Markov games. Most notably, several works have investigated two-player zero-sum games in a \emph{decentralized} environment: \citet{daskalakis2020independent} have shown non-asymptotic convergence guarantees for independent policy gradient methods when the learning rates of the two agents follow a two-timescale rule. \citet{tian2020provably} have studied online learning when the actions of the opponents are not observable, and have achieved the first sub-linear regret $\widetilde{O}(K^{\frac{3}{4}})$ in the decentralized setting for $K$ episodes. More recently, \citet{wei2021last} have proposed an Optimistic Gradient Descent Ascent algorithm with a slowly-learning critic, and have shown a strong finite-time last-iterate convergence result in the decentralized/agnostic environment. Overall, these works have mainly focused on two-player zero-sum games. These results do not carry over in any way to general-sum games or MPGs that we consider in this paper.  

In general-sum normal-form games, a folklore result is that when the agents independently run no-regret learning algorithms, their empirical frequency of plays converges to the set of coarse correlated equilibria (CCE) of the game \citep{hart2000simple}. However, a CCE may suggest that the agents play obviously non-rational strategies. For example, \citet{viossat2013no} have constructed an example where a CCE assigns positive probabilities only to strictly dominated strategies. On the other hand, given the PPAD completeness of finding a Nash equilibrium, convergence to NE seems hopeless in general. An impossibility result \citep{hart2003uncoupled} has shown that uncoupled no-regret learning does not converge to Nash equilibrium in general, due to the informational constraint that the adjustment in an agent's strategy does not depend on the reward functions of the others.  Hence, convergence to Nash equilibria is guaranteed mostly in games with special reward structures, such as two-player zero-sum games \citep{freund1999adaptive} and potential games \citep{kleinberg2009multiplicative,cohen2017learning}. 

For learning in general-sum Markov games, \citet{rubinstein2016settling} has shown a sample complexity lower bound for NE that is exponential in the number of agents. Recently, \citet{liu2020sharp} has presented a line of results on learning NE, CE, or CCE, but their algorithm is model-based, and suffers from such exponential dependence. \citet{song2021can,jin2021v,mao2022provably} have proposed V-learning based methods for learning CCE and/or CE, which are similar to the ones that we study here, and avoid the exponential dependence. Nevertheless, our methods significantly simplify their algorithmic design and analysis, by introducing a stage-based V-learning update rule that circumvents their rather complicated no-weighted-regret bandit subroutine.

Another line of research has considered RL in Markov potential games~\citep{macua2018learning,mguni2021learning}. \citet{arslan2016decentralized} has shown that decentralized Q-learning style algorithms can converge to NE in weakly acyclic games, which cover MPGs as an important special case. Their decentralized setting is similar to ours in that each agent is completely oblivious to the presence of the others. Later, such a method has been improved in \citet{yongacoglu2019learning} to achieve team-optimality. However, both of them require a coordinated exploration phase, and only yield asymptotic guarantees. Decentralized learning has also been studied in single-stage weakly acyclic games~\citep{marden2009payoff} or potential games \citep{marden2009cooperative,cohen2017learning}.  Two recent works \citep{zhang2021gradient,leonardos2021global} have proposed independent policy gradient methods in MPGs, which are most relevant to ours. We improve their sample complexity dependence on $\epsilon$ by utilizing a decentralized variance reduction technique, and do not require the two-timescale framework to coordinate policy evaluation as in \citet{zhang2021gradient}. \citet{fox2021independent} has shown that independent Natural Policy Gradient also converges to NE in MPGs, though only  asymptotic convergence has been established. Finally, MPGs have also been studied in \citet{song2021can}, but their model-based method is not decentralized, and requires the agents to take turns to learn the policies. 

MARL has also been studied in teams or cooperative games, which can be considered as a subclass of MPGs. Without enforcing a decentralized environment, \citet{boutilier1996planning} has proposed to coordinate the agents by letting them take actions in a lexicographic order. In a similar setting, \citet{wang2002reinforcement} have studied optimal adaptive learning that converges to the optimal NE in Markov teams. \citet{verbeeck2002learning} have presented an independent learning algorithm that achieves a Pareto optimal NE in common interest games with limited communication.  These methods critically relied on communications among the agents (beforehand) or observing the teammates' actions. In contrast, the distributed Q-learning algorithm in \cite{lauer2000algorithm} is decentralized and coordination-free, which, however, only works for deterministic tasks, and has no non-asymptotic guarantees.

Efficient exploration has also been widely studied in the literature of single-agent RL, see, e.g., \citet{brafman2002r,jaksch2010near,azar2017minimax,jin2018q}. For the tabular episodic setting, various methods \citep{azar2017minimax,zhang2020almost,menard2021ucb} have achieved the sample complexity of $\widetilde{O}(H^3 SA/\epsilon^2)$, which matches the information-theoretical lower bound. When reduced to the bandit case, decentralized MARL is also related to the cooperative multi-armed bandit (MAB) problem \citep{lai2008medium,avner2014concurrent}, originated from the literature of cognitive radio networks. The difference is that, in cooperative MAB, each agent is essentially interacting with an individual copy of the bandit, with an extra caution of action collisions; in the MARL formulation, the reward function is defined on the Cartesian product of the action spaces, which allows the agents to be coupled in more general forms. A concurrent work~\citep{chang2021online} has studied cooperative multi-player multi-armed bandits with information asymmetry. Nevertheless, \citep{chang2021online} requires stronger conditions than our decentralized setting as their algorithm relies on playing a predetermined sequence of actions.

\section{Technical Lemmas}\label{app:lemmas}
\begin{lemma}\label{lemma:36}
	\citep[Lemma 3.6]{bubeck2015convex}. Let $f$ be a $\beta$-smooth function with a convex domain $\mc{X}$. For any $x\in\mc{X}$, let $x^+ = \text{Proj}_{\mc{X}}\l x-\eta\grad f(x) \r$ be a projected gradient descent update with $\eta = \frac{1}{\beta}$, and let $G^\eta(x) = \frac{1}{\eta}(x-x^+)$. Then, the following holds true
	\[
	f(x^+) - f(x) \leq -\frac{1}{2\beta}\norm{G^{\eta}(x)}_2^2. 
	\]
\end{lemma}

\begin{lemma}\label{lemma:gradient_map}
	\citep[Proposition B.1]{agarwal2021theory}. Let $f:\mc{X}\ra \rr$ be a $\beta$-smooth function. Define the gradient mapping as 
	\[
	G^\eta(x) = \frac{1}{\eta} \l \text{Proj}_{\mc{X}}\l x + \eta \grad f(x) \r - x \r. 
	\]
	The update rule for projected gradient ascent is $x^+ = x + \eta G^\eta(x)$. If $\norm{G^\eta(x)}_2 \leq \epsilon$, then
	\[
	\max _{x+\delta \in \mc{X},\|\delta\|_{2}^2 \leq 1} \delta\T \grad f(x^{+}) \leq \epsilon(\eta \beta+1). 
	\]
\end{lemma}

\begin{lemma}\label{lemma:gradient_map_epsilon}
	\citep[Lemma D.3]{leonardos2021global}. Let $\Phi_\rho:\Pi\ra \rr$ be the potential function (which is $\beta$-smooth), and assume that $\pi\in\Pi$ uses $\epsilon_i$-greedy parameterization. Define the gradient mapping as 
	\[
	G^\eta(\pi) = \frac{1}{\eta} \l \text{Proj}_{\Pi}\l \pi + \eta \grad \Phi_\rho(\pi) \r - \pi \r. 
	\]
	The update rule for projected gradient ascent is $\pi^+ = \pi + \eta G^\eta(\pi)$. If $\eta\beta\leq 1$ and $\norm{G^\eta(\pi)}_2 \leq \epsilon$, then
	\[
	\max _{\pi+\delta \in \Pi,\|\delta\|_{2}^2 \leq 1} \delta\T \grad \Phi_\rho(\pi^{+}) \leq 2\epsilon + \sqrt{NSA^2_{\max}H^5\epsilon_i^2}. 
	\]
\end{lemma}

\begin{lemma}\label{lemma:block_matrix}
	\citep[Claim C.2]{leonardos2021global}. Consider a symmetric block matrix $C$ with $n\times n$ sub-matrices, and let $C_{ij}$ denote the sub-matrix at the $i$-th and $j$-th column. If $\norm{C_{ij}}_2\leq L$ for some $L > 0$, then it holds that $\norm{C}\leq nL$, i.e., if every sub-matrix of $C$ have a spectral norm of at most $L$, then $C$ has a spectral norm of at most $nL$. 
\end{lemma}

\section{Proofs for Section~\ref{subsec:cce}}\label{app:cce}

We first introduce a few notations to facilitate the analysis. For a step $h\in[H]$ of an episode $k\in[K]$, we denote by $s_h^k$ the state that the agents observe at this time step. For any state $s\in\mc{S}$, we let $\mu_{h,i}^k(\cdot \mid s)\in\Delta(\mc{A}_i)$ be the distribution prescribed by Algorithm~\ref{alg:sbv} to agent $i$ at this step. Notice that such notations are well-defined for every $s\in\mc{S}$ even if $s$ might not be the state $s_h^k$ that is actually visited at the given step. We further let $\mu_{h,i}^k = \{\mu_{h,i}^k(\cdot\mid s):s\in\mc{S}\}$, and let $a_{h,i}^k\in\mc{A}_i$ be the actual action taken by agent $i$. For any $s\in\mc{S}$, let $N_{h}^{k}(s)$ and  $\check{N}_{h}^{k}(s)$ denote, respectively, the values of $N_{h}(s)$ and $\check{N}_{h}(s)$ at the \emph{beginning} of the $k$-th episode. Note that it is proper to use the same notation to denote these values from all the agents' perspectives, because the agents maintain the same estimates of these terms as they can be calculated from the common observations (of the state-visitation). We also use $\overline{V}_{h,i}^{k}(s)$ and $\ot{V}_{h,i}^{k}(s)$ to denote the values of $\overline{V}_{h,i}(s)$ and $\ot{V}_{h,i}(s)$, respectively, at the beginning of the $k$-th episode from agent $i$'s perspective.

Further, for a state $s_h^k$, let $\check{n}_h^k$ denote the number of times that state $s_h^k$ has been visited (at the $h$-th step) in the stage right before the current stage, and let $\check{l}_{h,j}^k$ denote the index of the episode that this state was visited the $j$-th time among the $\check{n}_h^k$ times. For notational convenience, we use $\check{n}$ to denote $\check{n}_h^k$, and $\check{l}_j$ to denote $\check{l}_{h,j}^k$, whenever $h$ and $k$ are clear from the context. With the new notations, the update rule in Line~\ref{line:1} of Algorithm~\ref{alg:sbv} can be equivalently expressed as
\begin{equation}\label{eqn:update_rule_new}
\ot{V}_{h,i}(s_h) \gets \frac{1}{\check{n}}\sum_{j=1}^{\check{n}} \l r_{h,i}(s_h,\bm{a}_{h}^{\check{l}_j}) + \overline{V}_{h+1,i}^{\check{l}_j}(s_{h+1}^{\check{l}_j}) \r+ b_{\check{n}}. 
\end{equation}

For notational convenience, we introduce the operators $\mathbb{P}_{h} V(s, \bm{a})= \mathbb{E}_{s^{\prime} \sim P_{h}(\cdot \mid s, \bm{a})} V\left(s^{\prime}\right)$ for any value function $V$, and $\mathbb{D}_{\bm{\mu}_h} Q(s)= \mathbb{E}_{\bm{a} \sim \bm{\mu}_h} Q(s, \bm{a})$. With these notations, the Bellman equations can be rewritten more succinctly as
$
Q_{h}^{\pi}(s, \bm{a})=\left(r_{h}+\mathbb{P}_{h} V_{h+1}^{\pi}\right)(s, \bm{a}),$ and $V_{h}^{\pi}(s)=\left(\mathbb{D}_{\bm{\mu}_h} Q_{h}^{\pi}\right)(s)
$ for any $(s,\bm{a},h)\in \mc{S}\times \mc{A}\times [H]$, where $\bm{\mu}_h = \pi_h$. In the following proof, we assume without loss of generality that the initial state $s_1$ is fixed, i.e., $\rho$ is a point mass distribution at $s_1$. Our proof can be easily generalized to the case where the initial state is drawn from a fixed distribution $\rho\in\Delta(\mc{S})$. 

In the following, we start with an intermediate result, which justifies our choice of the bonus term. 

\begin{lemma}\label{lemma:bonus}
	With probability at least $1-\frac{p}{2}$, it holds for all $(i,s,h,k)\in\mc{N}\times \mc{S}\times [H]\times [K]$ that
	\[
	\max_{\mu_{h,i}} \frac{1}{\check{n}} \sum_{j=1}^{\check{n}} \mathbb{D}_{\mu_{h,i}\times \mu_{h,-i}^{\check{l}_j}}\l r_{h,i} + \pp_h\ol{V}_{h+1,i}^{\check{l}_j}\r(s) - \frac{1}{\check{n}} \sum_{j=1}^{\check{n}} \l r_{h,i}(s,\bm{a}_h^{\check{l}_j}) + \ol{V}_{h+1,i}^{\check{l}_j}(s_{h+1}^{\check{l}_j}) \r \leq 6\sqrt{H^2 A_i\iota/\check{n}}. 
	\]
\end{lemma}
\begin{proof}
	For a fixed $(s,h,k)\in\mc{S}\times [H]\times [K]$, let $\mc{F}_j$ be the $\sigma$-algebra generated by all the random variables up to episode $\check{l}_j$. Then, $\left\{ r_{h,i}(s,\bm{a}_{h,i}^{\check{l}_j}) + \ol{V}_{h+1,i}^{\check{l}_j} (s_{h+1}^{\check{l}_j}) - \mathbb{D}_{\bm{\mu}_{h}^{\check{l}_j}}\l r_{h,i} + \pp_h\ol{V}_{h+1,i}^{\check{l}_j}\r (s) \right\}_{j=1}^{\check{n}}$ is a martingale difference sequence with respect to $\{\mc{F}_j\}_{j=1}^{\check{n}}$. From the Azuma-Hoeffding inequality, it holds with probability at least $1-p/(4NSHK)$ that
	\[
	\frac{1}{\check{n}} \sum_{j=1}^{\check{n}} \mathbb{D}_{\bm{\mu}_h^{\check{l}_j}}\l r_{h,i} + \pp_h\ol{V}_{h+1,i}^{\check{l}_j}\r(s) - \frac{1}{\check{n}} \sum_{j=1}^{\check{n}} \l r_{h,i}(s,\bm{a}_h^{\check{l}_j}) + \ol{V}_{h+1,i}^{\check{l}_j}(s_{h+1}^{\check{l}_j}) \r \leq \sqrt{H^2\iota/\check{n}}. 
	\]
	Therefore, we only need to bound 
	\begin{equation}\label{eqn:bandit_regret}
	R_{\check{n}}^\star \defeq \max_{\mu_{h,i}} \frac{1}{\check{n}} \sum_{j=1}^{\check{n}} \mathbb{D}_{\mu_{h,i}\times \mu_{h,-i}^{\check{l}_j}}\l r_{h,i} + \pp_h\ol{V}_{h+1,i}^{\check{l}_j}\r(s) - \frac{1}{\check{n}} \sum_{j=1}^{\check{n}} \mathbb{D}_{\bm{\mu}_h^{\check{l}_j}}\l r_{h,i} + \pp_h\ol{V}_{h+1,i}^{\check{l}_j}\r(s).
	\end{equation}
	Notice that $R_{\check{n}}^\star$ can be considered as the averaged regret of visiting the state $s$ with respect to the optimal policy in hindsight. Such a regret minimization problem can be handled by an adversarial multi-armed bandit problem, where the loss function at step $j\in[\check{n}]$ is defined as
	\[
	\ell_j(a_i) = \ee_{{a}_{-i}\sim \mu_{h,-i}^{\check{l}_j}}(s) \L H-h+1 - r_{h,i}(s,\bm{a}) - \pp_h\ol{V}_{h+1,i}^{\check{l}_j}(s,\bm{a}) \R/H. 
	\]
	Algorithm~\ref{alg:sbv} applies the Exp3-IX algorithm~\citep{neu2015explore}, which ensures that with probability at least $1-\frac{p}{4NHS}$, it holds for all $k\in[K]$ that
	\[
	R_{\check{n}}^\star \leq \sqrt{\frac{8 H^2 A_i  \log A_i}{\check{n}}} + \l \sqrt{\frac{2A_i}{\check{n}\log A_i}} + \frac{1}{\check{n}} \r H\log(2/p).  
	\]
	A union bound over all $(i,s,h,k)\in\mc{N}\times\mc{S}\times [H]\times [K]$ completes the proof. 
\end{proof}
\begin{remark}\label{rmk:learning_rate}
	We would like to discuss the alternative of using V-learning with the celebrated learning rate $\alpha_t = \frac{H+1}{H+t}$ \citep{jin2018q} to update $\ol{V}_h$ instead of employing stage-based updates. This is the case for several recent works also under the V-learning formulation for MARL~\citep{bai2020near,jin2021v,song2021can,mao2022provably}.  Such a learning rate induces an update rule as follows:
	\begin{equation}\label{eqn:jin}
	\ol{V}_{h,i}\left(s_{h}\right) \leftarrow\left(1-\alpha_{t}\right) \ol{V}_{h,i}\left(s_{h}\right)+\alpha_{t}\left(r_{h,i}\left(s_{h}, \bm{a}_{h}\right)+\ol{V}_{h+1,i}\left(s_{h+1}\right)+\beta_{t}\right),
	\end{equation}
	where $t$ is the number of times that $s_h$ has been visited, and $\beta_t$ is some bonus term. In this way, $\ol{V}_{h,i}(s_h)$ is updated every time the state $s_h$ is visited. With such a learning rate, the update rule \eqref{eqn:jin} of $\ol{V}_{h,i}$ can be equivalently expressed as
	\[
	\ol{V}_{h,i}^{k}(s_h)=\alpha_{t}^{0} H+\sum_{j=1}^{t} \alpha_{t}^{j}\left[r_{h,i}\left(s, \bm{a}_{h}^{k^{j}}\right)+\ol{V}_{h+1,i}^{k^{j}}\left(s_{h+1}^{k^{j}}\right)+\beta_{j}\right],
	\]
	where $k^j$ is the index of the episode such that $s_h$ is visited the $j$-th time. The weights $\alpha_t^j$ are given by
	\[
	\alpha_{t}^{0}=\prod_{j=1}^{t}\left(1-\alpha_{j}\right), \quad \text{and}\quad  \alpha_{t}^{j}=\alpha_{j} \prod_{k=j+1}^{t}\left(1-\alpha_{k}\right),\forall 1\leq j \leq t.
	\]
	Compared with stage-based updates~\eqref{eqn:bandit_regret}, we now need to upper bound a regret term of the following form:
	\[
	R_{t}^\star(s)=\max _{\mu_{h,i}}\sum_{j=1}^{t} \alpha_{t}^{j}  \mathbb{D}_{\mu_{h,i} \times \mu_{h,-i}^{k^{j}}}\left(r_{h,i}+\mathbb{P}_{h} \ol{V}_{h+1,i}^{k^{j}}\right)(s)-\sum_{j=1}^{t} \alpha_{t}^{j} \mathbb{D}_{\mu_{h,i}^{k^{j}} \times \mu_{h,-i}^{k j}}\left(r_{h,i}+\mathbb{P}_{h} \ol{V}_{h+1,i}^{k^{j}}\right)(s).
	\]
	Notice that the above definition of regret induces a adversarial bandit problem with a time-varying weighted regret, where the loss at time $j$ is assigned a weight $\alpha_t^j$. As $t$ varies, the weight $\alpha_t^j$ assigned to the same step $j$ also changes over time. These weights also cannot be pre-computed, because it relies on knowing the total number of times that a certain state $s_h$ is visited during the entire horizon, which is impossible before seeing the output of the algorithm. To address such an additional challenge, \citet{bai2020near} proposed a Follow-the-Regularized-Leader (FTRL) algorithm that simultaneously achieves with a changing step size, a weighted regret, and a high-probability guarantee, which inevitably leads to a more delicate analysis. In contrast, we have shown in \eqref{eqn:bandit_regret} that our stage-based update rule leads to an adversarial bandit problem with a simple averaged regret. In our approach, it suffices to plug in any existing adversarial bandit solution with a high-probability regret bound, such as the Exp3-IX method that we used in Algorithm~\ref{alg:sbv}. Therefore, our stage-based update significantly simplifies both the algorithmic design and the analysis of V-learning in MARL. 
\end{remark}

\begin{algorithm*}[!tbp]
	\textbf{Input:} The distribution trajectory $\{\mu_{h,i}^k:i\in\mc{N},h\in[H],k\in[K]\}$ specified by Algorithm~\ref{alg:sbv}. 
	
	\textbf{Initialize: } $k'\gets k$.
	
	\For{step $h'\gets h$ to $H$}
	{
		Receive $s_{h'}$\;
		Take joint action $\bm{a}_{h'}\sim \times_{i=1}^N\mu_{h',i}^{k'}(\cdot \mid s_{h'})$\;
		Uniformly sample $j$ from $\{1,2,\dots, \check{N}_{h'}^{k'}(s_{h'})\}$\;
		Set $k'\gets \check{l}_{h',j}^{k'}$, where  $\check{l}^{k'}_{h',j}$ is the index of the episode such that state $s_{h'}$ was visited the $j$-th time (among the total $\check{N}_{h'}^{k'}(s_{h'})$ times) in the last stage\; 
	}
	\caption{Construction of the Correlated Policy $\bar{\pi}_h^k$}\label{alg:certifyhk}
\end{algorithm*}

Based on the trajectory of the distributions $\{\mu_{h,i}^k:i\in\mc{N},h\in[H],k\in[K]\}$ specified by Algorithm~\ref{alg:sbv}, we construct a correlated policy $\bar{\pi}_h^k$ for each $(h,k)\in[H]\times [K]$. Our construction of the correlated policies, largely inspired by the ``certified policies'' \citep{bai2020near} for learning in two-player zero-sum games, is formally presented  in Algorithm~\ref{alg:certifyhk}. We further define an output policy $\bar{\pi}$ that first uniformly samples an index $k$ from $[K]$, and then proceed with $\bar{\pi}_1^k$. A more formal description of $\bar{\pi}$ has been given in Algorithm~\ref{alg:certify}. By construction of the correlated policies $\bar{\pi}_h^k$, we know that for any $(i,s,h,k)\in\mc{N}\times\mc{S}\times [H+1]\times [K]$, the corresponding value function can be written recursively as follows:
\[
V_{h,i}^{\bar{\pi}_h^k}(s) = \frac{1}{\check{n}}\sum_{j=1}^{\check{n}} \mathbb{D}_{\mu_{h}^{\check{l}_j}}\bigg(r_{h,i}+\mathbb{P}_{h} V_{h+1,i}^{\bar{\pi}_{h+1}^{\check{l}_j}}\bigg)(s), 
\]
and $V_{h,i}^{\bar{\pi}_{h}^k}(s) = 0$ if $h=H+1$ or $k$ is in the first stage of the corresponding $(h,s)$ pair. We also immediately obtain that
\[
V_{1,i}^{\bar{\pi}}(s_1) = \frac{1}{K}\sum_{k=1}^K V_{1,i}^{\bar{\pi}_1^k}(s_1).
\]

Only for analytical purposes, we introduce two new notations $\underline{V}$ and $\undertilde{V}$ that serve as lower confidence bounds of the value estimates. Specifically, for any $(i,s,h,k)\in\mc{N}\times\mc{S}\times [H+1]\times [K]$, we define $\underline{V}_{h,i}^k(s) = \ut{V}_{h,i}^k(s) = 0$ if $h = H+1$ or $k$ is in the first stage of the $(h,s)$ pair, and
\[
\ut{V}_{h,i}^k(s) = \frac{1}{\check{n}}\sum_{j=1}^{\check{n}} \l r_{h,i}(s_h,\bm{a}_{h}^{\check{l}_j}) + \ul{V}_{h+1,i}^{\check{l}_j}(s_{h+1}^{\check{l}_j}) \r- b_{\check{n}},\text{ and } \ul{V}_{h,i}^k(s) = \max\left\{\ut{V}_{h,i}^k(s), 0\right\}. 
\]
Notice that these two notations are only introduced for ease of analysis, and the agents need not explicitly maintain such values during the learning process. Further, recall that $V_{h,i}^{\star,\bar{\pi}_{h,-i}^k}(s)$ is agent $i$'s best response value against its opponents' policy $\bar{\pi}_{h,-i}^k$. Our next lemma shows that $\ol{V}_{h,i}^k(s)$ and $\ul{V}_{h,i}^k(s)$ are indeed valid upper and lower bounds of $V_{h,i}^{\star,\bar{\pi}_{h,-i}^k}(s)$ and $V_{h,i}^{\bar{\pi}_h^k}(s)$, respectively. 

\begin{lemma}\label{lemma:upperbound}
	It holds with probability at least $1-p$ that for all $(i,s,h,k)\in\mc{N}\times\mc{S}\times[H]\times[K]$,
	\[
	\ol{V}_{h,i}^k(s)\geq V_{h,i}^{\star,\bar{\pi}_{h,-i}^k}(s), \text{ and } \ul{V}_{h,i}^k(s)\leq V_{h,i}^{\bar{\pi}_h^k}(s). 
	\]
\end{lemma}
\begin{proof}
	Consider a fixed $(i,s,h,k)\in\mc{N}\times\mc{S}\times[H]\times[K]$. The desired result clearly holds for any state $s$ that is in its first stage, due to our initialization of $\ol{V}_{h,i}^k(s)$ and  $\ul{V}_{h,i}^k(s)$ for this special case. In the following, we only need to focus on the case where $\overline{V}_{h,i}(s)$ and  $\ul{V}_{h,i}^k(s)$ have been updated at least once at the given state $s$ before the $k$-th episode. 
	
	We first prove the first inequality. It suffices to show that $\ot{V}_{h,i}^k(s)\geq V_{h,i}^{\star,\bar{\pi}_{h,-i}^k}(s)$ because $\ol{V}_{h,i}^k(s)=\min\{\ot{V}_{h,i}^k(s), H-h+1\}$, and $V_{h,i}^{\star,\bar{\pi}_{h,-i}^k}(s)$ is always less than or equal to $H-h+1$. Our proof relies on induction on $k\in [K]$. First, the claim holds for $k=1$ due to the aforementioned logic. For each step $h\in[H]$ and $s\in\mc{S}$, we consider the following two cases. 
	
	\textbf{Case 1:} $\ot{V}_{h,i}(s)$ has just been updated in (the end of) episode $k-1$. In this case, 
	\begin{equation}
	\ot{V}^k_{h,i}(s) = \frac{1}{\check{n}}\sum_{j=1}^{\check{n}} \l r_{h,i}(s,\bm{a}_{h}^{\check{l}_j}) + \overline{V}_{h+1,i}^{\check{l}_j}(s_{h+1}^{\check{l}_j}) \r+ b_{\check{n}} .
	\end{equation}
	By the definition of $V_h^{\star, \bar{\nu}_h^k }(s)$, it holds with probability at least $1-\frac{p}{2NSKH}$ that
	\begin{align}
	V_{h,i}^{\star,\bar{\pi}_{h,-i}^k}(s)\leq  &\max_{\mu_{h,i}} \frac{1}{\check{n}}\sum_{j=1}^{\check{n}} \mb{D}_{\mu_{h,i}\times \mu_{h,-i}^{\check{l}_{j}}} \l r_{h,i} + \mb{P}_hV_{h+1,i}^{\star, \bar{\pi}_{h+1,-i}^{\check{l}_{j}}}\r  (s)\nonumber\\
	\leq&\max_{\mu_{h,i}} \frac{1}{\check{n}}\sum_{j=1}^{\check{n}} \mb{D}_{\mu_{h,i}\times \mu_{h,-i}^{\check{l}_{j}}} \l r_{h,i} + \mb{P}_h\ol{V}_{h+1,i}^{\check{l}_j}\r  (s)\nonumber\\
	\leq & \frac{1}{\check{n}} \sum_{j=1}^{\check{n}} \l r_{h,i}(s,\bm{a}_h^{\check{l}_j}) + \ol{V}_{h+1,i}^{\check{l}_j}(s_{h+1}^{\check{l}_j}) \r + 6\sqrt{H^2 A_i\iota/\check{n}}\nonumber\\
	\leq & \ot{V}^k_{h,i}(s),
	\end{align}
	where the second step is by the induction hypothesis, the third step holds due to Lemma~\ref{lemma:bonus}, and the last step is by the definition of $b_{\check{n}}$. 
	
	\textbf{Case 2: } $\ot{V}_{h,i}(s)$ was not updated in (the end of) episode $k-1$.  Since we have excluded the case that $\ot{V}_{h,i}$ has never been updated, we are guaranteed that there exists an episode $j$ such that $\ot{V}_{h,i}(s)$ has been updated in the end of episode $j-1$ most recently. In this case, $\ot{V}_{h,i}^k(s) = \ot{V}_{h,i}^{k-1}(s) = \dots = \ot{V}_{h,i}^{j}(s) \geq V_{h,i}^{\star,\bar{\pi}_{h,-i}^j}(s)$, where the last step is by the induction hypothesis. Finally, observe that by our definition, the value of $V_{h,i}^{\star,\bar{\pi}_{h,-i}^j}(s)$ is a constant for all episode indices $j$ that belong to the same stage. Since we know that episode $j$ and episode $k$ lie in the same stage, we can conclude  that $V_{h,i}^{\star,\bar{\pi}_{h,-i}^k}(s) = V_{h,i}^{\star,\bar{\pi}_{h,-i}^j}(s) \leq \ot{V}_{h,i}^k(s)$. 
	
	Combining the two cases and applying a union bound over all $(i,s,h,k)\in\mc{N}\times\mc{S}\times[H]\times[K]$ complete the proof of the first inequality. 
	
	Next, we prove the second inequality in the statement of the lemma. Notice that it suffices to show $\ut{V}_{h,i}^k(s)\leq V_{h,i}^{\bar{\pi}_h^k}(s)$ because $\ul{V}_{h,i}^k(s) = \max\{\ut{V}_{h,i}^k(s),0\}$. Our proof again relies on induction on $k\in[K]$. Similar to the proof of the first inequality, the claim apparently holds for $k=1$, and we consider the following two cases for each step $h\in[H]$ and $s\in\mc{S}$.
	
	\textbf{Case 1:} The value of $\ut{V}_{h,i}(s)$ has just changed in (the end of) episode $k-1$. In this case, 
	\begin{equation}
	\ut{V}^k_{h,i}(s) = \frac{1}{\check{n}}\sum_{j=1}^{\check{n}} \l r_{h,i}(s,\bm{a}_{h}^{\check{l}_j}) + \ut{V}_{h+1,i}^{\check{l}_j}(s_{h+1}^{\check{l}_j}) \r- b_{\check{n}} .
	\end{equation}
	By the definition of $V_{h,i}^{\bar{\pi}_h^k}(s)$, it holds with probability at least $1-\frac{p}{2NSKH}$ that
	\begin{align}
	V_{h,i}^{\bar{\pi}_h^k}(s)=  &\frac{1}{\check{n}}\sum_{j=1}^{\check{n}} \mathbb{D}_{\bm{\mu}_{h}^{\check{l}_j}}\l r_{h,i}+\mathbb{P}_{h} V_{h+1,i}^{\bar{\pi}_{h+1}^{\check{l}_j}}\r(s)\nonumber\\
	\geq&\frac{1}{\check{n}}\sum_{j=1}^{\check{n}} \mathbb{D}_{\bm{\mu}_{h}^{\check{l}_j}}\l r_{h,i}+\mathbb{P}_{h} \ut{V}_{h+1,i}^{\check{l}_j}\r(s)\nonumber\\
	\geq & \frac{1}{\check{n}} \sum_{j=1}^{\check{n}} \l r_{h,i}(s,\bm{a}_h^{\check{l}_j}) +\ut{V}_{h+1,i}^{\check{l}_j}(s_{h+1}^{\check{l}_j}) \r - \sqrt{H^2 \iota/\check{n}}\nonumber\\
	\geq & \ut{V}^k_{h,i}(s),
	\end{align}
	where the second step is by the induction hypothesis, the third step holds due to the Azuma-Hoeffding inequality, and the last step is by the definition of $b_{\check{n}}$. 
	
	\textbf{Case 2: } The value of $\ut{V}_{h,i}(s)$ has not changed in (the end of) episode $k-1$.  Since we have excluded the case that $\ut{V}_{h,i}$ has never been updated, we are guaranteed that there exists an episode $j$ such that $\ut{V}_{h,i}(s)$ has changed in the end of episode $j-1$ most recently. In this case, we know that indices $j$ and $k$ belong to the same stage, and $\ut{V}_{h,i}^k(s) = \ut{V}_{h,i}^{k-1}(s) = \dots = \ut{V}_{h,i}^{j}(s) \leq V_{h,i}^{\bar{\pi}_h^j}(s)$, where the last step is by the induction hypothesis. Finally, observe that by our definition, the value of $V_{h,i}^{\bar{\pi}_h^j}(s)$ is a constant for all episode indices $j$ that belong to the same stage. Since we know that episode $j$ and episode $k$ lie in the same stage, we can conclude  that $V_{h,i}^{\bar{\pi}_h^k}(s)= V_{h,i}^{\bar{\pi}_h^j}(s) \geq \ut{V}_{h,i}^k(s)$. 
	
	Again, combining the two cases and applying a union bound over all $(i,s,h,k)\in\mc{N}\times\mc{S}\times[H]\times[K]$ complete the proof. 
\end{proof}

The following result shows that the agents have no incentive to deviate from the correlated policy $\bar{\pi}$, up to a regret term of the order $\widetilde{O}(\sqrt{H^5 S \amax /K})$. 

\begin{theorem}\label{thm:main}
	For any $p\in(0,1]$, let $\iota = \log(2NS\amax KH/p)$. Suppose $K \geq \frac{SH}{\amax\iota}$, with probability at least $1-p$, it holds that
	$$
	V_{1,i}^{\star,\bar{\pi}_{-i}}(s_1) - V_{1,i}^{\bar{\pi}}(s_1)\leq O\l\sqrt{H^5 SA_{\max} \iota/K}\r, 
	$$  
\end{theorem}
\begin{proof}
	We first recall the definitions of several notations and define a few new ones. For a state $s_h^k$, recall that $\check{n}_h^k$ denotes the number of visits to the state $s_h^k$ (at the $h$-th step)  in the stage right before the current stage, and $\check{l}_{h,j}^k$ denotes the $j$-th episode among the $\check{n}_h^k$ episodes. Similarly, let $n_h^k$ be the total number of episodes that this state has been visited prior to the current stage, and let $l_{h,j}^k$ denote the index of the episode that this state was visited the $j$-th time among the total $n_h^k$ times. For simplicity, we use $l_j$ and $\check{l}_j$ to denote $l_{h,j}^k$ and $\check{l}_{h,j}^k$, and $\check{n}$ to denote $\check{n}_h^k$, whenever $h$ and $k$ are clear from the context.
	
	From Lemma~\ref{lemma:upperbound}, we know that
	\[
	\begin{aligned}
	V_{1,i}^{\star,\bar{\pi}_{-i}}(s_1) - V_{1,i}^{\bar{\pi}}(s_1) \leq & \frac{1}{K}\sum_{k=1}^K \l V_{1,i}^{\star,\bar{\pi}_{1,-i}^k}(s_1) - V_{1,i}^{\bar{\pi}_1^k} (s_1) \r\\
	\leq & \frac{1}{K}\sum_{k=1}^K \l \ol{V}_{1,i}^{k}(s_1) - \ul{V}_{1,i}^{k}(s_1)\r. 
	\end{aligned}
	\]
	We hence only need to upper bound $\frac{1}{K}\sum_{k=1}^K ( \ol{V}_{1,i}^{k}(s_1) - \ul{V}_{1,i}^{k}(s_1))$. For a fixed agent $i\in\mc{N}$, we define the following notation:
	$$
	\delta_h^k \defeq \overline{V}_{h,i}^{k}(s_h^k)- \ul{V}_{h,i}^{k}(s_h^k).
	$$ 
	The main idea of the subsequent proof is to upper bound $\sum_{k=1}^K \delta_h^k$ by the next step $\sum_{k=1}^K \delta_{h+1}^k$, and then obtain a recursive formula. From the update rule of $\ol{V}_{h,i}^k(s_h^k)$ in \eqref{eqn:update_rule_new}, we know that
	\[
	\ol{V}_{h,i}^k(s_h^k) \leq \mb{I}[n_h^k = 0]H + \frac{1}{\check{n}}\sum_{j=1}^{\check{n}} \l r_{h,i}(s_h,\bm{a}_{h}^{\check{l}_j}) + \overline{V}_{h+1,i}^{\check{l}_j}(s_{h+1}^{\check{l}_j}) \r+ b_{\check{n}},
	\]
	where the $\mb{I}[n_h^k=0]$ term counts for the event that the optimistic value function has never been updated for the given state. 
	
	Further recalling the definition of $\ul{V}_{h,i}^{k}(s_h^k)$, we have
	\begin{align}
	\delta_h^k \leq &\mb{I}[n_h^k = 0]H+ \frac{1}{\check{n}}\sum_{j=1}^{\check{n}} \l \overline{V}_{h+1,i}^{\check{l}_j}(s_{h+1}^{\check{l}_j}) - \ul{V}_{h+1,i}^{\check{l}_j}(s_{h+1}^{\check{l}_j}) \r+2 b_{\check{n}}\nonumber\\
	\leq& \mb{I}[n_h^k = 0]H + \frac{1}{\check{n}}\sum_{j=1}^{\check{n}}\delta_{h+1}^{\check{l}_j} + 2b_{\check{n}},\label{eqn:delta}
	\end{align}
	To find an upper bound of $\sum_{k=1}^K \delta_h^k$, we proceed to upper bound each term on the RHS of \eqref{eqn:delta} separately. First, notice that $\sum_{k=1}^K \mb{I}\left[n_h^k=0\right]\leq SH$, because each fixed state-step pair $(s,h)$ contributes at most $1$ to $\sum_{k=1}^K \mb{I}\left[n_h^k=0\right]$. Next, we turn to analyze the second term on the RHS of \eqref{eqn:delta}. Observe that
	\begin{align}
	\sum_{k=1}^K \frac{1}{\check{n}_h^k}\sum_{j=1}^{\check{n}_h^k} \delta_{h+1}^{\check{l}_{h,j}^k} =& \sum_{k=1}^K \sum_{m=1}^K \frac{1}{\check{n}_h^k} \delta_{h+1}^{m} \sum_{j=1}^{\check{n}_h^k} \indicator\left[\check{l}_{h,j}^k = m\right]\nonumber\\
	=&\sum_{m=1}^K  \delta_{h+1}^{m} \sum_{k=1}^K  \frac{1}{\check{n}_h^k} \sum_{j=1}^{\check{n}_h^k} \indicator\left[\check{l}_{h,j}^k = m\right].\label{eqn:d1}
	\end{align}
	For a fixed episode $m$, notice that $\sum_{j=1}^{\check{n}_h^k} \indicator[\check{l}_{h,j}^k = m]\leq 1$, and that $\sum_{j=1}^{\check{n}_h^k} \indicator[\check{l}_{h,j}^k = m]= 1$ happens if and only if $s_h^k = s_h^m$ and $(m,h)$ lies in the previous stage of $(k,h)$ with respect to the state-step pair $(s_h^k, h)$.  Define $\mc{K}_m\defeq \{ k\in[K]: \sum_{j=1}^{\check{n}_h^k} \indicator[\check{l}_{h,j}^k = m]= 1 \}$. We then know that all episode indices $k\in \mc{K}_m$ belong to the same stage, and hence these episodes have the same value of $\check{n}_h^k$. That is, there  exists an integer $N_m>0$, such that $\check{n}_h^k = N_m,\forall k \in \mc{K}_m$. Further, since the stages are partitioned in a way such that each stage is at most $(1+\frac{1}{H})$ times longer than the previous stage, we know that $|\mc{K}_m|\leq (1+\frac{1}{H})N_m$. Therefore, for every $m$, it holds that
	\begin{equation}\label{eqn:d2}
	\sum_{k=1}^K  \frac{1}{\check{n}_h^k} \sum_{j=1}^{\check{n}_h^k} \indicator\left[\check{l}_{h,j}^k = m\right] \leq 1+ \frac{1}{H}.
	\end{equation} 
	Combining \eqref{eqn:d1} and \eqref{eqn:d2} leads to the following upper bound of the second term in \eqref{eqn:delta}:	
	\begin{equation}\label{eqn:d3}
	\sum_{k=1}^K \frac{1}{\check{n}_h^k}\sum_{j=1}^{\check{n}_h^k} \delta_{h+1}^{\check{l}_{h,j}^k}\leq (1+\frac{1}{H}) \sum_{k=1}^{K} \delta_{h+1}^{k}.
	\end{equation}
	So far, we have obtained the following upper bound:
	\[
	\sum_{k=1}^K\delta_h^k \leq SH^2 + (1+\frac{1}{H})  \sum_{k=1}^K \delta_{h+1}^k + 2\sum_{k=1}^K b_{\check{n}_h^k}. 
	\]
	
	Iterating the above inequality over $h = H, H-1, \dots, 1$ leads to
	\begin{equation}\label{eqn:tmp9}
	\sum_{k=1}^{K}\delta_1^k \leq O\l SH^3 + \sum_{h=1}^H \sum_{k=1}^K (1+\frac{1}{H})^{h-1}b_{\check{n}_h^k} \r,
	\end{equation}
	where we used the fact that $(1+\frac{1}{H})^H \leq e$. In the following, we analyze the bonus term $b_{\check{n}_h^k}$ more carefully. Recall our definitions that  $e_1 =H,\ e_{i+1} = \floor{(1+\frac{1}{H})e_i},i\geq 1$, and $b_{\check{n}}= 6\sqrt{H^2 A_i\iota/\check{n}}$. For any $h\in[H]$, 
	\begin{align}
	\sum_{k=1}^K (1+\frac{1}{H})^{h-1} b_{\check{n}_h^k} \leq& \sum_{k=1}^K (1+\frac{1}{H})^{h-1}6\sqrt{H^2 A_i\iota/\check{N}_h^k}\nonumber\\
	=&6\sqrt{H^2 A_i\iota}\sum_{s\in\mc{S}}\sum_{j\geq 1} (1+\frac{1}{H})^{h-1}e_j^{-\frac{1}{2}}\sum_{k=1}^K\mb{I}\left[s_h^k = s, \check{N}_h^k(s_h^k) = e_j\right]\nonumber\\
	=& 6\sqrt{H^2 A_i\iota}\sum_{s\in\mc{S}}\sum_{j\geq 1} (1+\frac{1}{H})^{h-1}w(s,j)e_j^{-\frac{1}{2}},\nonumber
	\end{align}
	where we define $w(s,j)\defeq \sum_{k=1}^K\mb{I}\left[s_h^k = s, \check{N}_h^k(s_h^k) = e_j\right]$ for any $s\in\mc{S}$. If we further let $w(s) \defeq \sum_{j\geq 1} w(s,j)$, we can see that $\sum_{s\in\mc{S}} w(s) = K$. For each fixed state $s$, we now seek an upper bound of its corresponding $j$ value, denoted as $J$ in what follows. Since each stage is $(1+\frac{1}{H})$ times longer than its previous stage, we know that $w(s,j) = \sum_{k=1}^K \mb{I}\left[s_h^k=s, \check{N}_h^k(s_h^k) = e_j\right] = \floor{(1+\frac{1}{H})e_j}$ for any $1\leq j\leq J$. Since $\sum_{j=1}^J w(s,j) = w(s)$, we obtain that $e_J \leq (1+\frac{1}{H})^{J-1} \leq \frac{10}{1+\frac{1}{H}}\frac{w(s)}{H}$ by taking the sum of a geometric sequence. Therefore, by plugging in $w(s,j) = \floor{(1+\frac{1}{H})e_j}$, 
	\[
	\begin{aligned} 
	\sum_{j\geq 1} (1+\frac{1}{H})^{h-1}w(s,j)e_j^{-\frac{1}{2}}\leq O\left(\sum_{j=1}^J e_j^{\frac{1}{2}}\right)\leq O\left(\sqrt{w(s)H}\right), 
	\end{aligned}
	\]
	where in the second step we again used the formula of the sum of a geometric sequence. Finally, using the fact that $\sum_{s\in\mc{S}}w(s) = K$ and applying the Cauchy-Schwartz inequality, we have
	\begin{align}
	\sum_{h=1}^H\sum_{k=1}^K (1+\frac{1}{H})^{h-1} b_{\check{n}_h^k} =&O\left(\sqrt{H^4 A_i\iota}\sum_{s\in\mc{S}}\sum_{j\geq 1} (1+\frac{1}{H})^{h-1}w(s,j)e_j^{-\frac{1}{2}}\right)\nonumber\\
	\leq&  O\left(\sqrt{SA_i KH^5\iota}\right).\label{eqn:tmp8}
	\end{align}
	Summarizing the results above leads to
	\[
	\sum_{k=1}^K \delta_1^k \leq O\l SH^3 + \sqrt{SA_i KH^5\iota} \r. 
	\]
	In the case when $K$ is large enough, such that $K \geq \frac{SH}{A_i\iota}$, the second term becomes dominant, and we obtain the desired result:	
	\[
	V_{1,i}^{\star,\bar{\pi}_{-i}}(s_1) - V_{1,i}^{\bar{\pi}}(s_1) \leq \frac{1}{K}\sum_{k=1}^{K}\delta_1^k \leq O\l\sqrt{SA_i H^5\iota/K} \r .
	\]
	This completes the proof of the theorem.
\end{proof}

An immediate corollary is that we obtain an $\epsilon$-approximate CCE when $\sqrt{S\amax H^5\iota/K}\leq \epsilon$, which is Theorem~\ref{thm:cce} in the main text.

\vspace{.8em}
\noindent\textbf{Theorem \ref{thm:cce}.} (Sample complexity of learning CCE). For any $p\in (0,1]$, set $\iota = \log(2NSA_{\max} KH/p)$, and let the agents run Algorithm~\ref{alg:sbv} for $K$ episodes with $K= O(S A_{\max}H^5 \iota/\epsilon^2)$. Then, with probability at least $1-p$, the output policy $\bar{\pi}$ constitutes an $\epsilon$-approximate coarse correlated equilibrium.

\section{Proofs for Section~\ref{subsec:ce}}\label{app:ce}
We first present a no-swap-regret learning algorithm for the adversarial bandit problem, which serves as an important subroutine to achieve correlated equilibria in Markov games. We consider a standard adversarial bandit problem that lasts for $T$ time steps. The agent has an action space of $\mc{A} = \{1,\dots,A\}$. At each time step $t\in[T]$, the agent specifies a distribution $p_t\in\Delta(\mc{A})$ over the action space, and takes an action $a_t$ according to $p_t$. The adversary then selects a loss vector $l_t\in[0,1]^A$, where $l_t(a)\in[0,1]$ denotes the loss of action $a$ at time $t$. We consider partial information (bandit) feedback, where the agent only receives the reward associated with the selected action $a_t$. The external regret measures the difference between the cumulative reward that an algorithm obtains and that of the best fixed action in hindsight. Specifically,
\[
R_{\text{external}}(T) = \max_{a^\star\in\mc{A}} \sum_{t=1}^T \l l_t(a_t) - l_t(a^\star) \r. 
\]
The swap regret, instead, measures the difference between the cumulative reward of an algorithm and the cumulative reward that could be achieved by swapping multiple pairs of actions of the algorithm. To be more specific, we define a strategy modification $F:\mc{A}\ra\mc{A}$ to be a mapping from the action space to itself. For any action selection distribution $p$, we let $F\diamond p$ be the swapped distribution that takes action $a\in\mc{A}$ with probability $\sum_{a'\in\mc{A},F(a')=a}p(a')$. The swap regret\footnote{This is a modified version of the swap regret used in~\citet{blum2007external}, which is defined as $R_{\text{swap}}(T) = \max_{F:\mc{A}\ra\mc{A}} \sum_{t=1}^T \l l_t(a_t) - l_t(F(a_t)) \r$.} is then defined as
\[
R_{\text{swap}}(T) = \max_{F:\mc{A}\ra\mc{A}} \sum_{t=1}^T \l \inner{p_t,l_t} - \inner{F\diamond p_t, l_t} \r, 
\]
where recall that $p_t$ is the distribution that the algorithm specifies at time $t$ for action selection.

\begin{algorithm*}[!t]
	\textbf{Initialize:} $p_1(a)\gets 1/A,\forall a\in\mc{A}$,  $\gamma \gets \sqrt{\log A/T}$, and $\eta\gets\sqrt{\log A/T}$. 
	
	\For{$t\gets 1$ to $T$}
	{
		Take action $a_t \sim p_t(\cdot)$, and observe loss $l_t(a_t)$\;
		\For{action $a\in\mc{A}$}
		{
			\For{action $a'\in\mc{A}$}
			{
				$\hat{l}_t(a'\mid a)\gets p_t(a)l_t(a_t)\mb{I}\{a_t=a'\}/(p_t(a') + \gamma)$\;
				$q_{t+1}(a' \mid a)\gets \frac{\exp(-\eta \sum_{i=1}^t \hat{l}_i(a'\mid a))}{\sum_{b\in\mc{A}}\exp(-\eta \sum_{i=1}^t \hat{l}_i(b\mid a))}$\;
			}
		}
		Set $p_{t+1}$ such that $p_{t+1}(\cdot )=\sum_{a\in\mc{A}}p_{t+1}(a)q_{t+1}(\cdot\mid a)$\;
	}
	\caption{No-swap-regret learning}\label{alg:swap}
\end{algorithm*}

We follow the generic reduction introduced in~\citet{blum2007external}, and convert a Follow-the-Regularized-Leader algorithm with sublinear external regret to a no-swap-regret algorithm \citep{jin2021v}. The resulting algorithm is presented as Algorithm~\ref{alg:swap}. The following lemma shows that Algorithm~\ref{alg:swap} is indeed a no-swap-regret learning algorithm.

\begin{lemma}\citep[Theorem 26]{jin2021v}.\label{lemma:swap_regret}
	For any $T\in\nn$ and $p\in(0,1)$, let $\iota = \log(A^2/p)$. With probability at least $1-3p$, it holds that
	\[
	R_{\text{swap}}(T)\leq 10\sqrt{A^2 T\iota}.  
	\]
\end{lemma}
It is worth noting that \citet{jin2021v} presented a more general analysis with an anytime weighted swap regret guarantee. Such complication can be avoided in our algorithm, as our stage-based learning approach only entails a simple averaged swap regret analysis.

\begin{algorithm*}[!t]\label{alg:ce}
	\textbf{Initialize:} $\overline{V}_{h,i}(s) \gets H-h+1, \ot{V}_{h,i}(s) \gets H-h+1, N_h(s)\gets 0, \check{N}_h(s)\gets 0, \check{r}_{h,i}(s)\gets 0,\allowbreak \check{v}_{h,i}(s)\gets 0,\allowbreak \check{T}_h(s)\gets H, p_{h,i}(a\mid s)\gets 1/A_i$, $L_{h,i}^s(a'\mid a)\gets 0$, $\forall h\in[H],s\in\mc{S},a,a'\in\mc{A}_i$. 
	
	\For{episode $k\gets 1$ to $K$}
	{
		Receive $s_1$\;
		\For{step $h\gets 1$ to $H$}
		{
			$ N_h(s_h) \gets N_h(s_h) + 1,  \check{n}\defeq \check{N}_h(s_h) \gets \check{N}_h(s_h) + 1$\;
			Take action $a_{h,i} \sim p_{h,i}(\cdot \mid s_h)$, and observe reward $r_{h,i}$ and next state $s_{h+1}$\;
			$\check{r}_{h,i}(s_h)\gets \check{r}_{h,i}(s_h) + r_{h,i}, \check{v}_{h,i}(s_h)\gets \check{v}_{h,i}(s_h) + \overline{V}_{h+1,i}(s_{h+1}) $\;
			$\eta_i \gets \sqrt{\iota/ \check{T}_h(s_h)}, \gamma_i \gets \eta_i$\;
			\For{action $a\in\mc{A}_i$}
			{
				\For{action $a'\in\mc{A}_i$}
				{
					$L_{h,i}^s(a'\mid a) \gets L_{h,i}^s(a'\mid a)+\frac{p_{h,i}(a\mid s_h) [H-h+1-(r_{h,i} + \overline{V}_{h+1,i}(s_{h+1}))]}{H(p_{h,i}(a_{h,i} \mid s_h)+\gamma_i)}\mb{I}\{a_{h,i} = a\}$\;
					$q_{h,i}^{s_h}(a'\mid a)\gets \frac{\exp(-\eta_i L_{h,i}^{s_h}(a'\mid a))}{\sum_{b\in\mc{A}_i} \exp(-\eta_i L_{h,i}^{s_h}(b\mid a)) }$\;
				}
			}
			Set $p_{h,i}(a\mid s_h)$ such that $p_{h,i}(\cdot \mid s_h) = \sum_{a\in\mc{A}}p_{h,i}(a\mid s_h)q_{h,i}^{s_h}(\cdot \mid a)$\;
			\If{$N_h(s_h)\in\mc{L}$}
			{
				\texttt{//Entering a new stage}
				
				$\ot{V}_{h,i}(s_h)\gets  \frac{\check{r}_{h,i}(s_h)}{\check{n}} + \frac{\check{v}_{h,i}(s_h)}{\check{n}}+b_{\check{n}}$, where 
				$b_{\check{n}}\gets 11\sqrt{H^2 A_i^2\iota / \check{n}}$\label{line:2}\;
				$\ol{V}_{h,i}(s_h)\gets \min\{\ot{V}_{h,i}(s_h),H-h+1\}$\;
				$\check{N}_h(s_h)\gets 0, \check{r}_{h,i}(s_h)\gets 0, \check{v}_{h,i}(s_h)\gets 0, \check{T}_h(s_h)\gets \floor{(1+\frac{1}{H})\check{T}_h(s_h)}$\;
				$p_{h,i}(a\mid s_h)\gets 1/A_i, L_{h,i}^{s_h}(a'\mid a)\gets 0, \forall a,a'\in\mc{A}_i$\;
			}
		}
	}
	\caption{Stage-Based V-Learning for CE (agent $i$)}
\end{algorithm*}

The complete Stage-Based V-Learning algorithm for CE is presented in Algorithm~\ref{alg:ce}. In the following analysis, we follow the same notations as have been used in the CCE analysis. We again start with the following lemma that justifies our choice of the bonus term. 

\begin{lemma}\label{lemma:bonus2}
	With probability at least $1-\frac{p}{2}$, it holds for all $(i,s,h,k)\in\mc{N}\times \mc{S}\times [H]\times [K]$ that
	\[
	\max_{\psi_{i}\in\Psi_i} \frac{1}{\check{n}} \sum_{j=1}^{\check{n}} \mathbb{D}_{\psi_{h,i}^s \diamond \bm{\mu}_{h}^{\check{l}_j}}\l r_{h,i} + \pp_h\ol{V}_{h+1,i}^{\check{l}_j}\r(s) - \frac{1}{\check{n}} \sum_{j=1}^{\check{n}} \l r_{h,i}(s,\bm{a}_h^{\check{l}_j}) + \ol{V}_{h+1,i}^{\check{l}_j}(s_{h+1}^{\check{l}_j}) \r \leq 11\sqrt{H^2 A_i^2\iota/\check{n}}. 
	\]
\end{lemma}
\begin{proof}
	For a fixed $(s,h,k)\in\mc{S}\times [H]\times [K]$, let $\mc{F}_j$ be the $\sigma$-algebra generated by all the random variables up to episode $\check{l}_j$. Then, $\left\{ r_{h,i}(s,\bm{a}_{h,i}^{\check{l}_j}) + \ol{V}_{h+1,i}^{\check{l}_j} (s_{h+1}^{\check{l}_j}) - \mathbb{D}_{\bm{\mu}_{h,i}^{\check{l}_j}}\l r_{h,i} + \pp_h\ol{V}_{h+1,i}^{\check{l}_j}\r (s) \right\}_{j=1}^{\check{n}}$ is a martingale difference sequence with respect to $\{\mc{F}_j\}_{j=1}^{\check{n}}$. From the Azuma-Hoeffding inequality, it holds with probability at least $1-p/(4NSHK)$ that
	\[
	\frac{1}{\check{n}} \sum_{j=1}^{\check{n}} \mathbb{D}_{\bm{\mu}_h^{\check{l}_j}}\l r_{h,i} + \pp_h\ol{V}_{h+1,i}^{\check{l}_j}\r(s) - \frac{1}{\check{n}} \sum_{j=1}^{\check{n}} \l r_{h,i}(s,\bm{a}_h^{\check{l}_j}) + \ol{V}_{h+1,i}^{\check{l}_j}(s_{h+1}^{\check{l}_j}) \r \leq \sqrt{H^2\iota/\check{n}}. 
	\]
	Therefore, we only need to bound 
	\[
	R_{\text{swap}}({\check{n}}) \defeq \max_{\psi_i\in\Psi_i} \frac{1}{\check{n}} \sum_{j=1}^{\check{n}} \mathbb{D}_{\psi_{h,i}^s \diamond \bm{\mu}_{h}^{\check{l}_j}}\l r_{h,i} + \pp_h\ol{V}_{h+1,i}^{\check{l}_j}\r(s) - \frac{1}{\check{n}} \sum_{j=1}^{\check{n}} \mathbb{D}_{\bm{\mu}_h^{\check{l}_j}}\l r_{h,i} + \pp_h\ol{V}_{h+1,i}^{\check{l}_j}\r(s).
	\]
	Notice that $R_{\text{swap}}(\check{n})$ can be considered as the swap regret of an adversarial bandit problem at state $s$, where the loss function at step $j\in[\check{n}]$ is defined as
	\[
	\ell_j(a_i) = \ee_{{a}_{-i}\sim \mu_{h,-i}^{\check{l}_j}}(s) \L H-h+1 - r_{h,i}(s,\bm{a}) - \pp_h\ol{V}_{h+1,i}^{\check{l}_j}(s,\bm{a}) \R/H. 
	\]
	Such a problem can be addressed by a no-swap-regret learning algorithm as presented in Algorithm~\ref{alg:swap}. Applying Lemma~\ref{lemma:swap_regret}, we obtain that with probability at least $1-\frac{p}{4NHS}$, it holds for all $k\in[K]$ that
	\[
	R_{\text{swap}}(\check{n})\leq 10\sqrt{\frac{H^2 A_i^2  \iota}{\check{n}}}.  
	\]
	A union bound over all $(i,s,h,k)\in\mc{N}\times\mc{S}\times [H]\times [K]$ completes the proof. 
\end{proof}

We again define the notations $\bar{\pi}_h^k, \bar{\pi}, V_{h,i}^{\bar{\pi}_h^k}, \underline{V}_{h,i}^k$, and $\ut{V}_{h,i}^k(s)$ in the same sense as in Appendix~\ref{app:cce}. The next lemma shows that $\ol{V}_{h,i}^k(s)$ and $\ul{V}_{h,i}^k(s)$ are valid upper and lower bounds. 

\begin{lemma}\label{lemma:upperbound2}
	It holds with probability at least $1-p$ that for all $(i,s,h,k)\in\mc{N}\times\mc{S}\times[H]\times[K]$,
	\[
	\ol{V}_{h,i}^k(s)\geq \max_{\psi_{i}\in\Psi_i} V_{h,i}^{\psi_i\diamond \bar{\pi}_h^k}(s), \text{ and } \ul{V}_{h,i}^k(s)\leq V_{h,i}^{\bar{\pi}_h^k}(s). 
	\]
\end{lemma}
\begin{proof}
	Consider a fixed $(i,s,h,k)\in\mc{N}\times\mc{S}\times[H]\times[K]$. The desired result clearly holds for any state $s$ that is in its first stage, due to our initialization of $\ol{V}_{h,i}^k(s)$ and  $\ul{V}_{h,i}^k(s)$ for this special case. In the following, we only need to focus on the case where $\overline{V}_{h,i}(s)$ and  $\ul{V}_{h,i}^k(s)$ have been updated at least once at the given state $s$ before the $k$-th episode. 
	
	We start with the first inequality. It suffices to show that $\ot{V}_{h,i}^k(s)\geq \max_{\psi_{i}\in\Psi_i} V_{h,i}^{\psi_i\diamond \bar{\pi}_h^k}(s)$ because $\ol{V}_{h,i}^k(s)=\min\{\ot{V}_{h,i}^k(s), H-h+1\}$, and $\max_{\psi_{i}\in\Psi_i} V_{h,i}^{\psi_i\diamond \bar{\pi}_h^k}(s)$ is always less than or equal to $H-h+1$. Our proof relies on induction on $k\in [K]$. First, the claim holds for $k=1$ due to the aforementioned logic. For each step $h\in[H]$ and $s\in\mc{S}$, we consider the following two cases. 
	
	\textbf{Case 1:} $\ot{V}_{h,i}(s)$ has just been updated in (the end of) episode $k-1$. In this case, 
	\begin{equation}
	\ot{V}^k_{h,i}(s) = \frac{1}{\check{n}}\sum_{j=1}^{\check{n}} \l r_{h,i}(s,\bm{a}_{h}^{\check{l}_j}) + \overline{V}_{h+1,i}^{\check{l}_j}(s_{h+1}^{\check{l}_j}) \r+ b_{\check{n}} .
	\end{equation}
	By the definition of $\max_{\psi_{i}\in\Psi_i} V_{h,i}^{\psi_i\diamond \bar{\pi}_h^k}(s)$, it holds with probability at least $1-\frac{p}{2NSKH}$ that
	\begin{align}
	\max_{\psi_{i}\in\Psi_i} V_{h,i}^{\psi_i\diamond \bar{\pi}_h^k}(s)\leq  &\max_{\psi_i\in\Psi_i} \frac{1}{\check{n}}\sum_{j=1}^{\check{n}} \mb{D}_{\psi_i\diamond \bm{\mu}_h^{\check{l}_j}} \l r_{h,i} + \mb{P}_h\max_{\psi_i'\in\Psi_i}V_{h+1,i}^{ \psi_i'\diamond \bar{\pi}_{h+1}^{\check{l}_{j}}}\r  (s)\nonumber\\
	\leq&\max_{\psi_i\in\Psi_i} \frac{1}{\check{n}}\sum_{j=1}^{\check{n}} \mb{D}_{\psi_i\diamond \bm{\mu}_h^{\check{l}_j}} \l r_{h,i} + \mb{P}_h\ol{V}_{h+1,i}^{\check{l}_j}\r  (s)\nonumber\\
	\leq & \frac{1}{\check{n}} \sum_{j=1}^{\check{n}} \l r_{h,i}(s,\bm{a}_h^{\check{l}_j}) + \ol{V}_{h+1,i}^{\check{l}_j}(s_{h+1}^{\check{l}_j}) \r + 11\sqrt{H^2 A_i^2\iota/\check{n}}\nonumber\\
	\leq & \ot{V}^k_{h,i}(s),
	\end{align}
	where the second step is by the induction hypothesis, the third step holds due to Lemma~\ref{lemma:bonus2}, and the last step is by the definition of $b_{\check{n}}$. 
	
	\textbf{Case 2: } $\ot{V}_{h,i}(s)$ was not updated in (the end of) episode $k-1$.  Since we have excluded the case that $\ot{V}_{h,i}$ has never been updated, we are guaranteed that there exists an episode $j$ such that $\ot{V}_{h,i}(s)$ has been updated in the end of episode $j-1$ most recently. In this case, $\ot{V}_{h,i}^k(s) = \ot{V}_{h,i}^{k-1}(s) = \dots = \ot{V}_{h,i}^{j}(s) \geq \max_{\psi_{i}\in\Psi_i} V_{h,i}^{\psi_i\diamond \bar{\pi}_h^j}(s)$, where the last step is by the induction hypothesis. Finally, observe that by our definition, the value of $\max_{\psi_{i}\in\Psi_i} V_{h,i}^{\psi_i\diamond \bar{\pi}_h^j}(s)$ is a constant for all episode indices $j$ that belong to the same stage. Since we know that episode $j$ and episode $k$ lie in the same stage, we can conclude  that $\max_{\psi_{i}\in\Psi_i} V_{h,i}^{\psi_i\diamond \bar{\pi}_h^k}(s) = \max_{\psi_{i}\in\Psi_i} V_{h,i}^{\psi_i\diamond \bar{\pi}_h^j}(s) \leq \ot{V}_{h,i}^k(s)$. 
	
	Combining the two cases and applying a union bound over all $(i,s,h,k)\in\mc{N}\times\mc{S}\times[H]\times[K]$ complete the proof of the first inequality. 
	
	Next, we prove the second inequality in the statement of the lemma. Notice that it suffices to show $\ut{V}_{h,i}^k(s)\leq V_{h,i}^{\bar{\pi}_h^k}(s)$ because $\ul{V}_{h,i}^k(s) = \max\{\ut{V}_{h,i}^k(s),0\}$. Our proof again relies on induction on $k\in[K]$. Similar to the proof of the first inequality, the claim apparently holds for $k=1$, and we consider the following two cases for each step $h\in[H]$ and $s\in\mc{S}$.
	
	\textbf{Case 1:} The value of $\ut{V}_{h,i}(s)$ has just changed in (the end of) episode $k-1$. In this case, 
	\begin{equation}
	\ut{V}^k_{h,i}(s) = \frac{1}{\check{n}}\sum_{j=1}^{\check{n}} \l r_{h,i}(s,\bm{a}_{h}^{\check{l}_j}) + \ut{V}_{h+1,i}^{\check{l}_j}(s_{h+1}^{\check{l}_j}) \r- b_{\check{n}} .
	\end{equation}
	By the definition of $V_{h,i}^{\bar{\pi}_h^k}(s)$, it holds with probability at least $1-\frac{p}{2NSKH}$ that
	\begin{align}
	V_{h,i}^{\bar{\pi}_h^k}(s)=  &\frac{1}{\check{n}}\sum_{j=1}^{\check{n}} \mathbb{D}_{\bm{\mu}_{h}^{\check{l}_j}}\l r_{h,i}+\mathbb{P}_{h} V_{h+1,i}^{\bar{\pi}_{h+1}^{\check{l}_j}}\r(s)\nonumber\\
	\geq&\frac{1}{\check{n}}\sum_{j=1}^{\check{n}} \mathbb{D}_{\bm{\mu}_{h}^{\check{l}_j}}\l r_{h,i}+\mathbb{P}_{h} \ut{V}_{h+1,i}^{\check{l}_j}\r(s)\nonumber\\
	\geq & \frac{1}{\check{n}} \sum_{j=1}^{\check{n}} \l r_{h,i}(s,\bm{a}_h^{\check{l}_j}) +\ut{V}_{h+1,i}^{\check{l}_j}(s_{h+1}^{\check{l}_j}) \r - \sqrt{H^2 \iota/\check{n}}\nonumber\\
	\geq & \ut{V}^k_{h,i}(s),
	\end{align}
	where the second step is by the induction hypothesis, the third step holds due to the Azuma-Hoeffding inequality, and the last step is by the definition of $b_{\check{n}}$. 
	
	\textbf{Case 2: } The value of $\ut{V}_{h,i}(s)$ has not changed in (the end of) episode $k-1$.  Since we have excluded the case that $\ut{V}_{h,i}$ has never been updated, we are guaranteed that there exists an episode $j$ such that $\ut{V}_{h,i}(s)$ has changed in the end of episode $j-1$ most recently. In this case, we know that indices $j$ and $k$ belong to the same stage, and $\ut{V}_{h,i}^k(s) = \ut{V}_{h,i}^{k-1}(s) = \dots = \ut{V}_{h,i}^{j}(s) \leq V_{h,i}^{\bar{\pi}_h^j}(s)$, where the last step is by the induction hypothesis. Finally, observe that by our definition, the value of $V_{h,i}^{\bar{\pi}_h^j}(s)$ is a constant for all episode indices $j$ that belong to the same stage. Since we know that episode $j$ and episode $k$ lie in the same stage, we can conclude  that $V_{h,i}^{\bar{\pi}_h^k}(s)= V_{h,i}^{\bar{\pi}_h^j}(s) \geq \ut{V}_{h,i}^k(s)$. 
	
	Again, combining the two cases and applying a union bound over all $(i,s,h,k)\in\mc{N}\times\mc{S}\times[H]\times[K]$ complete the proof. 
\end{proof}

\begin{theorem}\label{thm:ce_regret}
	For any $p\in(0,1]$, let $\iota = \log(2NS\amax KH/p)$. Suppose $K \geq \frac{SH}{\amax^2\iota}$. With probability at least $1-p$, 
	$$
	\max_{\psi_{i}\in\Psi_i} V_{1,i}^{\psi_i\diamond \bar{\pi}}(s_1)- V_{1,i}^{\bar{\pi}}(s_1)\leq O\l\sqrt{H^5 S \amax^2 \iota/K}\r, 
	$$  
\end{theorem}
\begin{proof}
	The proof follows a similar procedure as the proof of Theorem~\ref{thm:main}. From Lemma~\ref{lemma:upperbound2}, we know that
	\[
	\begin{aligned}
	\max_{\psi_{i}\in\Psi_i} V_{1,i}^{\psi_i\diamond \bar{\pi}}(s_1) - V_{1,i}^{\bar{\pi}}(s_1) = & \max_{\psi_{i}\in\Psi_i}\frac{1}{K}\sum_{k=1}^K \l  V_{1,i}^{\psi_i\diamond \bar{\pi}_1^k}(s_1) - V_{1,i}^{\bar{\pi}_1^k} (s_1) \r\\
	\leq & \frac{1}{K}\sum_{k=1}^K \l \max_{\psi_{i}\in\Psi_i} V_{1,i}^{\psi_i\diamond \bar{\pi}_1^k}(s_1) - V_{1,i}^{\bar{\pi}_1^k} (s_1) \r\\
	\leq & \frac{1}{K}\sum_{k=1}^K \l \ol{V}_{1,i}^{k}(s_1) - \ul{V}_{1,i}^{k}(s_1)\r. 
	\end{aligned}
	\]
	We hence only need to upper bound $\frac{1}{K}\sum_{k=1}^K ( \ol{V}_{1,i}^{k}(s_1) - \ul{V}_{1,i}^{k}(s_1))$. For a fixed agent $i\in\mc{N}$, we define the following notation:
	$$
	\delta_h^k \defeq \overline{V}_{h,i}^{k}(s_h^k)- \ul{V}_{h,i}^{k}(s_h^k).
	$$ 
	The main idea of the subsequent proof is to upper bound $\sum_{k=1}^K \delta_h^k$ by the next step $\sum_{k=1}^K \delta_{h+1}^k$, and then obtain a recursive formula. From the update rule of $\ol{V}_{h,i}^k(s_h^k)$ in \eqref{eqn:update_rule_new}, we know that
	\[
	\ol{V}_{h,i}^k(s_h^k) \leq \mb{I}[n_h^k = 0]H + \frac{1}{\check{n}}\sum_{j=1}^{\check{n}} \l r_{h,i}(s_h,\bm{a}_{h}^{\check{l}_j}) + \overline{V}_{h+1,i}^{\check{l}_j}(s_{h+1}^{\check{l}_j}) \r+ b_{\check{n}},
	\]
	where the $\mb{I}[n_h^k=0]$ term counts for the event that the optimistic value function has never been updated for the given state. 
	
	Further recalling the definition of $\ul{V}_{h,i}^{k}(s_h^k)$, we have
	\begin{align}
	\delta_h^k \leq &\mb{I}[n_h^k = 0]H+ \frac{1}{\check{n}}\sum_{j=1}^{\check{n}} \l \overline{V}_{h+1,i}^{\check{l}_j}(s_{h+1}^{\check{l}_j}) - \ul{V}_{h+1,i}^{\check{l}_j}(s_{h+1}^{\check{l}_j}) \r+2 b_{\check{n}}\nonumber\\
	\leq& \mb{I}[n_h^k = 0]H + \frac{1}{\check{n}}\sum_{j=1}^{\check{n}}\delta_{h+1}^{\check{l}_j} + 2b_{\check{n}},\label{eqn:delta2}
	\end{align}
	To find an upper bound of $\sum_{k=1}^K \delta_h^k$, we proceed to upper bound each term on the RHS of \eqref{eqn:delta2} separately. First, notice that $\sum_{k=1}^K \mb{I}\left[n_h^k=0\right]\leq SH$, because each fixed state-step pair $(s,h)$ contributes at most $1$ to $\sum_{k=1}^K \mb{I}\left[n_h^k=0\right]$. Next, we turn to analyze the second term on the RHS of \eqref{eqn:delta2}. Observe that
	\begin{align}
	\sum_{k=1}^K \frac{1}{\check{n}_h^k}\sum_{j=1}^{\check{n}_h^k} \delta_{h+1}^{\check{l}_{h,j}^k} =& \sum_{k=1}^K \sum_{m=1}^K \frac{1}{\check{n}_h^k} \delta_{h+1}^{m} \sum_{j=1}^{\check{n}_h^k} \indicator\left[\check{l}_{h,j}^k = m\right]\nonumber\\
	=&\sum_{m=1}^K  \delta_{h+1}^{m} \sum_{k=1}^K  \frac{1}{\check{n}_h^k} \sum_{j=1}^{\check{n}_h^k} \indicator\left[\check{l}_{h,j}^k = m\right].\label{eqn:d1x}
	\end{align}
	For a fixed episode $m$, notice that $\sum_{j=1}^{\check{n}_h^k} \indicator[\check{l}_{h,j}^k = m]\leq 1$, and that $\sum_{j=1}^{\check{n}_h^k} \indicator[\check{l}_{h,j}^k = m]= 1$ happens if and only if $s_h^k = s_h^m$ and $(m,h)$ lies in the previous stage of $(k,h)$ with respect to the state-step pair $(s_h^k, h)$.  Define $\mc{K}_m\defeq \{ k\in[K]: \sum_{j=1}^{\check{n}_h^k} \indicator[\check{l}_{h,j}^k = m]= 1 \}$. We then know that all episode indices $k\in \mc{K}_m$ belong to the same stage, and hence these episodes have the same value of $\check{n}_h^k$. That is, there  exists an integer $N_m>0$, such that $\check{n}_h^k = N_m,\forall k \in \mc{K}_m$. Further, since the stages are partitioned in a way such that each stage is at most $(1+\frac{1}{H})$ times longer than the previous stage, we know that $|\mc{K}_m|\leq (1+\frac{1}{H})N_m$. Therefore, for every $m$, it holds that
	\begin{equation}\label{eqn:d2x}
	\sum_{k=1}^K  \frac{1}{\check{n}_h^k} \sum_{j=1}^{\check{n}_h^k} \indicator\left[\check{l}_{h,j}^k = m\right] \leq 1+ \frac{1}{H}.
	\end{equation} 
	Combining \eqref{eqn:d1x} and \eqref{eqn:d2x} leads to the following upper bound of the second term in \eqref{eqn:delta2}:	
	\begin{equation}
	\sum_{k=1}^K \frac{1}{\check{n}_h^k}\sum_{j=1}^{\check{n}_h^k} \delta_{h+1}^{\check{l}_{h,j}^k}\leq (1+\frac{1}{H}) \sum_{k=1}^{K} \delta_{h+1}^{k}.
	\end{equation}
	So far, we have obtained the following upper bound:
	\[
	\sum_{k=1}^K\delta_h^k \leq SH^2 + (1+\frac{1}{H})  \sum_{k=1}^K \delta_{h+1}^k + 2\sum_{k=1}^K b_{\check{n}_h^k}. 
	\]
	
	Iterating the above inequality over $h = H, H-1, \dots, 1$ leads to
	\begin{equation}
	\sum_{k=1}^{K}\delta_1^k \leq O\l SH^3 + \sum_{h=1}^H \sum_{k=1}^K (1+\frac{1}{H})^{h-1}b_{\check{n}_h^k} \r,
	\end{equation}
	where we used the fact that $(1+\frac{1}{H})^H \leq e$. In the following, we analyze the bonus term $b_{\check{n}_h^k}$ more carefully. Recall our definitions that  $e_1 =H,\ e_{i+1} = \floor{(1+\frac{1}{H})e_i},i\geq 1$, and $b_{\check{n}}= 11\sqrt{H^2 A_i^2\iota/\check{n}}$. For any $h\in[H]$, 
	\begin{align}
	\sum_{k=1}^K (1+\frac{1}{H})^{h-1} b_{\check{n}_h^k} \leq& \sum_{k=1}^K (1+\frac{1}{H})^{h-1}11\sqrt{H^2 A_i^2\iota/\check{N}_h^k}\nonumber\\
	=&11\sqrt{H^2 A_i^2\iota}\sum_{s\in\mc{S}}\sum_{j\geq 1} (1+\frac{1}{H})^{h-1}e_j^{-\frac{1}{2}}\sum_{k=1}^K\mb{I}\left[s_h^k = s, \check{N}_h^k(s_h^k) = e_j\right]\nonumber\\
	=& 11\sqrt{H^2 A_i^2\iota}\sum_{s\in\mc{S}}\sum_{j\geq 1} (1+\frac{1}{H})^{h-1}w(s,j)e_j^{-\frac{1}{2}},\nonumber
	\end{align}
	where we define $w(s,j)\defeq \sum_{k=1}^K\mb{I}\left[s_h^k = s, \check{N}_h^k(s_h^k) = e_j\right]$ for any $s\in\mc{S}$. If we further let $w(s) \defeq \sum_{j\geq 1} w(s,j)$, we can see that $\sum_{s\in\mc{S}} w(s) = K$. For each fixed state $s$, we now seek an upper bound of its corresponding $j$ value, denoted as $J$ in what follows. Since each stage is $(1+\frac{1}{H})$ times longer than its previous stage, we know that $w(s,j) = \sum_{k=1}^K \mb{I}\left[s_h^k=s, \check{N}_h^k(s_h^k) = e_j\right] = \floor{(1+\frac{1}{H})e_j}$ for any $1\leq j\leq J$. Since $\sum_{j=1}^J w(s,j) = w(s)$, we obtain that $e_J \leq (1+\frac{1}{H})^{J-1} \leq \frac{10}{1+\frac{1}{H}}\frac{w(s)}{H}$ by taking the sum of a geometric sequence. Therefore, by plugging in $w(s,j) = \floor{(1+\frac{1}{H})e_j}$, 
	\[
	\begin{aligned} 
	\sum_{j\geq 1} (1+\frac{1}{H})^{h-1}w(s,j)e_j^{-\frac{1}{2}}\leq O\left(\sum_{j=1}^J e_j^{\frac{1}{2}}\right)\leq O\left(\sqrt{w(s)H}\right), 
	\end{aligned}
	\]
	where in the second step we again used the formula of the sum of a geometric sequence. Finally, using the fact that $\sum_{s\in\mc{S}}w(s) = K$ and applying the Cauchy-Schwartz inequality, we have
	\begin{align}
	\sum_{h=1}^H\sum_{k=1}^K (1+\frac{1}{H})^{h-1} b_{\check{n}_h^k} =&O\left(\sqrt{H^4 A_i^2\iota}\sum_{s\in\mc{S}}\sum_{j\geq 1} (1+\frac{1}{H})^{h-1}w(s,j)e_j^{-\frac{1}{2}}\right)\nonumber\\
	\leq&  O\left(\sqrt{SA_i^2 KH^5\iota}\right).
	\end{align}
	Summarizing the results above leads to
	\[
	\sum_{k=1}^K \delta_1^k \leq O\l SH^3 + \sqrt{SA_i^2 KH^5\iota} \r. 
	\]
	In the case when $K$ is large enough, such that $K \geq \frac{SH}{A_i^2\iota}$, the second term becomes dominant, and we obtain the desired result:	
	\[
	\max_{\psi_{i}\in\Psi_i} V_{1,i}^{\psi_i\diamond \bar{\pi}}(s_1) - V_{1,i}^{\bar{\pi}}(s_1) \leq \frac{1}{K}\sum_{k=1}^{K}\delta_1^k \leq O\l\sqrt{SA_i^2 H^5\iota/K} \r .
	\]
	This completes the proof of the theorem.
\end{proof}

An immediate corollary is that we obtain an $\epsilon$-approximate CE when $\sqrt{S\amax^2 H^5\iota/K}\leq \epsilon$, which is Theorem~\ref{thm:ce} in the main text.

\vspace{.8em}
\noindent\textbf{Theorem \ref{thm:ce}.} (Sample complexity of learning CE). For any $p\in (0,1]$, set $\iota = \log(2NSA_{\max} KH/p)$, and let the agents run Algorithm~\ref{alg:ce} for $K$ episodes with $K= O(S A_{\max}^2H^5 \iota/\epsilon^2)$. Then, with probability at least $1-p$, the output policy $\bar{\pi}$ constitutes an $\epsilon$-approximate correlated equilibrium. 

\section{Proofs for Section~\ref{subsec:exact}}\label{app:mpg}

We start with a multi-agent variant of the performance difference lemma~\citep{kakade2002approximately} in the finite-horizon setting. 

\begin{lemma}(Performance difference lemma).\label{lemma:performance_diff}
	For any policy $\pi=(\pi_i,\pi_{-i})\in\Pi$ and $\pi'=(\pi_i',\pi_{-i})\in\Pi$, it holds for any $i\in\mc{N}$ that
	\[
	V_{1,i}^{\pi}(\rho) - V_{1,i}^{\pi'}(\rho) = \sum_{h=1}^H \ee_{s_h\sim d_{h,\rho}^\pi}\ee_{\bm{a}_h\sim \pi_h(\cdot \mid s_h)} \L A_{h,i}^{\pi'}(s_h,\bm{a}_h) \R,
	\]
	where $A_{h,i}^{\pi'}(s_h,\bm{a}_h) = Q_{h,i}^{\pi'}(s_h,\bm{a}_h) - V_{h,i}^{\pi'}(s_h)$ is the advantage function. 
\end{lemma}
\begin{proof}
	For any state-action trajectory $\tau = (s_1,\bm{a}_1,\dots,s_H,\bm{a}_H)$, let $\pp^\pi(\tau\mid \rho)$ denote the probability of observing the trajectory $\tau$ by following the policy $\pi$ starting from the initial state distribution $\rho$. From the definition of the value function, it holds that
	\[
	\begin{aligned}
	V_{1,i}^{\pi}(\rho) - V_{1,i}^{\pi'}(\rho) =& \ee_{\tau\sim \pp^\pi(\tau\mid \rho)}\L \sum_{h=1}^H r_h(s_h,\bm{a}_h) \R - V_{1,i}^{\pi'}(\rho)\\
	=&\ee_{\tau\sim \pp^\pi(\tau\mid \rho)}\L \sum_{h=1}^H \l r_h(s_h,\bm{a}_h) + V_{h,i}^{\pi'}(s_h) - V_{h,i}^{\pi'}(s_h) \r\R - V_{1,i}^{\pi'}(\rho)\\
	=&\ee_{\tau\sim \pp^\pi(\tau\mid \rho)}\L \sum_{h=1}^{H-1} \l r_h(s_h,\bm{a}_h) + V_{h+1,i}^{\pi'}(s_{h+1}) - V_{h,i}^{\pi'}(s_h) \r + r_H(s_H,\bm{a}_H) - V_{H,i}^{\pi'}(s_H)\R\\
	\eqa & \ee_{\tau\sim \pp^\pi(\tau\mid \rho)}\L \sum_{h=1}^{H-1} \l r_h(s_h,\bm{a}_h) + \ee \L V_{h+1,i}^{\pi'}(s_{h+1}) \mid s_h,\bm{a}_h \R - V_{h,i}^{\pi'}(s_h) \r + r_H(s_H,\bm{a}_H) - V_{H,i}^{\pi'}(s_H)\R\\
	\eqb &\ee_{\tau\sim \pp^\pi(\tau\mid \rho)}\L \sum_{h=1}^{H}A_{h,i}^{\pi'}(s_h,\bm{a}_h) \R = \sum_{h=1}^H \ee_{s_h\sim d_{h,\rho}^\pi}\ee_{\bm{a}_h\sim \pi_h(\cdot \mid s_h)} \L A_{h,i}^{\pi'}(s_h,\bm{a}_h) \R,
	\end{aligned}
	\]
	where $(a)$ uses the tower property of conditional expectation, and $(b)$ is due to the definition of the advantage function and the fact that $Q_{H,i}^{\pi'}(s_h,\bm{a}_H) = r_H(s_H,\bm{a}_H)$. 
\end{proof}

In the following, we reproduce a variant of the policy gradient theorem~\citep{sutton2000policy} in the setting of finite-horizon MPGs. 

\begin{lemma}\label{lemma:policy_gradient}
	(Policy gradient theorem). For any $i\in\mc{N}$, it holds that 
	\[
	\grad V_{1,i}^\pi(\rho) = \sum_{h=1}^H \ee_{s_h\sim d_{h,\rho}^\pi}\ee_{\bm{a}_h\sim \pi_h(\cdot\mid s_h)}\L Q_{h,i}^\pi (s_h,\bm{a}_h)\grad \log \pi_h(\bm{a}_h \mid s_h) \R.
	\]
\end{lemma}
\begin{proof}
	For any fixed initial state $s_1\in\mc{S}$, differentiating both sides of the Bellman equation leads to
	\[
	\begin{aligned}
	\grad V_{1,i}^\pi(s_1) = &\grad \sum_{\bm{a}_1} \pi_1(\bm{a}_1\mid s_1) Q_{1,i}^\pi (s_1,\bm{a}_1)\\
	=& \sum_{\bm{a}_1} \l \grad \pi_1(\bm{a}_1\mid s_1) Q_{1,i}^\pi (s_1,\bm{a}_1) + \pi_1(\bm{a}_1\mid s_1) \grad Q_{1,i}^\pi (s_1,\bm{a}_1) \r\\
	=& \sum_{\bm{a}_1} \l \pi_1(\bm{a}_1\mid s_1)\grad \log \pi_1(\bm{a}_1\mid s_1) Q_{1,i}^\pi (s_1,\bm{a}_1) + \pi_1(\bm{a}_1\mid s_1) \grad \l r_{1,i}(s_1,\bm{a}_1) + \ee_{s_2 \sim P_1(\cdot \mid s_1,\bm{a}_1)}\L V_{2,i}^\pi(s_2) \R \r \r\\
	=& \sum_{\bm{a}_1}  \pi_1(\bm{a}_1\mid s_1) \l \grad \log \pi_1(\bm{a}_1\mid s_1) Q_{1,i}^\pi (s_1,\bm{a}_1) + \ee_{s_2 \sim P_1(\cdot \mid s_1,\bm{a}_1)}\L \grad V_{2,i}^\pi(s_2) \R  \r .
	\end{aligned}
	\]
	From the linearity of expectation, we know that for any state distribution $\rho$, 
	\[
	\begin{aligned}
	\grad V_{1,i}^\pi(\rho) = \grad \l \ee_{s_1\sim d_{1,\rho}^\pi}\L V_{1,i}^\pi(s_1) \R \r= \ee_{s_1\sim d_{1,\rho}^\pi}\ee_{\bm{a}_1\sim \pi_1(\cdot\mid s_1)}   \L Q_{1,i}^\pi (s_1,\bm{a}_1) \grad \log \pi_1(\bm{a}_1\mid s_1) \R + \ee_{s_2 \sim d_{2,\rho}^\pi}\L \grad V_{2,i}^\pi(s_2) \R.
	\end{aligned}
	\]
	Applying the above equation recursively, we obtain that
	\[
	\grad V_{1,i}^\pi(\rho) = \sum_{h=1}^H \ee_{s_h\sim d_{h,\rho}^\pi}\ee_{\bm{a}_h\sim \pi_h(\cdot\mid s_h)}\L Q_{h,i}^\pi (s_h,\bm{a}_h)\grad \log \pi_h(\bm{a}_h \mid s_h) \R.
	\]
	This completes the proof of the policy gradient theorem in the finite-horizon case. 
\end{proof}

With direct parameterization, we can further derive from the policy gradient theorem that for any $ h\in[H], s\in\mc{S},a\in\mc{A}_i$,
\begin{equation}\label{eqn:policy_gradient}
\frac{\partial V_{1,i}^\pi(\rho)}{\partial \pi_{h,i}(a\mid s)} = \ee_{{a}_{h,-i}\sim \pi_{h,-i}(\cdot \mid s_h)} \L d_{h,\rho}^\pi(s) Q_{h,i}^\pi(s,a,a_{h,-i})\R . 
\end{equation}

In the following, we state and prove a finite-horizon variant of the gradient domination property, which has been shown in single-agent policy gradient methods~\citep{agarwal2021theory} and infinite-horizon discounted MPGs~\citep{zhang2021gradient,leonardos2021global}. 

\begin{lemma}\label{lemma:gradient_domination}
	(Gradient domination). For any policy $\pi = (\pi_i,\pi_{-i})\in\Pi$ in a Markov potential game, let $\pi_i^\star$ be agent $i$'s best response to $\pi_{-i}$, and let $\pi^\star = (\pi_i^\star,\pi_{-i})$. With direct policy parameterization, for any initial state distribution $\rho\in\Delta(\mc{S})$, it holds that
	\[
	V_{1,i}^{\pi^\star}(\rho) - V_{1,i}^\pi(\rho) \leq \norm{\frac{d_{\rho}^{\pi^\star}}{d_{\rho}^{\pi}}}_{\infty} \max_{\ol{\pi}_i\in\Pi_i} (\ol{\pi} - \pi)\T \grad_{\pi_i} V_{1,i}^\pi(\rho), 
	\]
	where the $L^{\infty}$ norm takes the maximum over $[H]\times\mc{S}$.
\end{lemma}
\begin{proof}
	From the performance difference lemma (Lemma~\ref{lemma:performance_diff}), we know that 
	\[
	\begin{aligned}
	V_{1,i}^{\pi^\star}(\rho) - V_{1,i}^\pi(\rho) = &\sum_{h=1}^H \ee_{s_h\sim d_{h,\rho}^{\pi^\star}}\ee_{{a}_{h,i}\sim \pi_{h,i}^\star(\cdot \mid s_h)}\ee_{{a}_{h,-i}\sim \pi_{h,-i}(\cdot \mid s_h)} \L A_{h,i}^{\pi}(s_h,a_{h,i},a_{h,-i}) \R\\
	\leq & \sum_{h=1}^H \sum_{s_h\in\mc{S}} d_{h,\rho}^{\pi^\star}(s_h) \max_{a_{h,i}^\star\in\mc{A}_i} \sum_{a_{h,-i}\in\mc{A}_{-i}} \pi_{h,-i}(a_{h,-i} \mid s_h) A_{h,i}^\pi(s_h,a_{h,i}^\star,a_{h,-i})\\
	=&\sum_{h=1}^H \sum_{s_h\in\mc{S}} \frac{d_{h,\rho}^{\pi^\star}(s_h)}{d_{h,\rho}^{\pi}(s_h)}d_{h,\rho}^{\pi}(s_h) \max_{a_{h,i}^\star\in\mc{A}_i} \sum_{a_{h,-i}\in\mc{A}_{-i}} \pi_{h,-i}(a_{h,-i}\mid s_h) A_{h,i}^\pi(s_h,a_{h,i}^\star,a_{h,-i})\\
	\leq & \l \max_{h\in[H],s_h\in\mc{S}} \frac{d_{h,\rho}^{\pi^\star}(s_h)}{d_{h,\rho}^{\pi}(s_h)} \r \sum_{h=1}^H \sum_{s_h\in\mc{S}} d_{h,\rho}^{\pi}(s_h) \max_{a_{h,i}^\star\in\mc{A}_i} \sum_{a_{h,-i}\in\mc{A}_{-i}} \pi_{h,-i}(a_{h,-i} \mid s_h) A_{h,i}^\pi(s_h,a_{h,i}^\star,a_{h,-i}),
	\end{aligned}
	\]
	where in the last step we used the fact that 
	\begin{equation}\label{eqn:adv}
	\sum_{a_{h,i}\in\mc{A}_i} \sum_{a_{h,-i}\in\mc{A}_{-i}} \pi_{h,i}(a_{h,i} \mid s_h) \pi_{h,-i}(a_{h,-i} \mid s_h) A_{h,i}^\pi(s_h,a_{h,i},a_{h,-i}) = 0,
	\end{equation}
	which in turn implies that 
	$$\max_{a_{h,i}^\star\in\mc{A}_i} \sum_{a_{h,-i}\in\mc{A}_{-i}} \pi_{h,-i}(a_{h,-i} \mid s_h) A_{h,i}^\pi(s_h,a_{h,i}^\star,a_{h,-i})\geq 0.$$ 
	
	Further, we have that
	\[
	\begin{aligned}
	&\sum_{h=1}^H \sum_{s_h\in\mc{S}} d_{h,\rho}^{\pi}(s_h)\max_{a_{h,i}^\star\in\mc{A}_i} \sum_{a_{h,-i}\in\mc{A}_{-i}} \pi_{h,-i}(a_{h,-i} \mid s_h) A_{h,i}^\pi(s_h,a_{h,i}^\star,a_{h,-i})\\
	=&\sum_{h=1}^H \sum_{s_h\in\mc{S}} d_{h,\rho}^{\pi}(s_h)\max_{\ol{\pi}_i\in\Pi_i} \sum_{a_{h,i}\in\mc{A}_i} \sum_{a_{h,-i}\in\mc{A}_{-i}} \ol{\pi}_{h,i}(a_{h,i} \mid s_h) \pi_{h,-i}(a_{h,-i} \mid s_h) A_{h,i}^\pi(s_h,a_{h,i},a_{h,-i}) \\
	\eqa&\sum_{h=1}^H \sum_{s_h\in\mc{S}} d_{h,\rho}^{\pi}(s_h)\max_{\ol{\pi}_i\in\Pi_i} \sum_{a_{h,i}\in\mc{A}_i} \sum_{a_{h,-i}\in\mc{A}_{-i}} \l\ol{\pi}_{h,i}(a_{h,i} \mid s_h) - {\pi}_{h,i}(a_{h,i} \mid s_h)\r \pi_{h,-i}(a_{h,-i} \mid s_h) A_{h,i}^\pi(s_h,a_{h,i},a_{h,-i}) \\
	\eqb & \sum_{h=1}^H \sum_{s_h\in\mc{S}} d_{h,\rho}^{\pi}(s_h)\max_{\ol{\pi}_i\in\Pi_i} \sum_{a_{h,i}\in\mc{A}_i} \sum_{a_{h,-i}\in\mc{A}_{-i}} \l\ol{\pi}_{h,i}(a_{h,i} \mid s_h) - {\pi}_{h,i}(a_{h,i} \mid s_h)\r \pi_{h,-i}(a_{h,-i} \mid s_h) Q_{h,i}^\pi(s_h,a_{h,i},a_{h,-i})\\
	\eqc& \max_{\ol{\pi}_i\in\Pi_i} (\ol{\pi} - \pi)\T \grad_{\pi_i} V_{1,i}^\pi(\rho)
	\end{aligned}
	\]
	where $(a)$ again uses \eqref{eqn:adv}, and $(b)$ relies on the fact that
	\[
	\sum_{a_{h,i}\in\mc{A}_i}  \l\ol{\pi}_{h,i}(a_{h,i} \mid s_h) - {\pi}_{h,i}(a_{h,i} \mid s_h)\r  V_{h,i}^\pi(s_h) = 0.
	\]
	Equality $(c)$ is due to the policy gradient theorem with direct parameterization~\eqref{eqn:policy_gradient}. Finally, putting everything together, we conclude that 
	\[
	\begin{aligned}
	V_{1,i}^{\pi^\star}(\rho) - V_{1,i}^\pi(\rho) \leq \norm{\frac{d_{\rho}^{\pi^\star}}{d_{\rho}^{\pi}}}_{\infty} \max_{\ol{\pi}_i\in\Pi_i} (\ol{\pi} - \pi)\T \grad_{\pi_i} V_{1,i}^\pi(\rho),
	\end{aligned}
	\]
	where the $L^{\infty}$ norm takes the maximum over the space $[H]\times\mc{S}$. This completes the proof of the gradient domination property. 
\end{proof}

We are now ready to prove Lemma~\ref{lemma:stationarity_implies_nash} from Section~\ref{sec:mpg}, which states that a stationary point of the potential function implies a NE policy. 

\vspace{.8em}
\noindent \textbf{Lemma \ref{lemma:stationarity_implies_nash}.} Let $\pi=(\pi_1,\dots,\pi_N)$ be an $\epsilon$-approximate stationary point of the potential function $\Phi_\rho$ of an MPG for some $\epsilon>0$. Then, $\pi$ is a $D\sqrt{SH}\epsilon$-approximate Nash equilibrium policy profile for this MPG. 
\begin{proof}
	For any $i\in\mc{N}$, since $\pi = (\pi_i,\pi_{-i})$ is an $\epsilon$-approximate stationary point of $\Phi_\rho$, we know from Definition~\ref{def:stationary_point} that
	\[
	\max_{\pi_i^\star \in \Pi_i} (\pi_i^\star - \pi_i)\T \grad_{\pi_i} \Phi_\rho(\pi) = \sqrt{SH} \max_{\pi_i^\star \in \Pi_i}\l\frac{ \pi_i^\star - \pi_i}{\sqrt{SH}}\r\T \grad_{\pi_i} \Phi_\rho(\pi) \leq \sqrt{SH}\epsilon,
	\]
	where we used the fact that $\norm{ \frac{ \pi_i^\star - \pi_i}{\sqrt{SH}}}_2^2\leq 1$. Let $\pi^\star = (\pi_i^\star,\pi_{-i})$. From the definition of the potential function, we obtain that $\grad_{\pi_i} V_{1,i}^\pi(\rho) = \grad_{\pi_i} \Phi_\rho(\pi)$. Further, the linearity of expectation immediately implies that $\Phi_\rho(\pi_i,\pi_{-i}) - \Phi_\rho(\pi_{i'},\pi_{-i}) = V_{1,i}^{\pi_i,\pi_{-i}}(\rho) - V_{1,i}^{\pi_{i'},\pi_{-i}}(\rho).$ By the gradient domination property (Lemma~\ref{lemma:gradient_domination}), we know that 
	\[
	\begin{aligned}
	V_{1,i}^{\pi^\star}(\rho) - V_{1,i}^{\pi}(\rho) \leq& \norm{\frac{d_{\rho}^{\pi^\star}}{d_{\rho}^{\pi}}}_{\infty} \max_{{\pi}^\star_i\in\Pi_i} ({\pi}^\star - \pi)\T \grad_{\pi_i} V_{1,i}^\pi(\rho)\\
	=& D\max_{{\pi}^\star_i\in\Pi_i} ({\pi}^\star - \pi)\T \grad_{\pi_i} \Phi_\rho(\pi)\\
	\leq & D\sqrt{SH}\epsilon. 
	\end{aligned}
	\]
	Since the above inequality holds for any $i\in\mc{N}$, we conclude that $\pi$ is a $D\sqrt{SH}\epsilon$-approximate Nash equilibrium policy profile of the MPG. 
\end{proof}

Before proceeding to the proof of Theorem~\ref{thm:mpg_exact_gradient}, we first state and prove the following supporting lemmas. The first lemma investigates the smoothness of the potential function, while the second one ensures that the projected gradient descent algorithm~\eqref{eqn:pga} can be executed in a decentralized way. 

\begin{lemma}\label{lemma:smoothness}
	For any state distribution $\rho$, the potential function $\Phi_\rho$ is $4NA_{\max}H^3$-smooth; that is, 
	\[
	\norm{\grad \Phi_\rho(\pi) - \grad\Phi_{\rho} (\pi')}_2 \leq 4NA_{\max}H^3\norm{\pi-\pi'}_2. 
	\]
\end{lemma}
\begin{proof}
	It suffices to show that 
	\[
	\norm{\grad^2 \Phi_\rho}_2 \leq 4NA_{\max}H^3.
	\]
	From Claim C.2 of \citet{leonardos2021global} (restated as Lemma~\ref{lemma:block_matrix} in Appendix~\ref{app:lemmas}), we know that we only need to show  
	$$\norm{\grad_{\pi_j \pi_i} V_{1,j}^\pi(\rho)}_2\leq 4A_{\max}H^3, \forall i,j\in\mc{N},$$
	and the desired result immediately follows. 
	
	Our proof follows a similar argument as in \citet{agarwal2021theory,leonardos2021global}. For a fixed policy profile $\pi$, initial state $s_1$, and agents $i\neq j\in\mc{N}$, let $s,t\geq 0$ be scalars and $u,v$ be unit vectors such that $\pi_i + tu\in\Pi_i$ and $\pi_j+sv\in\Pi_j$. Further, define
	\[
	V(t) = V_{1, i}^{(\pi_i+tu,\pi_{-i})}(s_1), \text{ and } W(t,s) = V_{1, i}^{(\pi_i+tu,\pi_j+sv,\pi_{-i,-j})}(s_1). 
	\]
	Then, it suffices to show that 
	\begin{equation}\label{eqn:mpg5}
	\max _{\|u\|_{2}=1}\left|\frac{d^{2} V(0)}{d t^{2}}\right| \leq 4A_{\max}H^3, \text { and } \max _{\|u\|_{2}=1}\left|\frac{d^{2} W(0,0)}{d t d s}\right| \leq 4A_{\max}H^3.
	\end{equation}
	We start with the first inequality. From the Bellman equation, we know that 
	\[
	V(t) = \sum_{a_{1,i}\in\mc{A}_i} \sum_{a_{1,-i}\in\mc{A}_{-i}}\l \pi_{1,i}(a_{1,i}\mid s_1) + tu_1(a_{1,i}\mid s_1)\r \prod_{j\neq i}\pi_{1,j}(a_{1,j} \mid s_1) Q_{1,i}^{(\pi_i+tu,\pi_{-i})}(s_1,\bm{a}_1),
	\]
	and in what follows we will write $\pi_{h,-i}(a_{h,-i}\mid s_h) = \prod_{j\neq i}\pi_{h,j}(a_{h,j} \mid s_h)$ for short. Taking the second derivative on both sides,
	\begin{align}
	\frac{d^2 V(t)}{dt^2} =& 2\sum_{a_{1,i}\in\mc{A}_i} \sum_{a_{1,-i}\in\mc{A}_{-i}} u_1(a_{1,i}\mid s_1) \pi_{1,-i}(a_{1,-i}\mid s_1)\frac{d Q_{1,i}^{(\pi_i+tu,\pi_{-i})}(s_1,\bm{a}_1)}{dt}\nonumber\\
	&+ \sum_{a_{1,i}\in\mc{A}_i} \sum_{a_{1,-i}\in\mc{A}_{-i}}( \pi_{1,i}(a_{1,i}\mid s_1) + tu_1(a_{1,i}\mid s_1))\pi_{1,-i}(a_{1,-i}\mid s_1)\frac{d^2 Q_{1,i}^{(\pi_i+tu,\pi_{-i})}(s_1,\bm{a}_1)}{dt^2}.\label{eqn:mpg8}
	\end{align}
	In the following, we will bound each of the two terms on the RHS separately. Let $\pi(t) = (\pi_i+tu,\pi_{-i})$. From the Bellman equation, we know that for any $h\in[H]$,
	\begin{align}
	Q_{h,i}^{\pi(t)}(s_h,\bm{a}_h) = &r_{h,i}(s_h,\bm{a}_h) + \sum_{s_{h+1}} P_h(s_{h+1}\mid s_h,\bm{a}_h) \sum_{a_{h+1,i}}\sum_{a_{h+1,-i}} \l \pi_{h+1,i}(a_{h+1,i}\mid s_{h+1}) + tu_{h+1}(a_{h+1,i} \mid s_{h+1}) \r\nonumber\\
	&\times \pi_{h+1,-i}(a_{h+1,-i}\mid s_{h+1}) Q_{h+1,i}^{\pi(t)}(s_{h+1},\bm{a}_{h+1}).  \label{eqn:mpg7}
	\end{align}
	Differentiating both sides of the equation, 
	\[
	\begin{aligned}
	\abs{\frac{d Q_{h,i}^{\pi(t)}(s_h,\bm{a}_h)}{dt}} \leq &\abs{\sum_{s_{h+1}} P_h(s_{h+1}\mid s_h,\bm{a}_h) \sum_{a_{h+1,i}}\sum_{a_{h+1,-i}} \l \pi_{h+1,i} + tu_{h+1} \r\pi_{h+1,-i} \frac{d Q_{h+1,i}^{\pi(t)}(s_{h+1},\bm{a}_{h+1})}{dt}}\\
	&+ \abs{\sum_{s_{h+1}} P_h(s_{h+1}\mid s_h,\bm{a}_h) \sum_{a_{h+1,i}}\sum_{a_{h+1,-i}} u_{h+1} \pi_{h+1,-i} Q_{h+1,i}^{\pi(t)}(s_{h+1},\bm{a}_{h+1})}\\
	\leq & \abs{\sum_{s_{h+1}} P_h(s_{h+1}\mid s_h,\bm{a}_h) \sum_{a_{h+1,i}}\sum_{a_{h+1,-i}} \l \pi_{h+1,i} + tu_{h+1} \r\pi_{h+1,-i} \frac{d Q_{h+1,i}^{\pi(t)}(s_{h+1},\bm{a}_{h+1})}{dt}} + \sqrt{A_{\max}}H,
	\end{aligned}
	\]
	where we abbreviated $\pi_{h+1,i}(a_{h+1,i}\mid s_{h+1})$ as $\pi_{h+1,i}$, $u_{h+1}(a_{h+1,i} \mid s_{h+1})$ as $u_{h+1}$, and $\pi_{h+1,-i}(a_{h+1,-i}\mid s_{h+1})$ as $\pi_{h+1,-i}$. The last step holds because $Q_{h+1,i}^{\pi(t)}(s_{h+1},\bm{a}_{h+1})\leq H$, $\sum_{a_{h+1,-i}} \pi_{h+1,-i}(a_{h+1,-i}\mid s_{h+1}) = 1$, $\sum_{a_{h+1,i}} \abs{u_{h+1}(a_{h+1,i} \mid s_{h+1})} \leq \sqrt{A_i}\leq \sqrt{A_{\max}}$, and $\sum_{s_{h+1}} P_h(s_{h+1}\mid s_h,\bm{a}_h) = 1$. Applying the above inequality recursively over $h = H,H-1,\dots, 1$, and recalling the facts that 
	\[
	\frac{d Q_{H,i}^{\pi(t)}(s_H,\bm{a}_H)}{dt} = \frac{d r_{H,i}(s_H,\bm{a}_H)}{dt} = 0
	\]
	and that $\sum_{a_{h+1,i}} \l \pi_{h+1,i}(a_{h+1,i}\mid s_{h+1}) + tu_{h+1}(a_{h+1,i}\mid s_{h+1}) \r = 1$ lead to the result that 
	\begin{equation}\label{eqn:mpg9}
	\abs{\frac{d Q_{h,i}^{\pi(t)}(s_h,\bm{a}_h)}{dt}} \leq \sqrt{A_{\max}}H^2,\forall h\in[H].
	\end{equation}
	Further, taking the second derivative on both sides of \eqref{eqn:mpg7}, we get that
	\[
	\begin{aligned}
	\abs{\frac{d^2 Q_{h,i}^{\pi(t)}(s_h,\bm{a}_h)}{dt^2}} \leq  &\abs{\sum_{s_{h+1}} P_h(s_{h+1}\mid s_h,\bm{a}_h) \sum_{a_{h+1,i}}\sum_{a_{h+1,-i}} \l \pi_{h+1,i} + tu_{h+1} \r\pi_{h+1,-i} \frac{d^2 Q_{h+1,i}^{\pi(t)}(s_{h+1},\bm{a}_{h+1})}{dt^2}}\\
	&+ 2\abs{\sum_{s_{h+1}} P_h(s_{h+1}\mid s_h,\bm{a}_h) \sum_{a_{h+1,i}}\sum_{a_{h+1,-i}} u_{h+1} \pi_{h+1,-i}\frac{dQ_{h+1,i}^{\pi(t)}(s_{h+1},\bm{a}_{h+1})}{dt}}\\
	\leq &\abs{\sum_{s_{h+1}} P_h(s_{h+1}\mid s_h,\bm{a}_h) \sum_{a_{h+1,i}}\sum_{a_{h+1,-i}} \l \pi_{h+1,i} + tu_{h+1} \r\pi_{h+1,-i} \frac{d^2 Q_{h+1,i}^{\pi(t)}(s_{h+1},\bm{a}_{h+1})}{dt^2}} + 2A_{\max} H^2,
	\end{aligned}
	\]
	where in the last step we used \eqref{eqn:mpg9}, and the facts  that $\sum_{a_{h+1,-i}} \pi_{h+1,-i}(a_{h+1,-i}\mid s_{h+1}) = 1$, $\sum_{a_{h+1,i}} \abs{u_{h+1}(a_{h+1,i} \mid s_{h+1})} \leq \sqrt{A_i}\leq \sqrt{A_{\max}}$, and $\sum_{s_{h+1}} P_h(s_{h+1}\mid s_h,\bm{a}_h) = 1$. Again, applying the above inequality recursively over $h = H,H-1,\dots, 1$, we obtain that
	\begin{equation}\label{eqn:mpg6}
	\abs{\frac{d^2 Q_{h,i}^{\pi(t)}(s_h,\bm{a}_h)}{dt^2}} \leq 2A_{\max}H^3,\forall h\in[H]. 
	\end{equation}
	Substituting \eqref{eqn:mpg9} and \eqref{eqn:mpg6} back into \eqref{eqn:mpg8}, we can conclude that 
	\[
	\abs{\frac{d^2 V(t)}{dt^2} }\leq 2A_{\max}H^2 + 2A_{\max}H^3 \leq 4A_{\max}H^3. 
	\]
	This proves the first inequality in \eqref{eqn:mpg5}. The second inequality in \eqref{eqn:mpg5} can be shown using a similar procedure. This completes the proof of the lemma. 
\end{proof}

\begin{lemma}\label{lemma:centralized}
	For any policy profile $\pi = (\pi_1,\dots,\pi_N)$, let $\pi^+=\text{Proj}_{\Pi}\l \pi +\eta \grad_{\pi} \Phi_{\rho}(\pi)\r$ be a PGA update step on the potential function, where $\eta>0$ is the step size. For each agent $i\in\mc{N}$, let $\pi^+_i = \text{Proj}_{\Pi_i}\l \pi_i + \eta \grad_{\pi_i}V_{1,i}^\pi(\rho)\r$ be an independent PGA update on its own value function with the same step size. Then, $\pi^+ = (\pi_1^+,\dots,\pi_N^+)$. That is, running PGA on the potential function as a whole is equivalent to running PGA independently on each agent's value function. 
\end{lemma}
\begin{proof}
	By the definition of the projection operator, 
	\[
	\begin{aligned}
	\pi^+=&\text{Proj}_{\Pi}\l \pi +\eta \grad_{\pi} \Phi_{\rho}(\pi)\r=\argmin_{x\in\Pi}\norm{\pi +\eta \grad_{\pi} \Phi_{\rho}(\pi) - x}_2^2\\
	=& \argmin_{(x_1,\dots,x_N)\in \Pi_1\times \dots\times \Pi_N} \sum_{i=1}^N \norm{\pi_i +\eta \grad_{\pi_i} \Phi_{\rho}(\pi) - x_i}_2^2\\
	\eqa &  \argmin_{(x_1,\dots,x_N)\in \Pi_1\times \dots\times \Pi_N} \sum_{i=1}^N \norm{\pi_i +\eta \grad_{\pi_i} V_{1,i}^\pi(\rho) - x_i}_2^2\\
	=& \l \argmin_{x_1\in\Pi_1}\norm{\pi_1 +\eta \grad_{\pi_1} V_{1,1}^\pi(\rho) - x_1}_2^2, \dots,  \argmin_{x_N\in\Pi_N}\norm{\pi_N +\eta \grad_{\pi_N} V_{1,N}^\pi(\rho) - x_N}_2^2 \r \\
	=& \l \pi_1^+,\dots,\pi_N^+ \r,
	\end{aligned}
	\]
	where $(a)$ is due to the fact that $\grad_{\pi_i}V_{1,i}^\pi(\rho) = \grad_{\pi_i} \Phi_\rho(\pi)$. 
\end{proof}

With Lemma~\ref{lemma:centralized} at hand, we only need to analyze the behavior of running PGA on the potential function, as it is equivalent to the case where each agent runs PGA independently on its own value function, i.e., the update rule given in \eqref{eqn:pga}. We are now ready to prove Theorem~\ref{thm:mpg_exact_gradient}. 

\vspace{.8em}
\noindent\textbf{Theorem~\ref{thm:mpg_exact_gradient}.}	For any initial state distribution $\rho\in\Delta(\mc{S})$, let the agents independently run the projected gradient ascent algorithm~\eqref{eqn:pga} with step size $\eta = \frac{1}{4NA_{\max}H^3}$ for $T = \frac{32NSA_{\max}D^2H^4\Phi_{\max}}{\epsilon^2}$ iterations. Then, there exists $t\in[T]$, such that $\pi^{(t)}$ is an $\epsilon$-approximate Nash equilibrium policy profile. 
\begin{proof}
	Let $\pi^{(t)}$ be the policy profile at the beginning of the $t$-th iteration of the PGA algorithm. Since we have shown in Lemma~\ref{lemma:smoothness} that $\Phi_{\rho}$ is $4NA_{\max}H^3$-smooth, we can apply a standard sufficient ascent property for smooth functions (Lemma~\ref{lemma:36} in Appendix~\ref{app:lemmas}) to conclude that
	\[
	\Phi_{\rho} (\pi^{(t+1)}) - \Phi_{\rho} (\pi^{(t)}) \geq \frac{1}{8NA_{\max}H^3} \norm{\frac{1}{\eta} \l \pi^{(t+1)} - \pi^{(t)}\r }_2^2. 
	\]
	Summing over $t = 1,2,\dots, T$, we know that 
	\[
	\sum_{t=1}^{T}\norm{\frac{1}{\eta}\l\pi^{(t+1)} - \pi^{(t)}\r }_2^2 \leq 8NA_{\max}H^3 \l \Phi_{\rho}(\pi^{(T+1)}) - \Phi_{\rho}(\pi^{(1)}) \r \leq 8NA_{\max}H^3\Phi_{\max}.
	\]
	When $T$ is large enough such that $T \geq \frac{32NSA_{\max}D^2H^4\Phi_{\max}}{\epsilon^2}$, we know that there exists a time step $t\in[T]$ that satisfies $\norm{\frac{1}{\eta} \l \pi^{(t+1)} - \pi^{(t)}\r }_2\leq \frac{\epsilon}{2D\sqrt{SH}}$. Then, from a standard gradient mapping property (Lemma~\ref{lemma:gradient_map} in Appendix~\ref{app:lemmas}), we know that $\pi^{(t+1)}$ is a $\frac{\epsilon}{D\sqrt{SH}}$-approximate stationary point of the potential function. Finally, invoking the result that an $\epsilon$-approximate stationary point implies an $D\sqrt{SH}\epsilon$-approximate Nash equilibrium policy profile, we can conclude that $\pi^{(t+1)}$ is an $\epsilon$-approximate NE policy profile. 
\end{proof}

\section{Proofs for Section~\ref{subsec:finite}}\label{app:mpg_finite}

\subsection{Proof of Theorem~\ref{thm:finite}}
\noindent\textbf{Lemma~\ref{lemma:unbiased_bounded}.} For any agent $i\in\mc{N}$ and any iteration $t\in[K]$, the REINFORCE gradient estimator \eqref{eqn:gradient_estimator} with $\tilde{\epsilon}$-greedy exploration is an unbiased estimator with a bounded variance:
\[
\ee_{\pi^{(t)}}\L \hat{\grad}_{\pi_i}^{(t)}(\pi^{(t)}) \R = \grad_{\pi_i}V_{1,i}^{\pi^{(t)}}(\rho),\text{ and } \ee_{\pi^{(t)}}\L \norm{\hat{\grad}_{\pi_i}^{(t)}(\pi^{(t)}) - \grad_{\pi_i}V_{1,i}^{\pi^{(t)}}(\rho)}_2^2 \R \leq \frac{A_{\max}^2 H^4}{\tilde{\epsilon}}.
\]
Further, it is mean-squared smooth, such that for any policy $\pi^{\prime(t)}\in\Pi_i$,
\[
\ee_{\pi^{(t)}}\L \norm{\hat{\grad}_{\pi_i}^{(t)}(\pi^{(t)}) - \hat{\grad}_{\pi_i'}^{(t)}(\pi^{\prime(t)})}_2^2 \R \leq \frac{A^3_{\max}H^3}{\tilde{\epsilon}^3} \norm{\pi^{(t)} - \pi^{\prime(t)}}_2^2.
\]
\begin{proof}
	In this proof, we omit the iteration index $t$ in the superscripts of the notations as there is no ambiguity. We first show that the gradient estimator is unbiased. For any state-action trajectory $\tau = (s_1,\bm{a}_1,\dots,s_H,\bm{a}_H)$, let $\pp^\pi(\tau)$ denote the probability of observing the trajectory $\tau$ by following the policy $\pi$ starting from the state distribution $\rho$, and let $R(\tau) = \sum_{h=1}^H r_h(s_h,\bm{a}_h)$ denote the total reward of the trajectory. Then, we know that
	\[
	\pp^\pi(\tau) = \prod_{h=1}^H \pi_{h,i}(a_{h,i}\mid s_h)\pi_{h,-i}(a_{h,-i}\mid s_h) P(s_{h+1}\mid s_h,\bm{a}_h).
	\]
	Therefore, by the definition of the value function,
	\[
	\begin{aligned}
	\grad_{\pi_i}V_{1,i}^{\pi}(\rho) =& \grad_{\pi_i} \sum_{\tau} R(\tau) \pp^\pi(\tau) = \sum_{\tau} R(\tau) \grad_{\pi_i}\pp^\pi(\tau)\\
	=& \sum_{\tau} R(\tau) \pp^\pi(\tau)\grad_{\pi_i}\log \pp^\pi(\tau)\\
	=& \sum_{\tau} R(\tau) \pp^\pi(\tau)\grad_{\pi_i} \l \sum_{h=1}^H \log \pi_{h,i}(a_{h,i}\mid s_h)  \r\\
	=&\ee_{\pi} \L \l \sum_{h=1}^H r_h(s_h,\bm{a}_h) \r \sum_{h=1}^H \grad_{\pi_i} \log \pi_{h,i}(a_{h,i}\mid s_h) \R\\
	=& \ee_{\pi}\L \hat{\grad}_{\pi_i} (\pi) \R
	\end{aligned}
	\]
	Next, we proceed to bound the variance of the gradient estimator. Since the gradient estimator is unbiased, 
	\[
	\begin{aligned}
	\ee_{\pi}\L \norm{\hat{\grad}_{\pi_i}(\pi) - \grad_{\pi_i}V_{1,i}^{\pi}(\rho)}_2^2 \R \leq& \ee_{\pi}\L \norm{\hat{\grad}_{\pi_i}(\pi)}_2^2 \R = \ee_{\pi}\L \norm{ R_i \sum_{h=1}^H \grad_{\pi_i} \log \pi_{h,i}(a_{h,i}\mid s_h)}_2^2 \R\\
	\leq & H^2\ee_{\pi}\L \norm{ \sum_{h=1}^H \grad_{\pi_i} \log \pi_{h,i}(a_{h,i}\mid s_h)}_2^2 \R \\
	\leq &H^3 \ee_{\pi}\L \sum_{h=1}^H \norm{\grad_{\pi_i} \log \pi_{h,i}(a_{h,i}\mid s_h)}_2^2 \R\\
	=& H^3\ee_{\pi}\L \sum_{h=1}^H \sum_{s,a_i}(1-\tilde{\epsilon})^2\mb{I}\{s = s_h,a_i = a_{h,i}\} \frac{1}{\pi_{h,i}^2(a_{i}\mid s)} \R,
	\end{aligned}
	\]
	where the last step is a consequence of direct parameterization with $\tilde{\epsilon}$-greedy exploration. We further upper bound the above term as
	\[
	\begin{aligned}
	\ee_{\pi}\L \norm{\hat{\grad}_{\pi_i}(\pi) - \grad_{\pi_i}V_{1,i}^{\pi}(\rho)}_2^2 \R \leq& H^3\ee_{\pi}\L \sum_{h=1}^H \sum_{s,a_i}\mb{I}\{s = s_h,a_i = a_{h,i}\} \frac{1}{\pi_{h,i}^2(a_{i}\mid s)} \R\\
	=&  H^3\ee_{\pi}\L \sum_{h=1}^H \sum_{s,a_i}\mb{I}\{s = s_h\} \frac{1}{\pi_{h,i}(a_{i}\mid s)} \R\\
	\leq & \frac{A_{\max}H^3}{\tilde{\epsilon}}\ee_{\pi}\L \sum_{h=1}^H \sum_{s,a_i}\mb{I}\{s = s_h\} \R\\
	\leq & \frac{A_{\max}^2 H^4}{\tilde{\epsilon}}.
	\end{aligned}
	\]
	Finally, we proceed to show the mean-squared smoothness of the gradient estimator. Using a similar argument as above, 
	\[
	\begin{aligned}
	\ee_{\pi}\L \norm{\hat{\grad}_{\pi_i}(\pi) - \hat{\grad}_{\pi_i'}(\pi')}_2^2 \R \leq& H^2\ee_{\pi}\L \norm{ \sum_{h=1}^H \l\grad_{\pi_i} \log \pi_{h,i}(a_{h,i}\mid s_h) - \grad_{\pi_i'} \log \pi_{h,i}'(a_{h,i}\mid s_h) \r}_2^2 \R\\
	\leq & H^3 \ee_{\pi}\L \sum_{h=1}^H\norm{  \grad_{\pi_i} \log \pi_{h,i}(a_{h,i}\mid s_h) - \grad_{\pi_i'} \log \pi_{h,i}'(a_{h,i}\mid s_h) }_2^2 \R\\
	\leq & H^3\ee_{\pi}\L \sum_{h=1}^H \sum_{s,a_i}\mb{I}\{s = s_h,a_i = a_{h,i}\} \l \frac{1}{\pi_{h,i}(a_{i}\mid s)} - \frac{1}{\pi_{h,i}'(a_{i}\mid s)} \r^2\R\\
	\leq & \frac{H^3 A_{\max}^3}{\tilde{\epsilon}^3}\ee_{\pi}\L \sum_{h=1}^H \sum_{s,a_i}\mb{I}\{s = s_h\} \l \pi_{h,i}(a_{i}\mid s) - \pi_{h,i}'(a_{i}\mid s) \r^2\R\\
	\leq &  \frac{H^3 A_{\max}^3}{\tilde{\epsilon}^3}.
	\end{aligned}
	\]
	This completes the proof of the lemma. 
\end{proof}

\vspace{.8em}
\noindent\textbf{Theorem~\ref{thm:finite}.} For any initial policies, let the agents independently run SGA policy updates (Algorithm~\ref{alg:storm}) for $T$ iterations with $T = O(1/\epsilon^{4.5})\cdot \text{poly}(N, D, S, A_{\max}, H)$. Then, there exists $t\in[T]$, such that $\pi^{(t)}$ is an $\epsilon$-approximate Nash equilibrium policy profile in expectation. 
\begin{proof}
	It suffices to verify that Assumption~\ref{assumption:1} is satisfied in our SGA policy updates, and then apply the convergence results from Proposition~\ref{thm:storm_app}. From Lemma~\ref{lemma:unbiased_bounded}, we know that the REINFORCE gradient estimator with $\tilde{\epsilon}$-greedy exploration is unbiased, mean-squared smooth, and has a bounded variance. We also know from Lemma~\ref{lemma:smoothness} that the potential function is smooth (The smoothness parameter does not increase with epsilon-greedy exploration, e.g., \citet[Proposition 3]{daskalakis2020independent}).  Hence, we can conclude that all the conditions in Assumption~\ref{assumption:1} are satisfied with the following choice of parameters:
	\[
	\sigma^2 = \frac{N A_{\max}^2 H^4}{\tilde{\epsilon}}, \text{ and } L^2 = \frac{N^2 A_{\max}^3 H^3}{\tilde{\epsilon}^3}.
	\]
	Applying the convergence result from Theorem~\ref{thm:storm_app}, and setting $\tilde{\epsilon} = \frac{\sqrt{N}\epsilon}{2NDSH^3A_{\max}}$, we conclude that we can obtain an $\epsilon$-approximate Nash equilibrium when $T = \widetilde{O}\l \frac{N^{9/4} D^{9/2}S^3A_{\max}^{9/2}H^{12} }{\epsilon^{9/2}} \r$. This completes the proof of the theorem. 
\end{proof}

\subsection{A Sample Complexity Lower Bound}\label{app:lowerbound}

To obtain a sample complexity lower bound of the problem, we consider a simple instance where the agents share the same reward function (which clearly satisfies the definition of an MPG) and the action spaces of all but one agent $i$ are singletons, i.e., $A_j = 1,\forall j\neq i$. Learning an approximate NE in such an MPG reduces to finding a near-optimal policy in a single-agent RL problem. Applying the regret lower bound of single-agent RL yields the following result for MARL in MPGs.
\begin{corollary}\label{corollary2} (Corollary of \citet{jaksch2010near}). 
	For any algorithm, there exists a Markov potential game that takes the algorithm at least $\Omega(H^3 S\amax/\epsilon^2)$ episodes to learn an $\epsilon$-approximate Nash equilibrium. 
\end{corollary}

We remark that such a lower bound might be very loose. Reducing to a single-agent RL problem evades the strategic learning behavior of the agents and the non-stationarity that such behavior causes to the environment, which in our opinion are the central difficulties of decentralized MARL. To derive a tighter lower bound in our decentralized setting, one should also utilize the additional constraint that each agent only has access to its local information, a factor that Corollary~\ref{corollary2} apparently does not take into account. It is hence unsurprising that when comparing Theorem~\ref{thm:finite} with Corollary~\ref{corollary2}, we can see an obvious gap in the parameter dependence. We leave the tightening of both the upper and lower bounds to our future work.

\subsection{Decentralized MARL in Smooth MPGs}\label{app:smooth}
In the following, we address an important subclass of MPGs named \emph{smooth MPGs}. We show that our independent SGA algorithm can achieve nearly global optimality, i.e., find the best Nash equilibrium, in such problems. 

Smooth games were first introduced in~\citet{roughgarden2009intrinsic} to study the Price of Anarchy (POA) in normal-form games. A large class of games are covered as examples of smooth games, including congestion games and many forms of auctions \citep{roughgarden2009intrinsic,syrgkanis2013composable}. The notion of smoothness was later extended to learning in normal-form games~\citep{syrgkanis2015fast,foster2016learning} and cooperative Markov games~\citep{radanovic2019learning,mao2020near}. This concept essentially ensures that the game has a bounded POA, and hence \emph{decentralized} no-regret learning dynamics can possibly converge to near-optimality. 

Let $\pi^\star = (\pi_i^{\star},\pi_{-i}^\star)$ be a policy that maximizes the potential function, i.e., $\Phi_{\rho}(\pi^\star) = \max_{\pi\in\Pi}\Phi_{\rho}(\pi)$. Let $V^\star_{1,i}$ denote the value function for agent $i$ under policy $\pi^\star$.  We consider the following definition of a smooth Markov potential game:

\begin{definition}\label{def:smooth}
	(Adapted from~\citet{radanovic2019learning}). For $\lambda\geq 0$ and $0<\omega<1$, an $N$-player Markov potential game is $(\lambda,\omega)$-smooth if for any policy profile $\pi = (\pi^i,\pi^{-i})$:
	\[
	\begin{aligned}
	V_{1,i}^{\pi_i^{\star}, \pi^{-i}}(s)\geq \lambda \cdot V_{1,i}^{\star}(s) - \omega\cdot V_{1,i}^{\pi}(s), \forall i \in \mc{N},s\in\mc{S}. 
	\end{aligned}
	\]
\end{definition}
The $(\lambda,\omega)$-smoothness ensures that agent $i$ continues doing well by playing its optimal policy even when the other agents are using slightly sub-optimal policies. It immediately follows that Algorithm~\ref{alg:storm} can nearly find the globally optimal NE in smooth MPGs.

\vspace{.8em}
\noindent\textbf{Theorem~\ref{thm:smooth}.} In a $(\lambda,\omega)$-smooth MPG, for any initial policies and any $\epsilon>0$, let the agents independently run SGA policy updates (Algorithm~\ref{alg:storm}) for $T$ iterations with $T = O(1/\epsilon^{4.5})\cdot \text{poly}(N, D, S, A_{\max}, H)$. Then, there exists $t\in[T]$, such that
\[
\begin{aligned}
\ee \L V_{1,i}^{\pi^{(t)}}(\rho)\R \geq  \frac{\lambda}{1+\omega} V_{1,i}^{\star}(\rho) -\frac{\epsilon}{1+\omega},\forall i \in\mc{N}.
\end{aligned}
\] 
\begin{proof}
	Since $T = O(1/\epsilon^{4.5})\cdot \text{poly}(N, D, S, A_{\max}, H)$, Theorem~\ref{thm:finite} guarantees that there exists $t\in[T]$, such that $\pi^{(t)}$ is an $\epsilon$-approximate NE in expectation. That is,
	\[
	\begin{aligned}
	\ee \L V_{1,i}^{\pi^{(t)}}(\rho)\R \geq& \ee \L V_{1,i}^{\pi_i^\star,\pi_{-i}^{(t)}}(\rho)\R - \epsilon	\geq \lambda \cdot V_{1,i}^{\star}(\rho) - \omega\cdot \ee\L V_{1,i}^{\pi^{(t)}}(\rho) \R- \epsilon,
	\end{aligned}
	\]
	where the second step is by the definition of smoothness. Rearranging the terms leads to the desired result.
\end{proof}

\section{SGD with Variance Reduction}\label{app:storm}

Before we present the convergence guarantee of Algorithm~\ref{alg:storm}, we first introduce a few notations for ease of presentations. For any $t\in[T]$, we break the update rule into two steps: 
$$
\tilde{x}_{t+1} \defeq x_t - \eta_t d_t, \text{ and } x_{t+1} = \text{Proj}_{\mc{X}}(\tilde{x}_{t+1}).
$$ 
In addition, for each $t\in[T]$, we define $x_{t+1}^+ \defeq \text{Proj}_{\mc{X}}(x_t - \eta_t \grad F(x_t))$ to be the next iterate updated using the full gradient $\grad F(x_t)$, a value we do not have access to. Define $\epsilon_t \defeq d_t - \grad F(x_t)$ to be the error in $d_t$. The high-level procedure of our proof is to seek to upper bound the value $\e{\sum_{t=1}^{T}\norm{\frac{1}{\eta_t}\l x_{t+1}^+ - x_t\r }^2}$, and then to invoke the gradient mapping property in Lemma~\ref{lemma:gradient_map} to conclude with a stationary point. This is in contrast with the unconstrained case, where \citet{cutkosky2019momentum} directly derive an upper bound of $\e{\sum_{t=1}^{T}\norm{\grad F(x_t)}^2}$. In the following, we start with a few technical lemmas.

\begin{lemma}\label{lemma:1}
	Suppose $\eta_t \leq \frac{1}{4L}$ for all $t\in[T]$. Then, 
	\[
	\ee[F(x_{t+1}) - F(x_t)] \leq \ee\L -\frac{3}{16\eta_t}\norm{x_{t+1}^+ - x_t}^2 + \frac{7\eta_t}{8}\norm{\epsilon_t}^2\R.
	\]
\end{lemma}
\begin{proof}
	From the first-order optimality condition, we know that 
	\[
	\inner{x-x_{t+1}, x_{t+1} - (x_t -\eta_t d_t)}\geq 0,
	\]
	for any $x\in\mc{X}$. Taking $x = x_t$ leads to
	\[
	\inner{x_t-x_{t+1}, x_{t+1} - x_t} + \inner{x_t-x_{t+1}, \eta_t d_t} \geq 0,
	\]
	which in turn implies that
	\begin{equation}\label{eqn:z9}
	\inner{x_{t+1}-x_t,d_t}\leq -\frac{1}{\eta_t} \norm{x_t-x_{t+1}}^2. 
	\end{equation}
	It follows that
	\[
	\begin{aligned}
	\inner{\grad F(x_t),x_{t+1}-x_t} =& \inner{d_t-\epsilon_t, x_{t+1}-x_t}\\
	\leq & -\frac{1}{\eta_t} \norm{x_t-x_{t+1}}^2 - \inner{\epsilon_t,x_{t+1}-x_t}\\
	\leq & -\frac{1}{\eta_t} \norm{x_t-x_{t+1}}^2 + \frac{\eta_t}{2}\norm{\epsilon_t}^2 + \frac{1}{2\eta_t}\norm{x_{t+1}-x_t}^2\\
	=& -\frac{1}{2\eta_t} \norm{x_t-x_{t+1}}^2 + \frac{\eta_t}{2}\norm{\epsilon_t}^2, 
	\end{aligned}
	\]
	where the first inequality uses \eqref{eqn:z9}, and the second inequality is due to H\"{o}lder's inequality and Young's inequality. From the smoothness of $F$,
	\begin{align}
	\ee[F(x_{t+1})] \leq& \ee\L F(x_t) + \inner{\grad F(x_t), x_{t+1}-x_t} + \frac{L}{2}\norm{x_{t+1}-x_t}^2 \R\nonumber\\
	\leq &  \ee\L F(x_t)-\frac{1}{2\eta_t} \norm{x_t-x_{t+1}}^2 + \frac{\eta_t}{2}\norm{\epsilon_t}^2 + \frac{L}{2}\norm{x_{t+1}-x_t}^2 \R\nonumber\\
	\leq & \ee\L F(x_t)-\frac{3}{8\eta_t} \norm{x_t-x_{t+1}}^2 + \frac{\eta_t}{2}\norm{\epsilon_t}^2  \R,\label{eqn:z8}
	\end{align}
	where the last step uses $\eta_t \leq \frac{1}{4L}$. From the fact that $\norm{x+y}^2 \leq 2\norm{x}^2 + 2\norm{y}^2$, we know
	\[
	\norm{x_{t+1}^+ - x_t}^2 = \norm{x_{t+1}^+ - x_{t+1}+x_{t+1} - x_t}^2 \leq 2\norm{x_{t+1}^+ - x_{t+1}}^2 + 2\norm{x_{t+1} - x_t}^2.
	\]
	Rearranging the terms leads to
	\begin{align}
	- \norm{x_{t} - x_{t+1}}^2 \leq& \norm{x_{t+1}^+ - x_{t+1}}^2 - \frac{1}{2}\norm{x_{t+1}^+ - x_t}^2\nonumber\\
	\leq & \norm{\text{Proj}_{\mc{X}}(x_t - \eta_t \grad F(x_t)) - \text{Proj}_{\mc{X}}(x_t - \eta_t d_t)}^2- \frac{1}{2}\norm{x_{t+1}^+ - x_t}^2\nonumber\\
	\leq & \norm{(x_t - \eta_t \grad F(x_t)) - (x_t - \eta_t d_t)}^2- \frac{1}{2}\norm{x_{t+1}^+ - x_t}^2\nonumber\\
	= &\eta_t^2 \norm{\epsilon_t}^2 - \frac{1}{2}\norm{x_{t+1}^+ - x_t}^2 \label{eqn:z7}. 
	\end{align}
	The second inequality uses the definition of $x_{t+1}^+$. The third step holds because the projection operator is non-expansive, i.e. $\norm{\text{Proj}_{\mc{X}}(x) - \text{Proj}_{\mc{X}}(y)}\leq \norm{x-y}$. Substituting \eqref{eqn:z7} back to \eqref{eqn:z8} leads to
	\[
	\ee[F(x_{t+1})] \leq \ee\L F(x_t)-\frac{3}{16\eta_t} \norm{x_{t+1}^+ - x_t}^2 + \frac{7\eta_t}{8}\norm{\epsilon_t}^2  \R.
	\]
	Rearranging the terms completes the proof. 
\end{proof}

\begin{lemma}(Lemma 3 in \citet{cutkosky2019momentum}). \label{lemma:3}
	For any $t\in[T]$, it holds that
	\[
	\begin{aligned}
	&\ee\L \frac{(1-a_t)^2}{\eta_{t-1}}(\grad f(x_t,\xi_t) - \grad F(x_t))\cdot \epsilon_{t-1} \R = 0,\\
	&\ee\L \frac{(1-a_t)^2}{\eta_{t-1}} (\grad f(x_t,\xi_t) - \grad f(x_{t-1},\xi_t) - \grad F(x_t)+\grad F(x_{t-1})) \cdot \epsilon_{t-1} \R=0.
	\end{aligned}
	\]
\end{lemma}

\begin{lemma} (Adapted from Lemma 5 in \citet{cutkosky2019momentum}).\label{lemma:5}
	With the notations in Algorithm~\ref{alg:storm}, we have
	\[
	\mathbb{E}\left[\frac{\left\|{\epsilon}_{t}\right\|^{2}}{\eta_{t-1}}\right] \leq \mathbb{E}\left[2 c^{2} \eta_{t-1}^{3}\sigma^2+\frac{(4L^2 \eta_{t-1}^2+1)(1-a_t)^2}{\eta_{t-1}}\norm{\epsilon_{t-1}}^2  + \frac{4\left(1-a_{t}\right)^{2}L^2}{\eta_{t-1}}\norm{x_t^+-x_{t-1}}^2 \right].
	\]
\end{lemma}
\begin{proof}
	First, observe that
	\begin{align}
	& \ee\L \norm{\grad f(x_t,\xi_t) - \grad f(x_{t-1},\xi_t) - \grad F(x_t) + \grad F(x_{t-1})}^2 \R\nonumber\\
	\leq & \ee\L \norm{\grad f(x_t,\xi_t) - \grad f(x_{t-1},\xi_t)}^2 + \norm{\grad F(x_t) - \grad F(x_{t-1})}^2\R\nonumber \\
	& - 2\ee\L \inner{\grad f(x_t,\xi_t) - \grad f(x_{t-1},\xi_t), \grad F(x_t) - \grad F(x_{t-1})} \R\nonumber\\
	\leq & \ee\L \norm{\grad f(x_t,\xi_t) - \grad f(x_{t-1},\xi_t)}^2 + \norm{\grad F(x_t) - \grad F(x_{t-1})}^2\R\nonumber \\
	& - 2\ee\L \ee\L \inner{\grad f(x_t,\xi_t) - \grad f(x_{t-1},\xi_t), \grad F(x_t) - \grad F(x_{t-1})}\mid \xi_1,\dots,\xi_{t-1} \R\R\nonumber\\
	=&\ee\L \norm{\grad f(x_t,\xi_t) - \grad f(x_{t-1},\xi_t)}^2 - \norm{\grad F(x_t) - \grad F(x_{t-1})}^2\R\nonumber\\
	\leq & \ee\L \norm{\grad f(x_t,\xi_t) - \grad f(x_{t-1},\xi_t)}^2\R\label{eqn:z6}.
	\end{align}
	By the definition of $\epsilon_t$, we have $\epsilon_t = d_t - \grad F(x_t) = \grad f(x_t,\xi_t) + (1-a_t)(d_{t-1}-\grad f(x_{t-1},\xi_t)) - \grad F(x_t)$. Therefore,
	\[
	\begin{aligned}
	\ee\L \frac{\left\|{\epsilon}_{t}\right\|^{2}}{\eta_{t-1}}\R =& \ee\L \frac{1}{\eta_{t-1}}\norm{\grad f(x_t,\xi_t) + (1-a_t)(d_{t-1}-\grad f(x_{t-1},\xi_t)) - \grad F(x_t)}^2 \R\\
	= & \ee \L \frac{1}{\eta_{t-1}} \| a_t (\grad f(x_t,\xi_t) - \grad F(x_t)) + (1-a_t)(d_{t-1}-\grad F(x_{t-1}))\right. \\
	&\quad  + (1-a_t) (\grad f(x_t,\xi_t) - \grad f(x_{t-1},\xi_t) - \grad F(x_t)+ \grad F(x_{t-1})) \|^2 \bigg]\\
	\leq & \mathbb{E}\left[2 c^{2} \eta_{t-1}^{3}\left\|\nabla f\left({x}_{t}, \xi_{t}\right)-\nabla F\left({x}_{t}\right)\right\|^{2}+\frac{1}{\eta_{t-1}}(1-a_t)^2\norm{\epsilon_{t-1}}^2 \right.\\
	& \quad + \frac{2}{\eta_{t-1}}\left(1-a_{t}\right)^{2}\left\|\nabla f\left({x}_{t}, \xi_{t}\right)-\nabla f\left({x}_{t-1}, \xi_{t}\right)-\nabla F\left({x}_{t}\right)+\nabla F\left({x}_{t-1}\right)\right\|^{2}\bigg],
	\end{aligned}
	\]
	where in the last step we used Lemma~\ref{lemma:3} and the simple fact that $\norm{x+y}^2 \leq 2\norm{x}^2+2\norm{y}^2$. Further applying \eqref{eqn:z6} and the assumption that $\ee\L \norm{\grad f(x_t,\xi_t) - \grad F(x_t)}^2 \R\leq \sigma^2$ leads to
	\[
	\begin{aligned}
	\ee\L \frac{\left\|{\epsilon}_{t}\right\|^{2}}{\eta_{t-1}}\R =&\mathbb{E}\left[2 c^{2} \eta_{t-1}^{3}\sigma^2+\frac{(1-a_t)^2}{\eta_{t-1}}\norm{\epsilon_{t-1}}^2  + \frac{2\left(1-a_{t}\right)^{2}}{\eta_{t-1}}\norm{\grad f(x_t,\xi_t) - \grad f(x_{t-1},\xi_t)}^2\right]\\
	\leq & \mathbb{E}\left[2 c^{2} \eta_{t-1}^{3}\sigma^2+\frac{(1-a_t)^2}{\eta_{t-1}}\norm{\epsilon_{t-1}}^2  + \frac{2\left(1-a_{t}\right)^{2}L^2}{\eta_{t-1}}\norm{x_t-x_{t-1}}^2\right]\\
	\leq & \mathbb{E}\left[2 c^{2} \eta_{t-1}^{3}\sigma^2+\frac{(1-a_t)^2}{\eta_{t-1}}\norm{\epsilon_{t-1}}^2  + \frac{4\left(1-a_{t}\right)^{2}L^2}{\eta_{t-1}}\l \norm{x_t-x_t^+}^2+ \norm{x_t^+-x_{t-1}}^2\r \right]\\
	\leq & \mathbb{E}\left[2 c^{2} \eta_{t-1}^{3}\sigma^2+\frac{(4L^2 \eta_{t-1}^2+1)(1-a_t)^2}{\eta_{t-1}}\norm{\epsilon_{t-1}}^2  + \frac{4\left(1-a_{t}\right)^{2}L^2}{\eta_{t-1}}\norm{x_t^+-x_{t-1}}^2 \right].
	\end{aligned}
	\]
	The first inequality is due to the $L$-smoothness of the function $f$. The second inequality again uses the fact that $\norm{x+y}^2 \leq 2\norm{x}^2+2\norm{y}^2$. The last step holds because of the non-expansiveness of the projection operator, that is,
	\[
	\norm{x_t - x_t^+} = \norm{\text{Proj}_{\mc{X}}(x_{t-1} - \eta_{t-1}d_{t-1}) - \text{Proj}_{\mc{X}}(x_{t-1}-\eta_{t-1}\grad F(x_{t-1}))}\leq \eta_{t-1}^2 \norm{\epsilon_{t-1}}^2.
	\]
	This completes the proof of the lemma. 
\end{proof}

Now, we are ready to present the convergence guarantee of Algorithm~\ref{alg:storm}. 

\vspace{.8em}
\noindent\textbf{Proposition~\ref{thm:storm_app}.} (Adapted from Theorem 2 in \citet{cutkosky2019momentum}).
	Suppose the conditions in Assumption~\ref{assumption:1} are satisfied. For any $b>0$, let $k = \frac{b\sigma^{\frac{2}{3}}}{L}, c = 32L^2 + \sigma^{2} /\left(7 L k^{3}\right)=L^{2}\left(32+1 /\left(7 b^{3}\right)\right), w=\max \left((4 L k)^{3}, 2 \sigma^{2},\left(\frac{c k}{4 L}\right)^{3}\right)=\sigma^{2} \max \left((4 b)^{3}, 2,\left(32 b+\frac{1}{7 b^{2}}\right)^{3} / 64\right)$, and $M = 16(F(x_1)-F^\star) + \frac{w^{1/3}\sigma^2}{2L^2 k} + \frac{k^3 c^2}{L^2}\ln(T+2)$. Then, the following convergence guarantee holds for Algorithm~\ref{alg:storm}:
	\[
	\ee\L\frac{1}{T}\sum_{t=1}^T  \norm{\frac{1}{\eta_t}(x_{t+1}^+-x_{t})}^2 \R \leq \frac{Mw^{1/3}}{Tk} + \frac{M\sigma^{2/3}}{T^{2/3}k}.
	\]
\begin{proof}
	First, define the Lyapunov function $\Phi_t = F(x_t) + \frac{1}{32L^2\eta_{t-1}}\norm{\epsilon_t}^2$. From Lemma~\ref{lemma:5}, we can derive that
	\[
	\begin{aligned}
	&\ee\L \frac{\norm{\epsilon_{t+1}}^2}{\eta_t} -  \frac{\norm{\epsilon_{t}}^2}{\eta_{t-1}}\R\\
	\leq& \mathbb{E}\left[2 c^{2} \eta_{t}^{3}\sigma^2+\frac{(4L^2 \eta_{t}^2+1)(1-a_{t+1})^2}{\eta_{t}}\norm{\epsilon_{t}}^2  + \frac{4\left(1-a_{t+1}\right)^{2}L^2}{\eta_{t}}\norm{x_{t+1}^+-x_{t}}^2 - \frac{\norm{\epsilon_{t}}^2}{\eta_{t-1}} \right]\\
	\leq & \mathbb{E}\left[\underbrace{2 c^{2} \eta_{t}^{3}\sigma^2}_{A_t}+ \underbrace{\l \frac{(4L^2 \eta_{t}^2+1)(1-a_{t+1})^2}{\eta_{t}} - \frac{1}{\eta_{t-1}}\r \norm{\epsilon_{t}}^2}_{B_t}  + \underbrace{\frac{4\left(1-a_{t+1}\right)^{2}L^2}{\eta_{t}}\norm{x_{t+1}^+-x_{t}}^2}_{C_t} \right]. 
	\end{aligned}
	\]
	The first two terms $A_t$ and $B_t$ are exactly the same as in the proof of Theorem 2 in \citet{cutkosky2019momentum}, and we refer to their results as follows:
	\[
	\sum_{t=1}^T A_t \leq 2k^3 c^2 \ln(T+2), \text{ and } \sum_{t=1}^T B_t \leq -28L^2 \sum_{t=1}^T \eta_t \norm{\epsilon_t}^2. 
	\]
	From $w\geq (4Lk)^3$, we know that $\eta_t \leq \frac{1}{4L}$. Further, since $a_{t+1}=c\eta_t^2$, we have that $a_{t+1}\leq \frac{ck}{4L w^{1/3}}\leq 1$ for all $t$, and hence $C_t\leq \frac{4L^2}{\eta_t}\norm{x_{t+1}^+ - x_t}^2$. Putting it all together, we obtain
	\begin{equation}
	\frac{1}{32L^2} \sum_{t=1}^T \l \frac{\norm{\epsilon_{t+1}}^2}{\eta_t} -  \frac{\norm{\epsilon_{t}}^2}{\eta_{t-1}}\r \leq \frac{k^3 c^2}{16L^2} \ln(T+2) + \sum_{t=1}^T \l \frac{1}{8\eta_t}\norm{x_{t+1}^+-x_{t}}^2 - \frac{7\eta_t}{8}\norm{\epsilon_t}^2 \r. \label{eqn:z5}
	\end{equation}
	From Lemma~\ref{lemma:1}, we know that
	\[
	\mathbb{E}\left[\Phi_{t+1}-\Phi_{t}\right] \leq \mathbb{E}\left[-\frac{3}{16\eta_{t}} \norm{x_{t+1}^+-x_{t}}^2  +\frac{7 \eta_{t}}{8}\left\|{\epsilon}_{t}\right\|^{2}+\frac{1}{32 L^{2} \eta_{t}}\left\|{\epsilon}_{t+1}\right\|^{2}-\frac{1}{32 L^{2} \eta_{t-1}}\left\|{\epsilon}_{t}\right\|^{2}\right].
	\]
	Summing over $t$ from $1$ to $T$ and then applying \eqref{eqn:z5}, we obtain
	\[
	\begin{aligned}
	\mathbb{E}\left[\Phi_{T+1}-\Phi_{1}\right] \leq& \sum_{t=1}^T \mathbb{E}\left[-\frac{3}{16\eta_{t}} \norm{x_{t+1}^+-x_{t}}^2  +\frac{7 \eta_{t}}{8}\left\|{\epsilon}_{t}\right\|^{2}+\frac{1}{32 L^{2} \eta_{t}}\left\|{\epsilon}_{t+1}\right\|^{2}-\frac{1}{32 L^{2} \eta_{t-1}}\left\|{\epsilon}_{t}\right\|^{2}\right]\\
	\leq & \ee\L \frac{k^3 c^2}{16L^2} \ln(T+2) - \sum_{t=1}^T  \frac{1}{16\eta_t}\norm{x_{t+1}^+-x_{t}}^2 \R. 
	\end{aligned}
	\]
	Rearranging the terms leads to
	\[
	\begin{aligned}
	\ee\L\sum_{t=1}^T  \frac{1}{\eta_t}\norm{x_{t+1}^+-x_{t}}^2 \R \leq&\ee\L 16(\Phi_{1} - \Phi_{T+1}) + \frac{k^3 c^2}{L^2}\ln(T+2) \R\\
	\leq & 16(F(x_1)-F^\star) + \frac{1}{2L^2 \eta_0}\ee[\norm{\epsilon_1}^2] + \frac{k^3 c^2}{L^2}\ln(T+2)\\
	\leq & 16(F(x_1)-F^\star) + \frac{w^{1/3}\sigma^2}{2L^2 k} + \frac{k^3 c^2}{L^2}\ln(T+2),
	\end{aligned}
	\]
	where the last step holds due to the definition that $\eta_0 = \frac{k}{w^{1/3}}$. Since $\eta_t$ is decreasing in $t$, 
	\[
	\ee\L\sum_{t=1}^T  \frac{1}{\eta_t}\norm{x_{t+1}^+-x_{t}}^2 \R = \ee\L\sum_{t=1}^T  \eta_t\norm{\frac{1}{\eta_t}(x_{t+1}^+-x_{t})}^2 \R \geq  \eta_T\ee\L\sum_{t=1}^T  \norm{\frac{1}{\eta_t}(x_{t+1}^+-x_{t})}^2 \R. 
	\]
	Dividing both sides by $T\eta_T$ and recalling the definition $M = 16(F(x_1)-F^\star) + \frac{w^{1/3}\sigma^2}{2L^2 k} + \frac{k^3 c^2}{L^2}\ln(T+2)$, we obtain
	\[
	\ee\L\frac{1}{T}\sum_{t=1}^T  \norm{\frac{1}{\eta_t}(x_{t+1}^+-x_{t})}^2 \R \leq \frac{M}{T\eta_T} = \frac{M(w+\sigma^2 T)^3}{Tk} \leq \frac{Mw^{1/3}}{Tk} + \frac{M\sigma^{2/3}}{T^{2/3}k},
	\]
	where in the last step we used the fact that $(a+b)^{1/3}\leq a^{1/3} + b^{1/3}$. 
\end{proof}

\section{Simulations}\label{app:simulations}

In this section, we demonstrate the empirical performances of our algorithms, and compare their performances with various benchmarks. We evaluate Algorithm~\ref{alg:storm} (SGA) on a classic matrix team task~\citep{claus1998dynamics}, and both Algorithms~\ref{alg:sbv} and~\ref{alg:storm} on two Markov games, namely GoodState and BoxPushing~\citep{seuken2007improved}.

\subsection{Markov Teams}
We use a classic matrix team example from the literature~\citep{claus1998dynamics,lauer2000algorithm}, where a team problem is a special case of potential games. Its reward table is reproduced in Table~\ref{tbl:1}, where agent 1 is the row player, and agent 2 is the column player, both being maximizers. The action spaces of the agents are $\mc{A}_1 = \{a_0,a_1,a_2\}$ and $\mc{B}_2 =\{b_0,b_1,b_2\}$. There are three deterministic Nash equilibria in this team, among which two of them, $(a_0,b_0)$ and $(a_2,b_2)$, are team-optimal. It would be preferred that the agents not only learn a NE, but also settle on the same NE out of the two team-optimal ones.

\begin{table}[!htb]
	\def\arraystretch{1.3}
	\begin{minipage}[t]{.5\linewidth}
		\centering
		\begin{tabular}[b]{c|ccc}
			& $b_0$ & $b_1$ & $b_2$\\
			\hline
			$a_0$ & 10 & 0 & -10\\
			$a_1$ & 0  & 2 & 0\\
			$a_2$ & -10 & 0 & 10
		\end{tabular}
		\caption{Reward table for the matrix team.}\label{tbl:1}
	\end{minipage}\hfill 
	\begin{minipage}[t]{.5\linewidth}
		\centering
		\begin{tabular}[b]{c|cc}
			$s_0$ & $b_0$ & $b_1$ \\
			\hline
			$a_0$ & -2 & 5 \\
			$a_1$ & 2 & -2
		\end{tabular}\qquad
		\begin{tabular}[b]{c|cc}
			$s_1$ & $b_0$ & $b_1$ \\
			\hline
			$a_0$ & 0 & 0 \\
			$a_1$ & 0 & 0
		\end{tabular}
		\caption{Reward tables for GoodState.}\label{tbl:2}
	\end{minipage} 
\end{table}

We run Algorithm~\ref{alg:storm} on this task for $T = 5000$ rounds, and we set the step size $\eta_t = 10^{-4}$ and the momentum parameter $a_{t} = 0.5$. We evaluate our algorithm in terms of both the rewards it obtained and its $L^2$ equilibrium gap. Specifically, we define the $L^2$ equilibrium gap as the $L^2$ distance to a equilibrium point. For a pair of strategies $(\mu,\nu)\in\Delta(\mc{A}_1)\times \Delta(\mc{A}_2)$, its $L^2$ equilibrium gap is defined as:
\begin{equation}\label{eqn:gap}
\text{Gap}(\mu,\nu) \defeq \norm{\mu- \mu^\dagger(\nu)}_2^2 + \norm{\nu-\nu^\dagger(\mu) }_2^2,
\end{equation}
where $\nu^\dagger(\mu)$ (resp. $\mu^\dagger(\nu)$) is the best response with respect to $\mu$ (resp. $\nu$), and $\norm{\cdot}_2$ is the $L^2$ norm. The simulation results are presented in Figure~\ref{fig:1}. All results are averaged over $20$ runs. Notice that we evaluate two sets of strategy trajectories: The ``Last Iterate'' strategy $(\mu_t,\nu_t)$ is the strategy pair used by Algorithm~\ref{alg:storm} at round $t$, while the ``Average'' strategy  is to uniformly draw a random time index $\tau$ from $\{1,\dots,t\}$ and run the strategy pair $(\mu_\tau,\nu_\tau)$. Notice that in Theorem~\ref{thm:finite}, our theoretical guarantees only hold in expectation, which correspond to the ``Average'' strategies.

\begin{figure*}[!htbp]
	\centering
	\subfigure[Equilibrium gap]{\label{subfig:1}\includegraphics[width=0.35\textwidth]{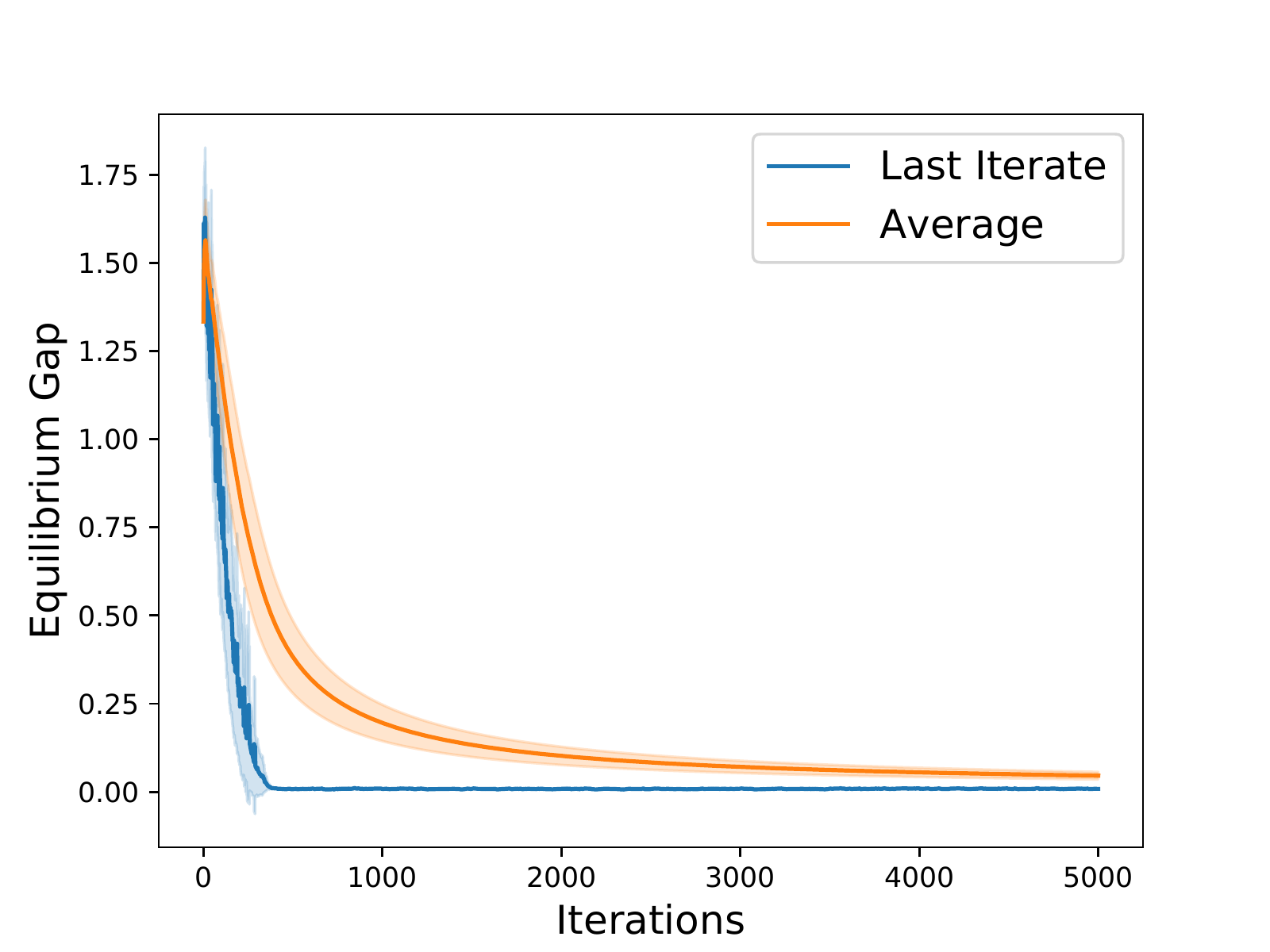}}
	\subfigure[Reward]{\label{subfig:2}\includegraphics[width=0.35\textwidth]{matrix1_reward}}
	\caption{(a) $L^2$ equilibrium gaps and (b) rewards of Algorithm~\ref{alg:storm} on the matrix team task given in Table~\ref{tbl:1}. ``Last Iterate'' denotes the actual strategy trajectory, while ``Average'' represents the uniformly sampled strategy pair. Shaded areas denote the standard deviations of the equilibrium gap or reward. All results are averaged over 20 runs. }
	\label{fig:1}
\end{figure*}

From Figure~\ref{subfig:1}, we can see that the equilibrium gap of both ``Last Iterate'' and ``Average'' converge to zero, indicating that they indeed find an equilibrium as the number of iterations increase. The convergence of ``Average'' slightly lags behind ``Last Iterate'' because ``Average'' essentially takes the time-averaged value of the actual trajectories, which requires some time to reflect the convergence behavior. A more promising result is that from Figure~\ref{subfig:2}, we can see that the rewards collected by ``Last Iterate'' and ``Average'' converge to values close to $9$. This suggests that Algorithm~\ref{alg:storm} not only finds a NE in this specific task, but actually converges to a team-optimal equilibrium most of the time. It does not exactly reach the team-optimal value of $10$ because it still converges to non-team-optimal NE at a rather low frequency. 

\subsection{Markov Games}

We further evaluate both Algorithm~\ref{alg:sbv} and Algorithm~\ref{alg:storm} on two Markov games, namely GoodState and BoxPushing~\citep{seuken2007improved}. The GoodState task is a simple Markov team problem inspired by~\citet{yongacoglu2019learning}. It has two states $\mc{S} = \{s_0,s_1\}$, where $s_0$ is the ``good state'' and $s_1$ is the ``bad state''. Each agent has two candidate actions $\mc{A}_1 = \{a_0,a_1\}$ and $\mc{A}_2 = \{b_0,b_1\}$. The reward function at each state is presented in Table~\ref{tbl:2}. Specifically, at state $s_1$, both agents get a reward of $0$ no matter what actions they select, while at state $s_0$, they will obtain a strictly positive reward if they either take the joint action $(a_0,b_1)$ or the one $(a_1,b_0)$. The state transition function is defined as follows:
\[
P_h(s_0 \mid s_0 \text{ or } s_1, a_0,b_1) = 1-\epsilon,\ P_h(s_1 \mid s_0 \text{ or } s_1, \text{ not } (a_0,b_1)) = 1-\epsilon,\ \forall h \in [H],
\]
and all the other transitions happen with probability $\epsilon$. Intuitively, no matter which state the agents are in, they will transition to the good state $s_0$ with a high probability $1-\epsilon$ at the next step as long as they select the action pair $(a_0,b_1)$. All the other joint actions will lead to the bad state $s_1$ with a high probability $1-\epsilon$. The task hence rewards the agents who learn to consistently play the action pair $(a_0,b_1)$. 

We run our two algorithms on this example for $K = 50000$ episodes, each episode containing $H=10$ steps. We set the transition probability $\epsilon=0.1$. For Algorithm~\ref{alg:sbv}, the step size is set to be $\eta_i = \frac{1}{5\sqrt{A_i \check{T}_h(s_h)}}$, and the implicit exploration parameter is $\gamma_i = \eta_i / 2$. For Algorithm~\ref{alg:storm}, the step size is set to be $\eta_t = 10^{-4}$ and the momentum parameter is $a_{t} = 0.5$. 

The BoxPushing task~\citep{seuken2007improved} is a classic DecPOMDP problem with with $\sim$100 states. It has two 2 agents, where each agent has 4 candidate actions. In the original BoxPushing problem, each agent only has a partial observation of the state. We make proper modifications to the task so that the agents can fully observe the state information and fit in our problem formulation. For Algorithm~\ref{alg:sbv} on this task, the step size is set to be $\eta_i = \frac{1}{20\sqrt{A_i \check{T}_h(s_h)}}$, and the implicit exploration parameter is $\gamma_i = \eta_i / 2$. For Algorithm~\ref{alg:storm}, the step size is set to be $\eta_t = 5\times  10^{-4}$ and the momentum parameter is $a_{t} = 0.1$.

\begin{figure*}[!t]
	\centering
	\subfigure[GoodState]{\includegraphics[width=0.35\textwidth]{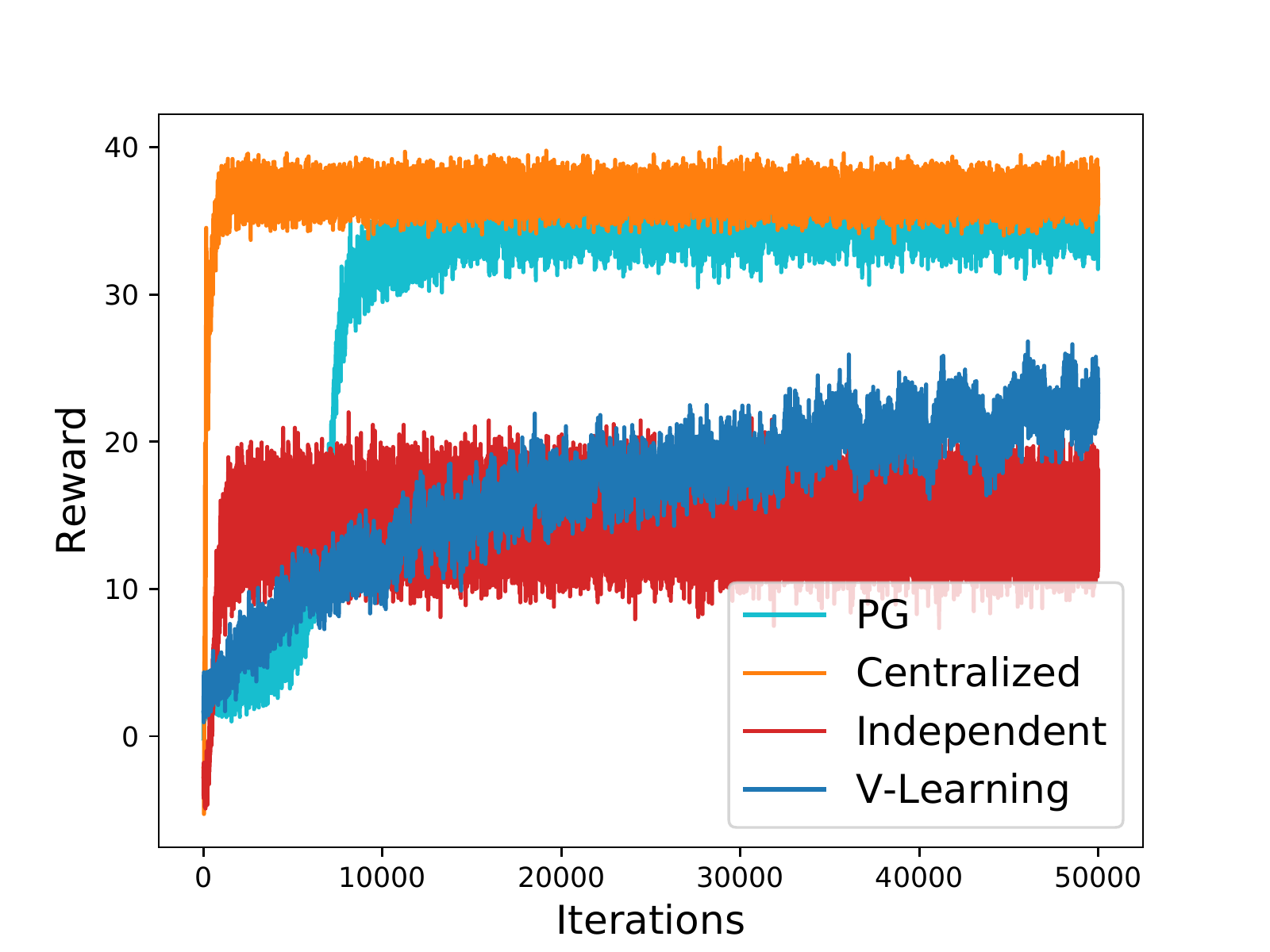}}
	\subfigure[BoxPushing]{\includegraphics[width=0.35\textwidth]{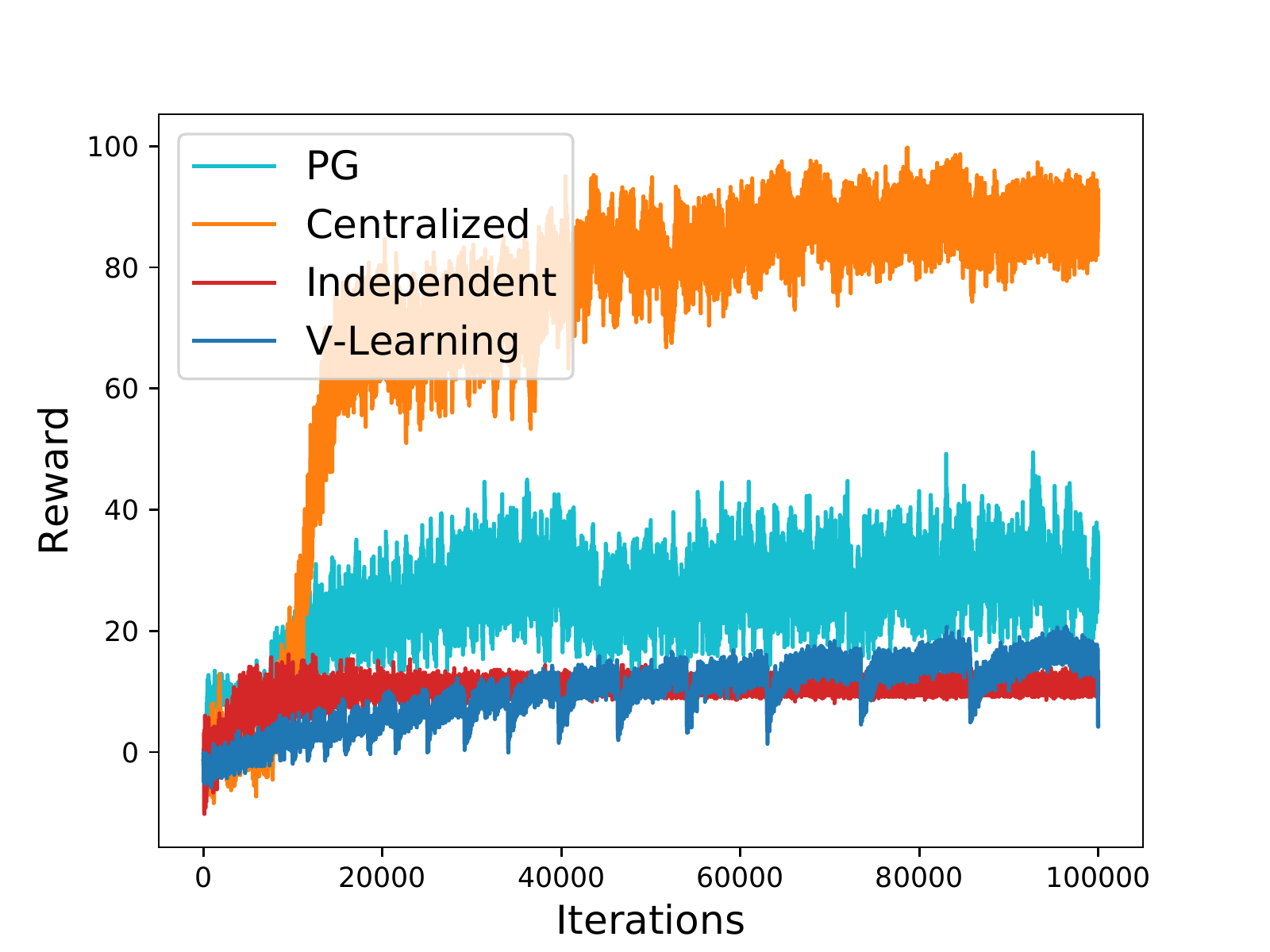}}
	\caption{Rewards of Algorithms~\ref{alg:sbv} and~\ref{alg:storm} on the (a) GoodState and (b) BoxPusing tasks. ``V-Learning''  and ``PG'' denote the policies at the current iterate $t$ of Algorithms~\ref{alg:sbv} and~\ref{alg:storm}, respectively.  ``Centralized'' is an oracle that can control the actions of the agents in a centralized way. In ``Independent'', each agent runs a na\"{i}ve single-agent Q-learning algorithm independently, by taking greedy actions w.r.t its local Q-function estimates. All results are averaged over 20 runs. }	
	\label{fig:2}
\end{figure*}

We compare our algorithms with two meaningful benchmarks. The first benchmark is a ``Centralized'' oracle. This oracle acts as a centralized coordinator that can control the actions of both agents. Such an oracle essentially converts the multi-agent task into a single-agent RL problem. The (randomized) action space of the centralized agent is $\Delta(\mc{A}_1\times\mc{A}_2)$, which is larger than the $\Delta(\mc{A}_1)\times \Delta(\mc{A}_2)$ space that we allow for Algorithm~\ref{alg:storm} in our decentralized approach.  ``Centralized'' clearly upper bounds the performances that our decentralized learning algorithms can possibly achieve in this task. In our simulations, we implement ``Centralized'' by using a Hoeffding-based variant of a state-of-the-art single-agent RL algorithm UCB-ADVANTAGE~\citep{zhang2020model}. This algorithm has achieved a tight sample complexity bound for single-agent RL in theory, and has also demonstrated remarkable empirical performances in practice~\citep{mao2020near}. Such an algorithm could provide a strong performance upper bound in our task. The second benchmark we consider is the na\"{i}ve ``Independent'' Q-learning. Specifically, we let each agent run a single-agent Q-learning algorithm independently, without being aware of the existence of the other agent or the structure of the game. Each agent maintains an local optimistic Q-function, and takes greedy actions with respect to such optimistic estimates, without taking into account the other agents' actions. Since the agents update their policies simultaneously, the stationarity assumption of the environment in single-agent RL quickly collapses, and the theoretical guarantees for single-agent Q-learning no longer hold. This is also reminiscent of the ``independent learner'' approach proposed in an early work~\citep{claus1998dynamics} for learning in Markov teams.  We believe that such a benchmark could provide meaningful intuitions about the consequences of not taking care of the multi-agent structure in decentralized methods. In our simulations, we implement such a benchmark by letting each agent running a variant of the single-agent UCB-ADVANTAGE~\citep{zhang2020model} algorithm independently, where the (randomized) action spaces of the agents are $\Delta(\mc{A}_1)$ and $\Delta(\mc{A}_2)$.

Figure~\ref{fig:2} illustrates the performances of our algorithms and the two benchmark methods in terms of the collected rewards, where ``V-Learning''  and ``PG'' denote the policies at the current iterate $t$ of Algorithms~\ref{alg:sbv} and~\ref{alg:storm}, respectively. Notice that the \emph{actual} policy trajectories of both algorithms numerically converge and achieve high rewards. This is more encouraging than our theoretical guarantees, because for Algorithm~\ref{alg:sbv}, our Theorem~\ref{thm:finite} only holds for a ``certified'' output policy but not the last-iterate policy. Further, both of our algorithms outperform the ``Independent'' learning benchmark on the two tasks. In the GoodState problem, Algorithm~\ref{alg:storm} even approaches the performance of the ``Centralized'' oracle. On the other hand, the ``Independent'' benchmark converges, albeit faster, to a clearly suboptimal value. This reiterates that the na\"ive idea of independent learning does not work well for MARL in general, and a careful treatment of the game structure (like our adversarial bandit subroutine) is necessary. Finally, the implemented algorithms take much fewer samples to converge than our theoretical results suggested. This indicates that the theoretical bounds might be overly conservative, and our algorithms could converge much faster in practice.

\end{document}